\documentclass{article} 

\usepackage{amsmath}
\usepackage{amsthm}
\usepackage{amssymb}
\usepackage[a4paper, total={6.5in, 9in}]{geometry}

\usepackage{amsmath,amsfonts,bm}

\def\eqref#1{equation~\ref{#1}}

\def\1{\bm{1}}

\DeclareMathAlphabet{\mathsfit}{\encodingdefault}{\sfdefault}{m}{sl}
\SetMathAlphabet{\mathsfit}{bold}{\encodingdefault}{\sfdefault}{bx}{n}

\newcommand{\E}{\mathbb{E}}

\newcommand{\R}{\mathbb{R}}

\newcommand{\Var}{\mathrm{Var}}

\newcommand{\Cov}{\mathrm{Cov}}

\DeclareMathOperator*{\argmax}{arg\,max}
\DeclareMathOperator*{\argmin}{arg\,min}

\def\eqref#1{(\ref{#1})}

\usepackage{bm}
\usepackage[table]{xcolor}
\usepackage{hyperref}
\usepackage{url}

\usepackage{natbib}
\newcommand{\doop}{\mathrm{do}}
\newcommand{\CF}{\mathrm{CF}}
\newcommand{\Kb}{\mathbf{K}} 
\DeclareMathOperator{\pa}{pa}
\DeclareMathOperator{\blockdiag}{blockdiag}
\DeclareMathOperator{\diag}{diag}

\usepackage{placeins}
\usepackage{adjustbox}
\usepackage{multirow}
\usepackage{colortbl}
\usepackage{graphicx}
\graphicspath{{./}{Arxiv/}}
\usepackage{tikz}
\usetikzlibrary{positioning, arrows.meta, calc}
\definecolor{contcolor}{HTML}{4477AA} 
\definecolor{bincolor}{HTML}{CC6677}  
\definecolor{nomcolor}{HTML}{DDCC77}  
\definecolor{ordcolor}{HTML}{117733}  
\definecolor{posetcolor}{HTML}{AA4499} 
\tikzset{
  var/.style={circle, draw, minimum size=6mm, inner sep=0pt, font=\small},
  continuous/.style={fill=contcolor!35},
  binary/.style={fill=bincolor!35},
  nominal/.style={fill=nomcolor!35},
  ordinal/.style={fill=ordcolor!35},
  poset/.style={fill=posetcolor!35},
}
\usepackage{algorithm}
\usepackage{algpseudocode}
\usepackage{tcolorbox}

\theoremstyle{plain}
\newtheorem{theorem}{Theorem}
\newtheorem{proposition}[theorem]{Proposition}
\newtheorem{lemma}[theorem]{Lemma}
\newtheorem{corollary}[theorem]{Corollary}

\theoremstyle{definition}

\theoremstyle{remark}
\newtheorem{remark}[theorem]{Remark}

\newcommand{\rev}[1]{#1}
\newcommand{\cutnew}[1]{#1}
\newenvironment{revblock}{}{}
\definecolor{revcolor}{HTML}{000000}
\definecolor{cutcolor}{HTML}{000000}
\definecolor{tvintbg}{gray}{0.92}
\newcommand{\tvi}[1]{\cellcolor{tvintbg}#1}
\definecolor{altcolbg}{gray}{0.94} 

\title{TACC: Type-Aware Counterfactual Couplings for Discrete Outcomes in Gaussian-Process Causal Models}

\usepackage{authblk}
\author[1]{Juliette Sinnott}
\author[2,3]{Amir-Hossein Karimi}
\author[1]{Mohammad Kohandel}
\affil[1]{Department of Applied Mathematics, University of Waterloo, Waterloo, Ontario, N2L 3G1, Canada}
\affil[2]{Department of Electrical and Computer Engineering, University of Waterloo, Waterloo, Ontario, N2L 3G1, Canada}
\affil[3]{Vector Institute for Artificial
Intelligence, Toronto, Ontario, M5G 0C6, Canada}

\date{*Email: jssinnot@uwaterloo.ca}

\begin{document}

\maketitle

\begin{abstract}
    Counterfactual inference in Gaussian-process structural causal models (GP-SCMs) has been developed primarily for continuous endogenous variables. Current approaches do not naturally extend to causal graphs with discrete children of continuous parents, which often occur in real-world systems. We pair GP predictors with explicit exogenous noise mechanisms and derive exact noise abduction for categorical outcomes of any order type in one construction. We use a latent Gaussian threshold that picks a block of tied classes and a Gumbel-max race that picks the class within it. This approach specialises to binary (uniform threshold), nominal (Gumbel-max race), and ordinal (latent Gaussian cutpoint) outcomes, and to tiered and partially ordered ones. Abducted noise is propagated through interventions while accounting for GP posterior uncertainty, and we prove that the resulting mechanisms reproduce the fitted model's observational and interventional distributions. We evaluate accuracy, consistency, and robustness to coupling misspecification on synthetic, random, and semi-synthetic SCMs. Our estimator, type-aware counterfactual couplings (TACC), has the lowest counterfactual error on every node of the synthetic SCMs, $2.4$--$14\times$ below discretised continuous baselines, while matching or improving on their interventional error. Using the wrong coupling inflates counterfactual error three- to ninefold even when the fit looks equally good, and more data does not remove it. Counterfactuals in systems with mixed variable types are therefore estimated best by choosing each variable's coupling from its type.
\end{abstract}


\section{Introduction}
\label{sec:intro}

Counterfactuals let us reason about how an intervention \textit{could have} changed the outcome of a given factual instance. They cannot be observed directly, so they must be estimated from a model, and such estimates support reasoning about real-world systems in finance, politics, and medicine. Structural causal models (SCMs) estimate them \cutnew{in three steps (abduction, action, and prediction; \S\ref{sec:cf-steps}) from structural equations fitted to data \citep{pearl2016causal}, which} can be fitted with different techniques \citep{spirtes2000causation,zhang2011kernel,hoyer2009nonlinear}; Gaussian process (GP) regression is a flexible choice that can approximate linear, non-linear, and additive relationships. Gaussian-process SCMs (GP-SCMs) also quantify uncertainty, allowing probabilistic counterfactual estimation, and assuming additive Gaussian noise at each node makes abduction exact \citep{DBLP:journals/corr/abs-2006-06831}. However, GP-SCMs were developed for continuous variables only.

Extending this probabilistic approach to real-world systems with mixed variable types requires rebuilding the abduction and prediction steps, starting with how the factual and counterfactual noise are coupled. We propose type-aware counterfactual couplings (TACC), one construction that covers binary, nominal, ordinal, tiered \cutnew{(ordered groups of unordered classes)}, and partially ordered features (\S\ref{sec:results}). Unordered discrete nodes are fitted with a GP classifier, and the factual and counterfactual instances are coupled by the Gumbel-max mechanism of \citet{oberst2019counterfactual}; for ordered nodes, we show that matching the coupling to the variable's order significantly reduces counterfactual error. Applying the Gumbel-max coupling to every discrete node fits ordinal data as well, yet its counterfactuals are several times worse. We formally prove TACC's properties from the structure of its exogenous noise and validate it experimentally on synthetic, random, and semi-synthetic SCMs \cutnew{(\S\ref{sec:findings})}.

\subsection{\texorpdfstring{\cutnew{Counterfactuals in three steps}}{Counterfactuals in three steps}}\label{sec:cf-steps}

We define counterfactuals and their evaluation as in \cite{pearl2016causal}.
Given an SCM $M$ with $d$ nodes and equations
$\mathbf{F}=\{\mathbf{X}_r:=f_r(\mathbf{X}_{\pa(r)},U_r)\}_{r=1}^d$\cutnew{, with $\pa(r)$ the parents of node $r$,} and an observed
instance $\mathbf{x}^F$, the counterfactual
$\mathbf{x}^{\CF}=\mathbf{x}(\doop(\mathbf{X}_{\mathcal{I}}=\bm{\theta}))\mid\mathbf{x}^F$
asks what would have happened had $\mathbf{X}_{\mathcal{I}}$ taken the value
$\bm{\theta}$. It is computed in three steps: \textbf{(1)} \emph{abduction}, using
$\mathbf{x}^F$ and the equations to infer $\{U_r\}_{r=1}^d$; \textbf{(2)}
\emph{action}, replacing the equations of $\mathbf{X}_{\mathcal{I}}$ by
$\doop(\mathbf{X}_{\mathcal{I}}=\bm{\theta})$; and \textbf{(3)} \emph{prediction},
propagating the abducted noise through the modified equations to obtain
$\mathbf{x}^{\CF}$.
When the equations are fitted from data, we also require \emph{compatibility}:
the SCM with explicit noise must reproduce the fitted model's observational and
interventional distributions. The choice of noise then changes only the
counterfactuals. We prove compatibility for every variable type
(\S\ref{sec:results}).

Typically, the most challenging step in this process is \textit{abduction}, because the method to isolate the factual noise depends on the form of the structural equations\cutnew{, and additive Gaussian noise does not apply to discrete features}.  Consider a bank deciding whether to extend a loan. A continuous feature such
as debt-to-income ratio determines an ordinal risk tier (subprime through
super-prime, from thresholding a composite score), which in turn determines
a nominal decision such as which loan product is offered. An applicant who
is denied may reasonably ask what would have happened had their financial
position been different. Answering this counterfactual query means carrying
abduction, action, and prediction through a causal chain whose structural
equations mix continuous, ordinal, and nominal variables.

\subsection{\texorpdfstring{\cutnew{Prior work}}{Prior work}}

Prior work falls into three families, according to the stage of the pipeline
it addresses: \emph{discovering} the graph, \emph{fitting} equations on a given
graph, or \emph{coupling} a fixed interventional law into a counterfactual
mechanism.

\textbf{Graph discovery.} Recent work uses ordinal structure to orient edges
\citep{ni2022ordinal}. \citet{yao2025mixed} use a categorical mechanism close to
ours to identify graphs in mixed additive noise models, but they do not address
counterfactual queries. We assume the causal graph is known, either from domain knowledge or from a discovery method such as these, and focus on estimating counterfactuals given that graph.

\textbf{Fitting discrete and ordinal equations.} Ordinal variables are
classically modelled by discretising a latent score at fixed cutpoints
\citep{mckelvey1975statistical,boes2013nonparametric}. \citet{chu2005gaussian}
give the Gaussian process version of this likelihood, which our construction
builds on. Deep generative models offer two alternatives. Some dequantise
discrete variables into normalising-flow SCMs \citep{pawlowski2020deep}. Others
relax categorical latents with a temperature-annealed softmax
\citep{jang2017categorical,maddison2017concrete}. This relaxation enables
gradient-based training, but it gives up exact abduction, which we keep.

\textbf{Coupling a fixed interventional law.} \citet{oberst2019counterfactual}
introduce the Gumbel-max coupling and justify it by an axiom they call
\emph{counterfactual stability}. \citet{lorberbom2021learning} instead
\emph{learn} a coupling that minimises the variance of counterfactual treatment
effects over a distribution of queries. Our setting differs in two ways. First,
our interventional law is not fixed; it is a fitted Gaussian process with
posterior uncertainty. Second, we choose the coupling to match the variable's
\emph{type}, not to optimise a downstream objective.

\section{\texorpdfstring{\cutnew{Gumbel noise and couplings}}{Gumbel noise and couplings}}
\label{gen_inst}

\subsection{\texorpdfstring{\cutnew{Gumbel noise for exact abduction}}{Gumbel noise for exact abduction}}

The Gumbel distribution gives an exact, explicit exogenous noise for categorical
outcomes. A random variable $G\sim\mathrm{Gumbel}(\mu,1)$ has cumulative distribution
function (CDF) $\exp(-e^{-(g-\mu)})$ \citep{gumbel1954}. We use three of its
properties. In each, $V\sim\mathrm{Uniform}(0,1)$.
\begin{itemize}
  \item \textbf{Reparameterisation.} $\mu-\log(-\log V)\sim\mathrm{Gumbel}(\mu,1)$.
  \item \textbf{Truncation.} $\mu-\log(e^{-(b-\mu)}-\log V)$ has the law of $G$
  conditioned on $G\le b$ \citep{maddison2014astar}.
  \item \textbf{Gumbel-max.} Let $\bm\ell\in\R^C$ be a vector of scores and let
  $G_c\sim\mathrm{Gumbel}(0,1)$ be iid. Then the maximum
  $M:=\max_c\{\ell_c+G_c\}$ follows $\mathrm{Gumbel}(\log\sum_c e^{\ell_c},1)$,
  the winning class $\argmax_c\{\ell_c+G_c\}$ follows $\mathrm{softmax}(\bm\ell)$,
  and the two are independent \citep{yellott1977}.
\end{itemize}

\subsection{\texorpdfstring{\cutnew{Many couplings, one fit}}{Many couplings, one fit}}
\label{sec:coupling-intro}

For a discrete outcome, abduction no longer has the unique answer that additive
Gaussian noise gives. The factual and counterfactual class distributions,
$\mathbf{p}^F$ and $\mathbf{p}^{\CF}$, are each pinned down by the fitted model,
but they alone do not determine a joint law over the pair $(X_r,X_r^{\CF})$:
infinitely many joint distributions share the same two marginals. Call any such
choice a \emph{coupling}. Abduction realises a coupling: we equip $X_r$ with an
explicit exogenous noise variable, infer its posterior from $x_r^F$, and reuse
that noise unchanged under the counterfactual parents. The coupling is therefore
determined entirely by how the noise is defined, which is the modelling choice
this paper studies for each variable type.

A \textbf{binary} outcome needs one uniform noise thresholded against the class
probability, the unique coupling under monotonicity. For $C\ge3$ \textbf{nominal}
classes, a Gumbel perturbation per class is abducted from which class won the
factual ``race'' and reused in a counterfactual race; this is the Gumbel-max
coupling, which is symmetric in the classes. For \textbf{ordinal} classes, one shared Gaussian
pushed through the same monotone thresholds in both worlds gives the
\emph{comonotone} coupling \cutnew{(the monotonicity assumption of
\citealp{pearl1999probabilities,lu2020sample}; Remark~\ref{rem:couplings})}, which respects the order. \textbf{Partially ordered}
classes use one such Gaussian per direction of the order (\S\ref{sec:poset}). At
$C=2$ the couplings coincide. Throughout, \emph{nominal}, \emph{ordinal}, and
\emph{partially ordered} name a variable's type, while \emph{all-categorical} names the estimator that fits every discrete node
with the Gumbel-max model, regardless of \cutnew{its} type.

\section{\texorpdfstring{\cutnew{Type-aware counterfactual couplings (TACC)}}{Type-aware counterfactual couplings}}
\label{sec:results}

\begin{revblock}
The constructions for every type share one template, which we state once.
The type then enters only through three choices: a score computed from the
fitted \cutnew{GP latent function (Appendix~\ref{app:defs})}, an exogenous noise, and a deterministic map from the two to a
class.

\begin{proposition}[Counterfactuals for discrete GP-SCM nodes]
\label{prop:general}
For node $r$, let $\mathbf s(\mathbf x)$ be a \emph{score} computed from a
realisation of the fitted latent (and any fitted mechanism parameters) at
$\mathbf X_{\pa(r)}=\mathbf x$, and model
\begin{equation}\label{eq:general_scm}
  X_r := h\bigl(\mathbf s(\mathbf X_{\pa(r)}),\,U_r\bigr),\qquad U_r\sim P_U,
\end{equation}
with $h$ deterministic and $U_r$ independent of the latent and of everything
upstream. Suppose $P_U\bigl(h(\mathbf s,U_r)=c\bigr)=p_c(\mathbf s)$ for every
$\mathbf s$, where $p_c$ is the fitted class probability.

\textbf{(Compatibility.)} For every $\mathbf x$ and every realisation of the
latent, $P\bigl(X_r=c\mid\doop(\mathbf X_{\pa(r)}=\mathbf x),\mathbf s\bigr)=p_c(\mathbf s(\mathbf x))$,
so~\eqref{eq:general_scm} matches the fitted model on all observational and
interventional distributions.

\textbf{(Abduction.)} Write $\mathbf s^F:=\mathbf s(\mathbf x^F_{\pa(r)})$. Given
$\mathbf s^F$ and $X_r=x_r^F$, the posterior of $U_r$ is $P_U$ restricted to the
\emph{abduction region} $A^F:=\{u:h(\mathbf s^F,u)=x_r^F\}$ and renormalised.

\textbf{(Action and Prediction.)} With $\tilde{\mathbf s}:=\mathbf s(\tilde{\mathbf x}_{\pa(r)})$
and $U_r$ held fixed, $X_r^{\CF}=h(\tilde{\mathbf s},U_r)$, and
\begin{equation}\label{eq:general_cf}
  P\bigl(X_r^{\CF}=c\mid\mathbf x^F\bigr)
  =\E_{(\mathbf s^F,\tilde{\mathbf s})}\,\E_{U_r\mid A^F}\bigl[\mathbb I\{h(\tilde{\mathbf s},U_r)=c\}\bigr],
\end{equation}
where $(\mathbf s^F,\tilde{\mathbf s})$ is a \emph{single} draw from the joint posterior
of the latent at the factual and counterfactual parents. Its Monte Carlo
estimate is unbiased, converges almost surely, and is a valid distribution at
every sample size.
\end{proposition}

\begin{corollary}[Consistency under the trivial intervention]
\label{cor:consistency}
In the setting of Proposition~\ref{prop:general}, intervening at the factual
parent value recovers the observed outcome with certainty:
\begin{equation}\label{eq:consistency}
    P\bigl(X_r^{\CF}=x_r^F \mid \mathbf{x}^F,\;\doop(\mathbf{X}_{\pa(r)}=\mathbf{x}^F_{\pa(r)})\bigr)=1.
\end{equation}
\end{corollary}

Consistency uses no property of $h$ or $P_U$: the joint posterior puts
$\tilde{\mathbf s}=\mathbf s^F$, and every $u\in A^F$ maps back to $x_r^F$. It does
need the single joint draw (proofs in Appendix~\ref{sec:proof-general}). What the
type fixes is the triple $(\mathbf s,U_r,h)$, and with it the coupling of
\S\ref{sec:coupling-intro}. We give one triple that covers unordered, totally
ordered, tiered, and partially ordered classes at once.

\textbf{One construction for every order type.} Describe the order on the
classes of node $r$ in two parts. First, partition the classes into
\emph{blocks} $\mathcal A=\{A_1,\ldots,A_L\}$ of tied (mutually incomparable)
classes, and write $A(a)$ for the block of class $a$. Second, order the blocks
by presenting them as the image of a product of chains
$\mathcal G:=\prod_{j=1}^k[C_j]$ under a surjection $\pi:\mathcal G\to\mathcal A$
that is \emph{order-preserving} ($\mathbf g\le\mathbf g'$ coordinatewise implies
$\pi(\mathbf g)\preceq\pi(\mathbf g')$). Each coordinate is one direction of the
order, and $k=0$ means a single cell, so a single block. A threshold on $k$
latent scores then picks the block, and a Gumbel-max race picks the class
within it.

\begin{proposition}[Counterfactuals for categorical nodes of any order type]
\label{prop:unified}
For $j=1,\ldots,k$, let $f_{r,j}\sim\mathcal{GP}(0,k_r)$ independently \cutnew{($k_r$ a covariance kernel)}, with
cutpoints $-\infty=b_{j,0}<b_{j,1}<\cdots<b_{j,C_j}=\infty$. For every block with
$\lvert A\rvert\ge2$, let $\bm\ell_A(\mathbf x)$ be the fitted log class-probabilities
within $A$ (a GP classifier on the classes of $A$, as in Corollary~\ref{cor:nominal}).
Model
\begin{equation}\label{eq:poset_scm}
    Y_{r,j}:=f_{r,j}(\mathbf X_{\pa(r)})+N_{r,j},\quad
    \Lambda_r:=\pi\bigl(\kappa(\mathbf Y_r)\bigr),\quad
    X_r:=\argmax_{a\in\Lambda_r}\bigl\{\ell_{\Lambda_r,a}(\mathbf X_{\pa(r)})+G_{r,a}\bigr\},
\end{equation}
with $\mathbf N_r\sim\mathcal N(\mathbf 0,I_k)$ and $G_{r,a}\stackrel{iid}{\sim}\mathrm{Gumbel}(0,1)$
for every class $a$, where $\kappa_j(y):=c$ iff $b_{j,c-1}\le y<b_{j,c}$ and a
singleton block's race is trivial. Write
$B_{\mathbf g}:=\prod_j[b_{j,g_j-1},b_{j,g_j})$ for the box of cell $\mathbf g$ and
$m_{\mathbf g}(\mathbf f):=\prod_j\bigl[\Phi(b_{j,g_j}-f_j)-\Phi(b_{j,g_j-1}-f_j)\bigr]$
for its mass\cutnew{, with $\Phi$ the standard normal CDF}.

\textbf{(Compatibility.)} For every $\mathbf x$ and every realisation of the latent,
\begin{equation}\label{eq:poset_identity}
    P\bigl(X_r=a\mid\mathbf X_{\pa(r)}=\mathbf x,\mathbf f,\bm\ell\bigr)
    =\Bigl[\textstyle\sum_{\mathbf g\in\pi^{-1}(A(a))}m_{\mathbf g}\bigl(\mathbf f(\mathbf x)\bigr)\Bigr]
     \,\mathrm{softmax}\bigl(\bm\ell_{A(a)}(\mathbf x)\bigr)_a.
\end{equation}

\textbf{(Abduction.)} Let $a^\ast:=x_r^F$, $A^\ast:=A(a^\ast)$ and
$\mathbf f^F:=\mathbf f(\mathbf x^F_{\pa(r)})$. The posterior of $(\mathbf N_r,\mathbf G_r)$
factorises into three independent parts, each sampled exactly. (i) $\mathbf N_r$
is restricted to the union of boxes
$\bigcup_{\mathbf g\in\pi^{-1}(A^\ast)}(B_{\mathbf g}-\mathbf f^F)$:
\begin{equation}\label{eq:poset_abduct}
    \mathbf g\sim\frac{m_{\mathbf g}(\mathbf f^F)}{\sum_{\mathbf g'\in\pi^{-1}(A^\ast)}m_{\mathbf g'}(\mathbf f^F)},
    \quad
    N_{r,j}\mid\mathbf g\;\stackrel{\text{indep}}{\sim}\;\mathcal N(0,1)\ \text{truncated to}\
    \bigl(b_{j,g_j-1}-f^F_j,\;b_{j,g_j}-f^F_j\bigr).
\end{equation}
(ii) $(G_{r,a})_{a\in A^\ast}$ follow the posterior~\eqref{eq:gumbel_abduct} of a race
over $A^\ast$ won by $a^\ast$. (iii) The Gumbels of every other block keep their prior.

\textbf{(Prediction.)} With $\tilde{\mathbf f}:=\mathbf f(\tilde{\mathbf x}_{\pa(r)})$ and
$\tilde{\bm\ell}:=\bm\ell(\tilde{\mathbf x}_{\pa(r)})$,
\begin{equation}\label{eq:poset_cf}
    \tilde\Lambda_r=\pi\bigl(\kappa(\tilde{\mathbf f}+\mathbf N_r)\bigr),\qquad
    X_r^{\CF}=\argmax_{a\in\tilde\Lambda_r}\bigl\{\tilde\ell_{\tilde\Lambda_r,a}+G_{r,a}\bigr\},
\end{equation}
where $(\mathbf f^F,\tilde{\mathbf f})$ is drawn from the latent's joint Laplace
posterior, of dimension $2k$ (Lemma~\ref{lem:poset-laplace}), and
$(\bm\ell^F,\tilde{\bm\ell})$ from that of the block classifiers~\eqref{eq:joint_latent}.
\end{proposition}

\begin{corollary}[No order reversal; separate fits]
\label{cor:poset-monotone}
In the setting of Proposition~\ref{prop:unified}:
(i) for any realisation with $\tilde{\mathbf f}\ge\mathbf f^F$ coordinatewise,
$\tilde\Lambda_r\succeq A^\ast$ almost surely, so the counterfactual class is never
below the factual one;
(ii) the posterior of the latent factorises into one $k$-latent threshold fit on
all rows and one classifier per block with $\lvert A\rvert\ge2$, on that block's rows.
\end{corollary}

Every case studied in this paper is a choice of $k$, $\pi$, and the blocks
(Table~\ref{tab:types}); the proofs are in Appendix~\ref{sec:proof-general}. What
the special cases below add is a closed form or a fitting detail, not a new
argument. The order is an \emph{input}: a wrong one costs accuracy, and the fit
cannot detect it (\S\ref{sec:ordering}). Which posets are represented exactly is
characterised in Appendix~\ref{app:poset} (every lattice is, Lemma~\ref{lem:lattice-pi}).

\begin{table}[t]
\centering
\small
\setlength{\tabcolsep}{4pt}
\caption{Special cases of Proposition~\ref{prop:unified}. Each is fixed by the
number of directions $k$, the map $\pi$, and the blocks. Order dimension: the
fewest total orders whose intersection is the order. Results numbered A$n$
are stated and proved in the appendix.}
\label{tab:types}
\begin{adjustbox}{max width=\textwidth}
\begin{tabular}{llllll}
\hline
type & $k$ & blocks & noise & coupling & stated as\\
\hline
nominal & $0$ & one, all classes & $\mathbf G_r$ & Gumbel-max & Cor.~\ref{cor:nominal}\\
binary & $0$ or $1$ & either reading & $U_r$ & both coincide & Cor.~\ref{cor:binary}\\
ordinal & $1$, $\pi=\mathrm{id}$ & singletons & $N_r$ & comonotone & Cor.~\ref{cor:ordinal}\\
tiered & $1$ & the tiers & $(N_r,\mathbf G_r)$ & comonotone across, Gumbel-max within & Cor.~\ref{cor:tiered}\\
lattice & order dimension & singletons & $\mathbf N_r$ & comonotone per direction & \S\ref{sec:poset}\\
\hline
\end{tabular}
\end{adjustbox}
\end{table}
\end{revblock}

\subsection{\texorpdfstring{\cutnew{Nominal: a Gumbel-max race}}{Nominal: a Gumbel-max race}}

With $k=0$ there is one cell and one block holding all $C$ classes, so
Proposition~\ref{prop:unified} is a single Gumbel-max race. Written out with its
fit:

\begin{corollary}[Nominal nodes, $k=0$]
\label{cor:nominal}
Fit $C$ independent binary Laplace GP classifiers, one per class, and
renormalise: $p_c := \sigma(f_c)/\sum_{c'}\sigma(f_{c'})$ \cutnew{($\sigma$ the logistic function)}. (The joint softmax fit
is treated in \S\ref{sec:laplace-machinery}; Remark~\ref{rem:ovr} records the
difference.) Take as score the \emph{log-probabilities}
$\bm{\ell}(\mathbf{x}):=\log\mathbf{p}(\mathbf{x};\mathbf{f})$, and model
\begin{equation}\label{eq:gumbel_scm}
    X_r := \argmax_{c\in\{1,\ldots,C\}}
    \bigl\{ \ell_c(\mathbf{X}_{\pa(r)}) + G_{r,c} \bigr\},
    \qquad
    G_{r,c}\stackrel{iid}{\sim}\mathrm{Gumbel}(0,1).
\end{equation}

\textbf{(Compatibility.)} For every $\mathbf{x}$ and every realisation of the latent,
\begin{equation}\label{eq:gumbel_identity}
    P\bigl(X_r = c \mid \mathbf{X}_{\pa(r)}=\mathbf{x},\,\mathbf{f}\bigr)
    = \mathrm{softmax}\bigl(\log\mathbf{p}(\mathbf{x};\mathbf{f})\bigr)_c
    = p_c(\mathbf{x};\mathbf{f}).
\end{equation}
Because the coupling acts on $\log\mathbf{p}$ rather than on $\mathbf{f}$, this
holds for either fit.

\textbf{(Abduction.)} Let $\bm{\ell}^F := \bm{\ell}(\mathbf{x}^F_{\pa(r)})$,
$c^\ast := x_r^F$, and let $M:=\max_c\{\ell^F_c+G_{r,c}\}$ be the \emph{factual
winning score}. Given $\bm{\ell}^F$ and $X_r = c^\ast$, the posterior of
$\mathbf{G}_r$ is sampled exactly by
\begin{equation}\label{eq:gumbel_abduct}
  \begin{gathered}
    M \sim \mathrm{Gumbel}\Bigl(\log\textstyle\sum_{c}e^{\ell^F_c},\,1\Bigr),
    \qquad
    G_{r,c^\ast} := M - \ell^F_{c^\ast},
    \\
    G_{r,c} := -\log\bigl(e^{-(M-\ell^F_c)} - \log V_c\bigr),\quad c\neq c^\ast,
  \end{gathered}
\end{equation}
with $V_c\stackrel{iid}{\sim}\mathrm{Uniform}(0,1)$.

\textbf{(Prediction.)} With $\tilde{\bm{\ell}}:=\bm{\ell}(\tilde{\mathbf{x}}_{\pa(r)})$,
\begin{equation}\label{eq:gumbel_cf_given_f}
    X_r^{\CF} = \argmax_{c}\bigl\{\tilde\ell_c + G_{r,c}\bigr\}.
\end{equation}
For each class, $(f^F_c,\tilde f_c)$ is drawn from its bivariate Laplace
posterior~\eqref{eq:joint_latent} (Appendix~\ref{sec:laplace-machinery}),
independently across classes under the one-versus-rest fit, and the counterfactual
distribution is~\eqref{eq:general_cf} with $\mathbf s=\bm\ell$ and $U_r=\mathbf G_r$.
\end{corollary}

\subsection{\texorpdfstring{\cutnew{Ordinal: one shared threshold}}{Ordinal: one shared threshold}}
\label{sec:ordinal}

The Gumbel-max coupling of Corollary~\ref{cor:nominal} treats every pair of
classes symmetrically. For ordered classes we build instead on the
Gaussian-process ordinal-regression likelihood \citep{chu2005gaussian}: the
continuous additive-noise-model (ANM) GP approach of \cite{DBLP:journals/corr/abs-2006-06831}, observed
through a binned channel rather than directly.

With $k=1$, $\pi=\mathrm{id}$ and singleton blocks there is no race, and
Proposition~\ref{prop:unified} becomes a GP-ordinal model. Unknown cutpoints are
learned on the same footing as $f_r$, through an unconstrained parameterisation
$\mathbf b(\theta)$ with Gaussian prior $p(\theta)$
(Equation~\eqref{eq:theta_reparam} in Appendix~\ref{app:ordinal-proofs}); this adds an outer
posterior and gives the counterfactual a closed form.

\begin{corollary}[Ordinal nodes with learned cutpoints, $k=1$]
\label{cor:ordinal}
Model the structural equation as
\begin{equation}\label{eq:ordinal_scm}
\begin{split}
    Y_r &:= f_r(\mathbf{X}_{\pa(r)}) + N_r,\quad
    f_r\sim\mathcal{GP}(0,k_r),\quad N_r\sim\mathcal{N}(0,1),\quad \theta\sim p(\theta),\\
    X_r &:= \sum_{c=1}^{C} c\cdot\mathbb{I}\bigl\{b_{c-1}(\theta)\le Y_r< b_c(\theta)\bigr\},
\end{split}
\end{equation}
with $b_0(\theta):=-\infty$, $b_C(\theta):=\infty$, and $(f_r,N_r,\theta)$ mutually
independent a priori.

\textbf{(Compatibility.)} For every
$\mathbf{x}$ and every realisation of $(f_r,\theta)$,
\begin{equation}\label{eq:ordinal_identity}
    P\bigl(X_r=c\mid\mathbf{X}_{\pa(r)}=\mathbf{x},f_r,\theta\bigr)
    = \Phi\bigl(b_c(\theta)-f_r(\mathbf{x})\bigr)-\Phi\bigl(b_{c-1}(\theta)-f_r(\mathbf{x})\bigr).
\end{equation}

\textbf{(Abduction.)} Let $c^\ast:=x_r^F$, $f_r^F:=f_r(\mathbf{x}^F_{\pa(r)})$, and
$(a^F(\theta),b^F(\theta)):=(b_{c^\ast-1}(\theta),b_{c^\ast}(\theta))$. Given $\theta$,
$f_r^F$ and $X_r=c^\ast$, the abduction region is an interval and
\begin{equation}\label{eq:ordinal_abduct}
    N_r\mid\mathbf{x}^F,f_r^F,\theta,X_r=c^\ast
    \;\sim\; \mathcal{N}(0,1)\ \text{truncated to}\ \bigl(a^F(\theta)-f_r^F,\;b^F(\theta)-f_r^F\bigr).
\end{equation}

\textbf{(Prediction.)} With $\tilde f_r:=f_r(\tilde{\mathbf{x}}_{\pa(r)})$, the
inner expectation of~\eqref{eq:general_cf} has the closed form
\begin{equation}\label{eq:ordinal_cf}
    P\bigl(X_r^{\CF}=k\mid\mathbf{x}^F,f_r^F,\tilde f_r,\theta\bigr)
    = \frac{\bigl[\Phi\bigl(\min(\beta,\gamma_k)\bigr)-\Phi\bigl(\max(\alpha,\gamma_{k-1})\bigr)\bigr]_+}
           {\Phi(\beta)-\Phi(\alpha)},
\end{equation}
where $\alpha:=a^F(\theta)-f_r^F$, $\beta:=b^F(\theta)-f_r^F$, $\gamma_k:=b_k(\theta)-\tilde f_r$ and
$(\cdot)_+:=\max(0,\cdot)$. The pair $(f_r^F,\tilde f_r)$ is drawn from its joint
Laplace posterior given $\theta$ \citep{chu2005gaussian}, and $\theta$ from its own
Laplace posterior (Lemma~\ref{lem:theta-laplace}), so
\begin{equation}\label{eq:ordinal_cf_marginal}
    P\bigl(X_r^{\CF}=k\mid\mathbf{x}^F\bigr)
    = \mathbb{E}_{\theta}\Bigl[\mathbb{E}_{(f_r^F,\tilde f_r)\mid\theta}\bigl[\,\text{\eqref{eq:ordinal_cf}}\,\bigr]\Bigr],
\end{equation}
approximated by nested Monte Carlo, refitting the latent posterior at each
drawn $\mathbf b(\theta^{(s)})$.
\end{corollary}

\begin{corollary}[Fixed cutpoints]
\label{cor:ordinal-fixed}
If the cutpoints $\mathbf b=\mathbf b(\theta_0)$ are known, so that $p(\theta)=\delta_{\theta_0}$, the
outer expectation in~\eqref{eq:ordinal_cf_marginal} collapses, and
Corollary~\ref{cor:ordinal} reduces to a single-latent GP-ordinal fit at the
given $\mathbf b$, the case used throughout \citet{chu2005gaussian}.
\end{corollary}

\subsection{\texorpdfstring{\cutnew{Partial orders: one latent per direction}}{Partial orders: one latent per direction}}
\label{sec:poset}
\begin{revblock}
With singleton blocks and injective $\pi$, Proposition~\ref{prop:unified} is $k$
independent copies of Corollary~\ref{cor:ordinal}, one per direction
(Lemma~\ref{lem:poset-laplace}). A non-injective $\pi$ merges cells, as when every
cell at or above $(2,2)$ of a $3{\times}3$ grid is one top class
(the \emph{merged-top} grid, Figure~\ref{fig:posets}): the class probability~\eqref{eq:poset_identity} is then a
sum of box masses, and abduction~\eqref{eq:poset_abduct} first picks the box.
Tiered orders use $k=1$ with the tiers as blocks (Corollary~\ref{cor:tiered}), as
for the education variable of the census data used in \S\ref{sec:setup}.
\end{revblock}

\section{\texorpdfstring{\cutnew{Experiments: does the coupling matter?}}{Experiments: does the coupling matter?}}
\label{sec:scm_suite}

\subsection{\texorpdfstring{\cutnew{Benchmarks with exact ground truth}}{Benchmarks with exact ground truth}}
\label{sec:setup}\label{sec:core}
We evaluate on three benchmarks: five small synthetic SCMs whose mechanisms
lie in \cutnew{the model families of \S\ref{sec:results}} (Figure~\ref{fig:scm_levels}), random SCMs, and a
semi-synthetic SCM built on UCI Adult. The last two have mechanisms outside our
families. In every benchmark, the true structural equations are known, so exact
counterfactuals are available. We obtain them by sampling: abduction factorises
over nodes, so we sample each node's exact noise posterior, by inversion for
continuous nodes and by rejection for discrete ones, and push the noise through
the intervened equations. Each estimator is fitted to data sampled from the
true SCM, with no knowledge of its equations. We compare the true counterfactuals with those estimated by TACC
(Proposition~\ref{prop:unified}) and by each method below.

\rev{\textbf{Methods compared.} \emph{TACC} is Proposition~\ref{prop:unified}
instantiated with each node's true order type. We explore three variants that change only that
type, isolating what it contributes (\S\ref{sec:ordering}): \emph{all-categorical},
the type-agnostic model we started from (\S\ref{sec:intro}), discards every
order and fits each discrete node with the GP classifier and Gumbel-max coupling
of Corollary~\ref{cor:nominal}; \emph{imposed order} conversely fits a nominal
node with the ordinal model of Corollary~\ref{cor:ordinal}; and \emph{linear ext.}\
fits a partially ordered node with the same ordinal model, after breaking each tie
arbitrarily to obtain a total order (a linear extension). These variants let us evaluate the impact of disregarding
the order or imposing an arbitrary one. We also compare against three baselines
that naively carry continuous or relaxed machinery over to discrete nodes: \emph{ANM + round} treats the class
code as a real number, fits the additive-noise GP-SCM of
\citet{DBLP:journals/corr/abs-2006-06831}, and rounds its exactly abducted
counterfactual to the nearest code; \emph{one-hot GPR} fits one GP regression (GPR) per
class indicator and takes the argmax of the per-class counterfactuals; and
\emph{relaxed} keeps our estimator but replaces the exact Gumbel abduction with its
Gumbel-softmax relaxation at temperature $\tau=0.5$
\citep{jang2017categorical,maddison2017concrete}, which changes only binary and
nominal races. Because this relaxation is the usual way to handle categorical
noise when a model must be differentiable, the relaxed arm tests whether
abducting the noise exactly matters. All methods use the same GP fitting protocol
(Appendices~\ref{app:estimators} and~\ref{app:adult-setup}).}

\textbf{Metrics.} We report two total-variation (TV) distances:
$\mathrm{TV}_{\mathrm{int}}$, the interventional error of the fitted structural
equation, and $\mathrm{TV}_{\mathrm{cf}}$, the counterfactual error. For continuous
nodes we report the 1-Wasserstein distance over the node's standard deviation,
$W_1/\sigma$. Results are means over seeds (or graphs), and methods are compared
with two-sided Wilcoxon signed-rank tests paired by seed (or graph). Consistency
error is exactly zero in every row, as Corollary~\ref{cor:consistency} requires,
whether or not the estimator's assumed mechanism matches the truth.

\textbf{Synthetic SCMs.} \cutnew{The five structures mirror the loan scenario
(\S\ref{sec:intro}): an approve/deny decision (Binary), routing to a loan product
(Single-parent), a risk tier from debt-to-income and loan-to-value (Interaction),
and the full tier-to-product pipeline (Chain). In the Poset SCM, a lender grades
repayment capacity and collateral separately as low, medium, or high and pools
applicants at least medium on both into one top grade, so an applicant strong on
only one dimension is neither above nor below one strong only on the other;
Appendix~\ref{app:poset-exp} adds further poset designs.}

For each SCM and each of $10$ seeds we draw $n=500$ training samples and
evaluate on $8$ factual/intervention scenarios. The equations, ground-truth
computation, and fitting protocol are in Appendices~\ref{app:protocol}
and~\ref{app:scms}\cutnew{.} The Interaction and Chain SCMs use the ordinal mechanism of
Corollary~\ref{cor:ordinal} with true cutpoints $b_1=-0.5,\,b_2=0.5$, \cutnew{which TACC
learns by MAP; giving it the true cutpoints (Corollary~\ref{cor:ordinal-fixed})
changes $\mathrm{TV}_{\mathrm{cf}}$ by at most $0.004$ (Table~\ref{tab:learned-cut}).}

\rev{\textbf{Random SCMs and Adult.} The random SCMs are $30$ directed acyclic
graphs, ten each with $d=8$, $12$, and $16$ nodes, mixing all five types, with
$n=500$. Their mechanisms are drawn from families outside ours
(Appendix~\ref{app:mechanisms}); ordinal nodes are either thresholded with
non-Gaussian noise or generated stage by stage (continuation ratio). The Adult
SCM is built on the UCI Adult dataset \citep{kohavi1996scaling}: $48{,}842$
records of adults extracted from the 1994 US Census, a standard benchmark for
predicting whether a person's income exceeds \$50{,}000 a year and for
counterfactual explanation and recourse \citep{karimi2021algorithmic}. We use
nine of its attributes (sex, age, native country, marital status, education,
workclass, occupation, hours worked per week, and income) on the causal graph of
\citet{nabi2018fair}. Fitting gradient-boosted trees to the real records gives
mechanisms whose equations we know (Appendix~\ref{app:scms}), with education a
nine-class tiered order (Figure~\ref{fig:posets}). We sample $n=1000$ records from it and query three interventions:
\texttt{do(marital=married)}, \texttt{do(age{+}{=}5)}, and
\texttt{do(education=Bachelors)}. We also fit the estimators to the real Adult
records and compare their held-out likelihoods and counterfactuals.}

\begin{figure}[t]
\centering
\resizebox{\textwidth}{!}{%
\begin{tikzpicture}[node distance=8mm and 9mm, ->, >={Stealth[length=1.8mm]}]
  \node[var, continuous] (x1) {$X_1$};
  \node[var, binary, right=of x1] (x2) {$B_1$};
  \draw (x1) -- (x2);
  \node[below=4mm of x1, xshift=7mm] {\small Binary};
\end{tikzpicture}\hspace{7mm}
\begin{tikzpicture}[node distance=8mm and 9mm, ->, >={Stealth[length=1.8mm]}]
  \node[var, continuous] (x1) {$X_1$};
  \node[var, nominal, right=of x1] (x2) {$C_1$};
  \draw (x1) -- (x2);
  \node[below=4mm of x1, xshift=7mm] {\small Single-parent};
\end{tikzpicture}\hspace{7mm}
\begin{tikzpicture}[node distance=8mm and 9mm, ->, >={Stealth[length=1.8mm]}]
  \node[var, continuous] (x1) {$X_1$};
  \node[var, continuous, right=of x1] (x2) {$X_2$};
  \node[var, ordinal, below=of $(x1)!0.5!(x2)$] (x3) {$O_1$};
  \draw (x1) -- (x3);
  \draw (x2) -- (x3);
  \node[below=3mm of x3] {\small Interaction};
\end{tikzpicture}\hspace{7mm}
\begin{tikzpicture}[node distance=8mm and 9mm, ->, >={Stealth[length=1.8mm]}]
  \node[var, continuous] (x1) {$X_1$};
  \node[var, ordinal, right=of x1] (x2) {$O_1$};
  \node[var, nominal, right=of x2] (x3) {$C_1$};
  \draw (x1) -- (x2);
  \draw (x2) -- (x3);
  \draw (x1) to[bend left=35] (x3);
  \node[below=4mm of x2] {\small Chain};
\end{tikzpicture}\hspace{7mm}
\begin{tikzpicture}[node distance=8mm and 9mm, ->, >={Stealth[length=1.8mm]}]
  \node[var, continuous] (x1) at (0,0.55) {$X_1$};
  \node[var, continuous] (x2) at (0,-0.55) {$X_2$};
  \node[var, poset] (p) at (1.3,0) {$P_1$};
  \node[var, nominal] (y) at (2.6,0) {$C_1$};
  \draw (x1) -- (p);
  \draw (x2) -- (p);
  \draw (p) -- (y);
  \draw (x1) to[bend left=30] (y);
  \node at (1.3,-1.25) {\small \rev{Poset}};
\end{tikzpicture}}

\vspace{2mm}
\begin{tikzpicture}[baseline, every node/.style={font=\small}]
  \node[var, continuous, minimum size=3mm] (l1) at (0,0) {};
  \node[right=1mm of l1] (t1) {continuous};
  \node[var, binary, minimum size=3mm, right=6mm of t1] (l2) {};
  \node[right=1mm of l2] (t2) {binary};
  \node[var, nominal, minimum size=3mm, right=6mm of t2] (l3) {};
  \node[right=1mm of l3] (t3) {nominal};
  \node[var, ordinal, minimum size=3mm, right=6mm of t3] (l4) {};
  \node[right=1mm of l4] (t4) {ordinal};
  \node[var, poset, minimum size=3mm, right=6mm of t4] (l5) {};
  \node[right=1mm of l5] {\rev{partially ordered}};
\end{tikzpicture}

\caption{The synthetic SCM structures used in \S\ref{sec:scm_suite} (Binary, Single-parent, Interaction, \cutnew{Chain, and Poset}),
increasing in graph complexity and node type. Node fill
colour \rev{and letter} indicate each variable's type\rev{: $X$ continuous, $B$ binary,
$C$ nominal, $O$ ordinal, $P$ partially ordered}.}
\label{fig:scm_levels}
\end{figure}

\begin{table}[t]
\centering
\caption{\rev{Our estimator against the other methods on every node of the
synthetic SCMs (Figure~\ref{fig:scm_levels}) and Adult (random SCMs: Table~\ref{tab:large-main}):
interventional error
$\mathrm{TV}_{\mathrm{int}}$ of the fitted structural equation (averaged over the
factual and counterfactual parent configurations) and counterfactual error
$\mathrm{TV}_{\mathrm{cf}}$, one column each per method ($\mathrm{TV}_{\mathrm{int}}$ shaded); mean over $10$ seeds at
$n=500$ (with standard deviation, sd: Table~\ref{tab:scm_suite_full}). Adult: $n=1000$, under \texttt{do(marital=married)}
($M$), \texttt{do(age{+}{=}5)} ($A$), and \texttt{do(education=Bachelors)} ($E$)
(all nodes, with sd: Table~\ref{tab:adult}). Mediator: the parent through
which the intervention reaches the node, fitted by the column's estimator (NA: the
intervention acts on a parent directly). Relaxed
($\tau=0.5$) changes only binary and nominal races, so on an ordinal node with
continuous parents it is our estimator; it was not run on the Poset SCM or Adult.
Bold: within one sd of the lowest in the row, for each error separately.}}
\label{tab:scm_suite_results}
\small
\setlength{\tabcolsep}{4pt}
\begin{adjustbox}{max width=\textwidth}
\begin{tabular}{lll*{8}{c}}
\hline
SCM & Node & Mediator & \multicolumn{2}{c}{TACC} & \multicolumn{2}{c}{ANM + round} & \multicolumn{2}{c}{one-hot GPR} & \multicolumn{2}{c}{relaxed} \\
\cline{4-5}\cline{6-7}\cline{8-9}\cline{10-11}
 & & & \multicolumn{1}{c}{$\mathrm{TV}_{\mathrm{cf}}$} & \multicolumn{1}{c}{\tvi{$\mathrm{TV}_{\mathrm{int}}$}} & \multicolumn{1}{c}{$\mathrm{TV}_{\mathrm{cf}}$} & \multicolumn{1}{c}{\tvi{$\mathrm{TV}_{\mathrm{int}}$}} & \multicolumn{1}{c}{$\mathrm{TV}_{\mathrm{cf}}$} & \multicolumn{1}{c}{\tvi{$\mathrm{TV}_{\mathrm{int}}$}} & \multicolumn{1}{c}{$\mathrm{TV}_{\mathrm{cf}}$} & \multicolumn{1}{c}{\tvi{$\mathrm{TV}_{\mathrm{int}}$}}\\
\hline
Binary & $B_1$ & NA & $\mathbf{0.020}$ & \tvi{$\mathbf{0.024}$} & $0.123$ & \tvi{$0.086$} & $0.125$ & \tvi{$0.043$} & $\mathbf{0.032}$ & \tvi{$\mathbf{0.024}$}\\
\hline
Single-parent & $C_1$ & NA & $\mathbf{0.033}$ & \tvi{$\mathbf{0.033}$} & $0.238$ & \tvi{$0.222$} & $0.246$ & \tvi{$0.077$} & $0.070$ & \tvi{$\mathbf{0.033}$}\\
\hline
Interaction & $O_1$ & NA & $\mathbf{0.074}$ & \tvi{$\mathbf{0.060}$} & $0.208$ & \tvi{$0.093$} & $0.316$ & \tvi{$0.097$} & $=$ TACC & \tvi{$=$ TACC} \\
\hline
\multirow{2}{*}{Chain} & $O_1$ & NA & $\mathbf{0.031}$ & \tvi{$\mathbf{0.028}$} & $0.252$ & \tvi{$0.111$} & $0.422$ & \tvi{$0.078$} & $=$ TACC & \tvi{$=$ TACC} \\
 & $C_1$ & $O_1$ & $\mathbf{0.038}$ & \tvi{$\mathbf{0.051}$} & $0.290$ & \tvi{$0.287$} & $0.333$ & \tvi{$0.095$} & $0.060$ & \tvi{$\mathbf{0.051}$}\\
\hline
\multirow{2}{*}{Poset} & $P_1$ & NA & $\mathbf{0.114}$ & \tvi{$\mathbf{0.085}$} & $0.358$ & \tvi{$0.263$} & $0.277$ & \tvi{$0.159$} & --- & \tvi{---} \\
 & $C_1$ & $P_1$ & $\mathbf{0.057}$ & \tvi{$\mathbf{0.069}$} & $0.156$ & \tvi{$0.237$} & $0.139$ & \tvi{$0.103$} & --- & \tvi{---} \\
\hline
Adult, $M$ & education & NA & $\mathbf{0.094}$ & \tvi{$\mathbf{0.087}$} & $0.181$ & \tvi{$0.379$} & $0.215$ & \tvi{$0.138$} & --- & \tvi{---} \\
Adult, $A$ & education & marital & $\mathbf{0.210}$ & \tvi{$\mathbf{0.087}$} & $0.393$ & \tvi{$0.379$} & $0.383$ & \tvi{$0.138$} & --- & \tvi{---} \\
Adult, $E$ & occupation & NA & $\mathbf{0.075}$ & \tvi{$\mathbf{0.096}$} & $0.507$ & \tvi{$0.226$} & $0.410$ & \tvi{$0.142$} & --- & \tvi{---} \\
\hline
\end{tabular}
\end{adjustbox}
\end{table}

\subsection{\texorpdfstring{\cutnew{Accuracy against baselines}}{Accuracy against baselines}}\label{sec:findings}

\rev{\textbf{TACC is the most accurate method on every node}
(Table~\ref{tab:scm_suite_results}). On the synthetic SCMs, TACC has a lower
counterfactual error than both discretised baselines on all $10$ seeds of every
row (Wilcoxon $p=0.002$). The error of ANM + round is $2.8$--$8.2\times$ that of
TACC, and the error of one-hot GPR is $2.4$--$14\times$. Both baselines also fit
the interventional distribution worse: their $\mathrm{TV}_{\mathrm{int}}$ ranges
from $0.04$ to $0.29$, against $0.02$ to $0.09$ for TACC. On nominal nodes, ANM +
round also depends on an arbitrary numeric coding of the classes
(Table~\ref{tab:negative}). \cutnew{The relaxed estimator shares TACC's fit, but its
inexact abduction raises the counterfactual error $1.6$--$2.1\times$ on binary and
nominal nodes and breaks consistency, with errors of $0.06$--$0.16$ where TACC's is
exactly zero (Table~\ref{tab:relaxed}).}}

\begin{table}[t]
\centering
\caption{\rev{Random SCMs ($n=500$, $d=12$): $\mathrm{TV}_{\mathrm{cf}}$ by
descendant type (for continuous $X$, $W_1/\sigma$: the 1-Wasserstein distance over the node's sd), mean over the $10$ graphs,
whose mechanisms all lie outside our families. Results for $d=8$ and $16$, sds,
and an in-family control are in Table~\ref{tab:large-full}; ordinal and nominal nodes by true
mechanism in Table~\ref{tab:large-mech}. Bold: within one sd of the lowest in
the row.}}
\label{tab:large-main}
\small
\setlength{\tabcolsep}{4pt}
\begin{tabular}{lccccc}
\hline
type & TACC & linear ext. & ANM + round & one-hot GPR & relaxed \\
\hline
$B$ & $\mathbf{0.028}$ & $\mathbf{0.028}$ & $0.069$ & $0.073$ & $0.083$ \\
$C$ & $\mathbf{0.049}$ & $\mathbf{0.051}$ & $0.132$ & $0.126$ & $0.128$ \\
$O$ & $\mathbf{0.123}$ & $\mathbf{0.126}$ & $\mathbf{0.164}$ & $0.210$ & $\mathbf{0.136}$ \\
$P$ & $\mathbf{0.064}$ & $0.159$ & $0.169$ & $0.156$ & $\mathbf{0.093}$ \\
$X$ & $\mathbf{0.097}$ & $\mathbf{0.108}$ & $\mathbf{0.128}$ & $\mathbf{0.138}$ & $\mathbf{0.133}$ \\
\hline
\end{tabular}
\end{table}

\rev{\textbf{The advantage holds outside our families.}\label{sec:additional}\label{sec:large}\label{sec:adult}\label{sec:poset-exp}
On the random SCMs, TACC's counterfactual error is below both discretised
baselines' on every discrete type ($p\le0.04$) and the relaxed estimator's on
every type but partially ordered ($p<0.001$; Table~\ref{tab:large-main}). Over all discrete nodes it is
$0.068$, against $0.077$ for the next best, all-categorical. Linear ext.\ fits
the same models as TACC except on partially ordered nodes, so its overall error
($0.089$) hides a large gap there: pooled over all $30$ graphs, its error on
these nodes is $0.195$, $3.1\times$ TACC's $0.063$. TACC's errors match an
in-family control \cutnew{(random SCMs with mechanisms from our own families;} Table~\ref{tab:large-full}\cutnew{)} except on continuation-ratio
ordinals, where one-hot GPR and ANM + round do better ($0.112$ and $0.126$
against $0.139$; Table~\ref{tab:large-mech}), as does all-categorical
(\S\ref{sec:ordering}). On Adult, education's counterfactual error is
$45$--$56\%$ below either discretised baseline's, and coding classes as numbers
fails badly on its children (Table~\ref{tab:adult-main}). On further poset designs, TACC cuts the best
alternative's error by $44$--$65\%$ on the poset node and $33$--$49\%$ on its
nominal child (Table~\ref{tab:poset}). Linear ext.\ matches TACC only where tied
classes are rare (Adult's education, $15\%$ of records); elsewhere its error is
$1.8$--$3.4\times$ TACC's.}

\subsection{\texorpdfstring{\cutnew{A good fit can hide a wrong coupling}}{A good fit can hide a wrong coupling}}
\label{sec:ordering}
\begin{revblock}
Table~\ref{tab:ordering} compares TACC with all-categorical on every ordered node
of every experiment, and with imposed order on a nominal node.

\begin{table}[t]
\centering
\caption{The value of the ordering. TACC against the same model with the node's
order type swapped: all-categorical (Gumbel-max, order discarded) on ordered
nodes, and an imposed order (Corollary~\ref{cor:ordinal}) on the nominal
Single-parent node. Mean over $10$ seeds ($n=500$; Adult $n=1000$), or over the
random graphs containing such a node (pooled over $d$). Ratio:
$\mathrm{TV}_{\mathrm{cf}}$ swapped over TACC. Wins: seeds (graphs) on which TACC has
the lower $\mathrm{TV}_{\mathrm{cf}}$; $p$: two-sided Wilcoxon signed-rank test, paired
by seed (graph). Bold: the lower of each pair. Further poset designs: Table~\ref{tab:ordering-poset}.}
\label{tab:ordering}
\small
\setlength{\tabcolsep}{3pt}
\begin{adjustbox}{max width=\textwidth}
\begin{tabular}{lllcccccrr}
\hline
 & & & \multicolumn{2}{c}{$\mathrm{TV}_{\mathrm{cf}}$} & \multicolumn{2}{c}{\tvi{$\mathrm{TV}_{\mathrm{int}}$}} & & & \\
experiment & node & true type & TACC & swapped & \tvi{TACC} & \tvi{swapped} & ratio & wins & $p$ \\
\hline
Single-parent & $C_1$ & nominal & $\mathbf{0.033}$ & $0.172$ & \tvi{$\mathbf{0.033}$} & \tvi{$0.193$} & $5.2$ & $10/10$ & $0.002$ \\
Interaction & $O_1$ & ordinal & $\mathbf{0.074}$ & $0.237$ & \tvi{$\mathbf{0.060}$} & \tvi{$0.084$} & $3.2$ & $10/10$ & $0.002$ \\
Chain & $O_1$ & ordinal & $\mathbf{0.031}$ & $0.275$ & \tvi{$\mathbf{0.028}$} & \tvi{$0.039$} & $8.9$ & $10/10$ & $0.002$ \\
 & $C_1$ (via $O_1$) & nominal & $\mathbf{0.038}$ & $0.087$ & \tvi{$0.051$} & \tvi{$0.051$} & $2.3$ & $10/10$ & $0.002$ \\
Poset, merged-top & $P_1$ & partial & $\mathbf{0.114}$ & $0.173$ & \tvi{$\mathbf{0.085}$} & \tvi{$0.099$} & $1.5$ & $10/10$ & $0.002$ \\
 & $C_1$ (via $P_1$) & nominal & $\mathbf{0.057}$ & $0.074$ & \tvi{$0.069$} & \tvi{$0.069$} & $1.3$ & $10/10$ & $0.002$ \\
Random SCMs & threshold & ordinal & $\mathbf{0.082}$ & $0.171$ & \tvi{$0.059$} & \tvi{$0.059$} & $2.1$ & $14/16$ & $<0.001$ \\
 & continuation ratio & ordinal & $0.139$ & $\mathbf{0.067}$ & \tvi{$0.063$} & \tvi{$\mathbf{0.061}$} & $0.5$ & $4/20$ & $<0.001$ \\
 & all & partial & $\mathbf{0.063}$ & $0.123$ & \tvi{$\mathbf{0.059}$} & \tvi{$0.076$} & $1.9$ & $18/20$ & $<0.001$ \\
Adult, \texttt{do(marital)} & education & tiered & $\mathbf{0.094}$ & $0.208$ & \tvi{$\mathbf{0.087}$} & \tvi{$0.142$} & $2.2$ & $9/10$ & $0.004$ \\
Adult, \texttt{do(age{+}{=}5)} & education & tiered & $\mathbf{0.210}$ & $0.356$ & \tvi{$\mathbf{0.087}$} & \tvi{$0.142$} & $1.7$ & $10/10$ & $0.002$ \\
\hline
\end{tabular}
\end{adjustbox}
\end{table}

We find that knowing the order is worth a factor of $1.5$--$9$ in counterfactual error
\cutnew{on every ordered node of Table~\ref{tab:ordering} except the random SCMs'
continuation-ratio ordinals (\S\ref{sec:findings})}, and
the fit does not reveal it\cutnew{: on the synthetic ordinal nodes
$\mathrm{TV}_{\mathrm{int}}$ differs by at most $0.024$ (Figure~\ref{fig:int-vs-cf} in
Appendix~\ref{app:ablations}), and TACC wins on at least $9$ of $10$ seeds in every
synthetic row}. On the \emph{real} Adult data, all-categorical even has slightly
higher held-out likelihood ($-1.747$ against $-1.781$ nats), while its
counterfactuals differ from TACC's by $\mathrm{TV}=0.38$ (Table~\ref{tab:adult-real}).
These errors also propagate to correctly specified children: the Chain and Poset
SCMs' nominal children have identical fits under both models, yet their
counterfactual error is $2.3\times$ and $1.3\times$ lower under TACC, because they
are fed a more accurate counterfactual of their ordered parent.

Imposing a false order, unlike discarding a true one, also degrades the fit, and
an order helps only when its type matches how the variable is generated
(Appendix~\ref{app:ablations}).

Because the error comes from the coupling, more data does not remove it. From
$n=100$ to $1000$ (Table~\ref{tab:sizes}), $\mathrm{TV}_{\mathrm{int}}$ and
$\mathrm{TV}_{\mathrm{cf}}$ fall together under the correct coupling, but under the
wrong one $\mathrm{TV}_{\mathrm{cf}}$ flattens near $0.21$--$0.24$ while
$\mathrm{TV}_{\mathrm{int}}$ keeps improving. Nor is the GP responsible. Ablating
our estimator (Table~\ref{tab:baselines}) by skipping abduction (predicting
interventionally), abducting only the latent, or plugging in the posterior mean
shows that abduction does the work, and an \emph{oracle} given the true class
probabilities with the same coupling recovers only a tenth of the
misspecification error at the Interaction SCM.

\begin{table}[t]
\centering
\caption{\rev{Calibration and partial identification, mean over $10$ seeds
($8$ queries each) at $n=500$ (with sd and details:
Tables~\ref{tab:calibration} and~\ref{tab:bounds}). Interval: central $90\%$
posterior interval of each per-class counterfactual probability, and the
fraction of exact probabilities it covers. Log score: cross-entropy
$H(p_{\mathrm{true}},\hat p)$ of the point estimate, from posterior averaging
(full) or the posterior-mean latent (plug-in); ``truth'' is its minimum. Bounds:
sharp bounds over all couplings, at the true class probabilities (oracle) or
widened by the GP posterior. $^\dagger$Coupling misspecified.}}
\label{tab:calibration-main}
\small
\setlength{\tabcolsep}{3pt}
\begin{adjustbox}{max width=\textwidth}
\newcolumntype{S}{>{\columncolor{altcolbg}}c}%
\begin{tabular}{l|Sc|ScS|c|Sc}
\hline
 & \multicolumn{2}{c|}{$90\%$ interval} & \multicolumn{3}{c|}{log score (nats)} & \multicolumn{1}{c|}{oracle bounds} & \multicolumn{2}{c}{posterior bounds} \\
fit & coverage & width & full & plug-in & truth & width & coverage & width \\
\hline
Single-parent $C_1$, Gumbel-max & $0.896$ & $0.072$ & $0.499$ & $0.500$ & $0.491$ & $0.581$ & $1.000$ & $0.650$ \\
Interaction $O_1$, Gumbel-max$^\dagger$ & $0.508$ & $0.122$ & $0.738$ & $0.783$ & $0.391$ & $0.711$ & $0.954$ & $0.787$ \\
Interaction $O_1$, ordinal (TACC) & $0.896$ & $0.155$ & $0.431$ & $0.438$ & $0.391$ & $0.711$ & $0.975$ & $0.777$ \\
\hline
\end{tabular}
\end{adjustbox}
\end{table}

\rev{Finally, the GP posterior cannot detect a wrong coupling
(Table~\ref{tab:calibration-main}). Drawing the latents from the posterior and
computing the counterfactual distribution under each draw gives credible
intervals. With the right coupling, $90\%$ intervals cover the exact
counterfactual probabilities $89.6\%$ of the time; under misspecification,
coverage falls to $51\%$ with intervals no wider, because the posterior describes
uncertainty about the class probabilities, not about the coupling. Sharp
Fr\'echet--Hoeffding bounds over \emph{all} couplings assume nothing about it and
always contain the truth, but are $0.58$--$0.71$ wide even at the true
probabilities: the data barely constrain the counterfactual, and the type
assumption is what makes it usable.}
\end{revblock}

\section{\texorpdfstring{\cutnew{Choose couplings by type, not objective}}{Choose couplings by type, not objective}}
\label{others}
 
Our nominal and ordinal models (Corollaries~\ref{cor:nominal}
and~\ref{cor:ordinal}) agree on every observational
and interventional quantity and differ only in the coupling
(Remark~\ref{rem:couplings}). \citet{lorberbom2021learning} choose the coupling by
optimising a downstream objective; our results argue that an objective alone is not
enough, since a good fit does not imply a correct counterfactual. We instead let the variable's type, known before any objective is
specified, constrain the admissible couplings, and find this constraint the
more consequential of the two: the gap between the type-appropriate coupling and a mismatched one
dwarfs the gains reported from tuning within a single family.

Two limitations follow. First, our validation is against synthetic, random, and semi-synthetic SCMs
with known mechanisms; on real data the exact mechanism is not observable,
though the gap is narrower than it sounds: choosing the coupling only
requires knowing the variable's \emph{type} --- whether its categories
carry an order --- not its structural equation, and type is ordinary
domain knowledge even when the
mechanism is not. \rev{The coupling also depends on \emph{how} the order arises: the comonotone coupling assumes one latent propensity crossed by thresholds, which a staged (continuation-ratio) order \citep{tutz1991sequential} violates. Stage-wise ordinals could fit the framework by chaining the binary abduction of Corollary~\ref{cor:binary} over reached stages; we leave this sequential coupling to future work.} Second, the constraint is structural:
\citet{lorberbom2021learning} show that no mechanism that abducts noise once and
reuses it across interventions can be optimal simultaneously for every pair
of interventional distributions, so neither coupling is uniformly best,
each being the natural choice for its variable type. Extending the construction to latent functions
shared across correlated ordinal nodes remains a natural next step.
\section{Conclusion}

We gave exact noise abduction in GP-SCMs for binary, nominal, ordinal, and
partially ordered variables, as special cases of one construction, each reproducing its fitted model's observational
and interventional distributions. The coupling is central to counterfactual
accuracy: a good interventional fit alone does not imply accurate
counterfactuals. Under the wrong coupling, the counterfactual error does not
shrink with more data, and goodness-of-fit reveals it only when a false order is
imposed on nominal data. A variable's type, known before any data is seen,
should therefore guide the choice of coupling\rev{, as it improves accuracy outside our
families and on semi-synthetic census data too}. TACC can thus be adopted when
building causal models of real-world systems with mixed variable types.

\subsubsection*{Acknowledgments}
This research was supported in part by the Natural Sciences and Engineering Research Council of Canada (NSERC) under grant to both MK and AHK.

\bibliography{arxiv_bib}

\subsection*{AI use statement}

In this work, we used generative AI tools \cutnew{(Claude, by Anthropic)} to help develop the theoretical
framework and propositions, formulate and refine mathematical claims,
provide critical ingredients for their proofs and assist in writing them,
propose and refine hypotheses, design and give feedback on the experimental
methodology, implement the counterfactual estimators and experiments,
generate the synthetic SCMs used for evaluation, and interpret the
resulting numbers. Assisting with translation, cleaning or reformatting a
dataset, and supporting qualitative or thematic data analysis are not
applicable to this work. Additionally, we used \cutnew{the same tools} to draft
and edit sections of the paper for readability and structure, create and
edit the accompanying code, create and modify the paper's figures, search
for and summarise related literature, and format references. We have
reviewed all AI-assisted work: code was read and tested by the authors for
correctness, and all AI-suggested claims, proofs, and related work were
checked by the authors against the literature and re-derived or verified
independently before inclusion. We take responsibility for the final
content of this work, including all text, claims, and artefacts produced
with the aid of generative AI.


\appendix
\setcounter{table}{0}\renewcommand{\thetable}{A\arabic{table}}\renewcommand{\theHtable}{A\arabic{table}}
\setcounter{figure}{0}\renewcommand{\thefigure}{A\arabic{figure}}\renewcommand{\theHfigure}{A\arabic{figure}}
\setcounter{algorithm}{0}\renewcommand{\thealgorithm}{A\arabic{algorithm}}\renewcommand{\theHalgorithm}{A\arabic{algorithm}}
\setcounter{theorem}{0}\renewcommand{\thetheorem}{A\arabic{theorem}}\renewcommand{\theHtheorem}{A\arabic{theorem}}

\section{Gaussian process background}
\label{app:background}\label{app:defs}
We write $f\sim\mathcal{GP}(m,k)$ for a Gaussian process with mean function
$m(\mathbf{x})=\E[f(\mathbf{x})]$ and covariance function
$k(\mathbf{x},\mathbf{x}')=\Cov[f(\mathbf{x}),f(\mathbf{x}')]$, and take $m\equiv0$
throughout \citep{Rasmussen_2006}. For regression with
$y_i=f(\mathbf{x}_i)+\varepsilon_i$, $\varepsilon_i\sim\mathcal{N}(0,\sigma_n^2)$, the
posterior at a test input $\mathbf{x}_*$ is Gaussian with
\begin{equation}
    \bar{f}_* = \mathbf{k}_*^\top(K + \sigma_n^2 I)^{-1}\mathbf{y},
    \qquad
    \Var[f_*] = k(\mathbf{x}_*, \mathbf{x}_*) - \mathbf{k}_*^\top(K + \sigma_n^2 I)^{-1}\mathbf{k}_*,
    \label{eq:gp-post-scalar}
\end{equation}
where $K:=k(X,X)$ and $\mathbf{k}_*:=k(X,\mathbf{x}_*)$
\citep[eqs.~2.25--2.26]{Rasmussen_2006}. This is the model fitted to continuous
nodes. For discrete nodes the likelihood is not Gaussian, and the posterior is
replaced by the Laplace approximation of \S\ref{sec:laplace-machinery}, whose
predictive variance \eqref{eq:pred_moments} has the same form with
$\sigma_n^2I$ replaced by $\hat W_c^{-1}$.

\section{Algorithms}
\label{app:algorithms}

Algorithms~\ref{alg:gumbel}, \ref{alg:ordinal}, \rev{and~\ref{alg:poset}} give the
estimators as implemented. All follow one skeleton: fit, compute the joint
posterior moments at the factual and counterfactual parents, then repeat
abduction, action, and prediction inside a Monte Carlo loop over the latent. They
differ only in the exogenous noise and how it is abducted.

\begin{algorithm}
\caption{Counterfactual for a nominal node (Corollary~\ref{cor:nominal})}
\label{alg:gumbel}
\begin{algorithmic}[1]
\Require sample $\{(\mathbf{x}^i_{\pa(r)},x^i_r)\}_{i=1}^n$; factual
  $(\mathbf{x}^F_{\pa(r)},x^F_r=c^\ast)$; intervention $\tilde{\mathbf{x}}_{\pa(r)}$;
  draws $S$
\Ensure $\hat P(X_r^{\CF}=c\mid\mathbf{x}^F)$ for $c=1,\dots,C$
\Statex \emph{Training}
\For{$c=1,\dots,C$}
  \State fit a binary Laplace GP classifier for class $c$ against the rest
    \Comment{one-versus-rest; Remark~\ref{rem:ovr}}
\EndFor
\Statex \emph{Posterior moments} --- independent of the draw, so computed once
\For{$c=1,\dots,C$}
  \State $(\mu^F_c,s^{F2}_c,\tilde\mu_c,\tilde s^2_c,\rho_c)\gets$ joint two-point
    Laplace moments at $(\mathbf{x}^F_{\pa(r)},\tilde{\mathbf{x}}_{\pa(r)})$,
    \eqref{eq:joint_latent}
\EndFor
\State $\mathbf{n}\gets\mathbf{0}\in\R^C$
\For{$s=1,\dots,S$} \Comment{marginalise over the GP posterior}
  \For{$c=1,\dots,C$}
    \State draw $(f^F_c,\tilde f_c)$ jointly from the bivariate Gaussian above
      \Comment{one joint draw per class}
  \EndFor
  \State $\bm\ell^F\gets\log\mathrm{norm}\bigl(\sigma(\mathbf{f}^F)\bigr)$,\quad
         $\tilde{\bm\ell}\gets\log\mathrm{norm}\bigl(\sigma(\tilde{\mathbf{f}})\bigr)$
    \Comment{coupling lives on $\log\mathbf{p}$}
  \Statex \quad\emph{Abduction} --- \eqref{eq:gumbel_abduct}
  \State $V\sim\mathrm{Uniform}(0,1)$;\quad
         $M\gets\log\sum_c e^{\ell^F_c}-\log(-\log V)$
    \Comment{the factual winning score}
  \State $b_c\gets M-\ell^F_c$;\quad $E_c\stackrel{iid}{\sim}\mathrm{Exp}(1)$
  \State $G_{c^\ast}\gets b_{c^\ast}$;\quad
         $G_c\gets-\operatorname{logaddexp}(-b_c,\log E_c)$ for $c\neq c^\ast$
  \Statex \quad\emph{Action and prediction}
  \State $k\gets\argmax_c\{\tilde\ell_c+G_c\}$;\quad $n_k\gets n_k+1$
\EndFor
\State \Return $\mathbf{n}/S$
\end{algorithmic}
\end{algorithm}

\begin{algorithm}
\caption{Counterfactual for an ordinal node with learned cutpoints (Corollary~\ref{cor:ordinal})}
\label{alg:ordinal}
\begin{algorithmic}[1]
\Require sample; cutpoint prior $p(\theta)$ (\eqref{eq:theta_reparam}); factual
  $(\mathbf{x}^F_{\pa(r)},x^F_r=c^\ast)$; intervention $\tilde{\mathbf{x}}_{\pa(r)}$;
  outer draws $S_\theta$; inner draws $S_f$
\Ensure $\hat P(X_r^{\CF}=k\mid\mathbf{x}^F)$ for $k=1,\dots,C$
\Statex \emph{Training}
\State fit $\hat\theta,\Sigma_\theta$ by Laplace-approximating the profiled
  cutpoint posterior \Comment{Lemma~\ref{lem:theta-laplace}}
\State $\mathbf{q}\gets\mathbf{0}\in\R^C$
\For{$t=1,\dots,S_\theta$} \Comment{marginalise over the cutpoint posterior}
  \State draw $\theta^{(t)}\sim\mathcal N(\hat\theta,\Sigma_\theta)$; set
    $b_c\gets b_c(\theta^{(t)})$ via \eqref{eq:theta_reparam}
  \State Laplace-fit one latent GP under~\eqref{eq:ordinal_identity} at these cutpoints
    \Comment{Lemma~\ref{lem:ordinal-map}}
  \State $(\mu^F,s^{F2},\tilde\mu,\tilde s^2,\rho)\gets$ joint two-point moments at
    $(\mathbf{x}^F_{\pa(r)},\tilde{\mathbf{x}}_{\pa(r)})$ under this fit
  \For{$s=1,\dots,S_f$} \Comment{marginalise over the GP posterior}
    \State draw $(f^F,\tilde f)$ jointly from the bivariate Gaussian above
    \Statex \qquad\emph{Abduction} --- the factual class confines $N_r$ to one
      interval, \eqref{eq:ordinal_abduct}
    \State $\alpha\gets b_{c^\ast-1}-f^F$;\quad $\beta\gets b_{c^\ast}-f^F$
    \Statex \qquad\emph{Action and prediction} --- \eqref{eq:ordinal_cf}, $N_r$
      integrated out in closed form
    \For{$k=1,\dots,C$}
      \State $q_k\gets q_k+\bigl[\Phi(\min(\beta,\,b_k-\tilde f))
        -\Phi(\max(\alpha,\,b_{k-1}-\tilde f))\bigr]_+\big/
        \bigl(\Phi(\beta)-\Phi(\alpha)\bigr)\big/S_f$
    \EndFor
  \EndFor
\EndFor
\State \Return $\mathbf{q}/S_\theta$
\end{algorithmic}
\vspace{2pt}
{\footnotesize Fixing $p(\theta)$ to a point mass at some $\theta_0$ (so $S_\theta=1$ and
$\theta^{(1)}=\theta_0$ deterministically) recovers the fixed-cutpoint algorithm of
Corollary~\ref{cor:ordinal-fixed} exactly, the Training step reducing to a single Laplace fit
at the given $\mathbf b(\theta_0)$.\par}
\end{algorithm}

\begin{revblock}
\begin{algorithm}
\caption{Counterfactual for a partially ordered node (Proposition~\ref{prop:unified})}
\label{alg:poset}
\begin{algorithmic}[1]
\Require sample; $\pi:\mathcal G\to P$; factual $(\mathbf x^F_{\pa(r)},x^F_r=p^\ast)$;
  counterfactual parents $\tilde{\mathbf x}_{\pa(r)}$ (one row per draw); draws $S$
\State fit the $k$-latent Laplace posterior and maximum a posteriori (MAP) cutpoints
  \Comment{Lemma~\ref{lem:poset-laplace}; damped Newton}
\For{$s=1,\dots,S$}
  \State draw $(\mathbf f^F,\tilde{\mathbf f})\in\R^{2k}$ jointly from~\eqref{eq:poset_pred}
  \State draw a cell $\mathbf g\in\pi^{-1}(p^\ast)$ with probability $\propto m_{\mathbf g}(\mathbf f^F)$
  \State $N_j\sim\mathcal N(0,1)$ truncated to $(b_{j,g_j-1}-f^F_j,\,b_{j,g_j}-f^F_j)$, $j=1..k$
  \State $x^{\CF,(s)}\gets\pi\bigl(\kappa(\tilde{\mathbf f}+\mathbf N)\bigr)$
\EndFor
\State \Return the empirical law of $\{x^{\CF,(s)}\}$
\end{algorithmic}
\end{algorithm}
\noindent The tiered-order estimator of Corollary~\ref{cor:tiered} composes
Algorithm~\ref{alg:ordinal} for the level with Algorithm~\ref{alg:gumbel} for
the race. If the counterfactual level differs from the factual one, the race
of that level is drawn from its interventional law.
\end{revblock}

\textbf{Implementation details.} Two points are easy to get wrong and are not
visible in the equations. First, the pair $(f^F_c,\tilde f_c)$ on line~10 of
Algorithm~\ref{alg:gumbel} must be a \emph{single} draw from the joint posterior.
Drawing the two marginally discards the covariance~\eqref{eq:joint_latent} and
breaks the exact consistency of Corollary~\ref{cor:consistency}; the same holds
for every algorithm above. When a response curve is traced over a grid of
interventions, the draws should also share one random-number stream across grid
points (Remark~\ref{rem:crn}). Second, line~15 uses $\operatorname{logaddexp}$
rather than $-\log(e^{-b_c}-\log V_c)$. The two agree in exact arithmetic, but
the naive form loses precision for classes that were probable factually yet lost
the $\argmax$, which are the individuals whose counterfactuals are most
interesting.

\textbf{Multi-node SCMs.} The algorithms are applied per node in topological
order. Each node abducts its own noise from its own factual parents and outcome,
and a parent's sampled counterfactual is the next node's intervened input.

\FloatBarrier
\section{Proofs}
\label{app:proofs}

Throughout, $\{\mathbf{x}^i\}_{i=1}^n$ denotes the observational sample, $\mathbf{x}^F$
the factual individual, and where $\mathbf{x}^F\in\{\mathbf{x}^i\}_{i=1}^n$ we write
$F\in[n]$ for its index. We write $\mathbf{e}_F\in\R^n$ for the
$F$-th standard basis vector, $\mathbf{1}_n$ for the all-ones vector, and reserve
$\tilde{\mathbf{x}}_{\pa(r)}$ for the counterfactual parent configuration. All statements are made
for a fixed $r\in[d]$ with non-empty parent set $\lvert \pa(r)\rvert>0$; the subscript $r$ is
suppressed wherever no ambiguity arises.
 
For the discrete results of \S\ref{sec:proof-prop-gumbel}, $\bm{\ell}=(\ell_1,\ldots,\ell_C)$ denotes a vector of
\emph{log-probability scores} entering an $\argmax$ over classes, and $M$ denotes the value of
that maximum. The Gumbel vector $\mathbf{G}_r$ is the exogenous noise of node $r$; $U$ is
reserved for the scalar uniform noise of the binary case (Corollary~\ref{cor:binary}).

\subsection{\rev{Proof of Proposition~\ref{prop:general} and Corollary~\ref{cor:consistency}}}
\label{sec:proof-general}
\begin{revblock}
\begin{proof}[Proof of Proposition~\ref{prop:general}]
\emph{Compatibility.} $\doop(\mathbf X_{\pa(r)}=\mathbf x)$ removes only the equations for
the parents, and $U_r$ is independent of the latent and of everything upstream, so
$P\bigl(X_r=c\mid\doop(\mathbf X_{\pa(r)}=\mathbf x),\mathbf s\bigr)=P_U\bigl(h(\mathbf s(\mathbf x),U_r)=c\bigr)=p_c(\mathbf s(\mathbf x))$.
The observational case is the same computation without the intervention.

\emph{Abduction.} Given $\mathbf s^F$, the observation $X_r=x_r^F$ is the event
$\{U_r\in A^F\}$, which has probability $p_{x_r^F}(\mathbf s^F)>0$ whenever the observed
class has positive fitted probability. $U_r$ is independent of $\mathbf s^F$, so by
Bayes' rule its posterior is $P_U$ restricted to $A^F$ and renormalised.

\emph{Action and prediction.} The intervention changes only the parents, so the latent
and $U_r$ are unchanged and $X_r^{\CF}=h(\tilde{\mathbf s},U_r)$. The latent is known
only through its posterior, so $(\mathbf s^F,\tilde{\mathbf s})$ is drawn from it, and the
tower rule over $U_r\mid A^F$ and then over $(\mathbf s^F,\tilde{\mathbf s})$
gives~\eqref{eq:general_cf}. The draw must be joint: $\mathbf s^F$ and $\tilde{\mathbf s}$ are
one function evaluated at two inputs, and $U_r$ is abducted from $\mathbf s^F$. The
integrand is a bounded indicator, so the Monte Carlo estimator is unbiased and converges
almost surely by the strong law of large numbers. Each draw is a vertex of the simplex,
so the empirical average is a valid distribution at every sample size. Nested averages,
such as the one over cutpoints in Corollary~\ref{cor:ordinal}, inherit both properties
level by level.
\end{proof}

\begin{proof}[Proof of Corollary~\ref{cor:consistency}]
At $\tilde{\mathbf x}_{\pa(r)}=\mathbf x^F_{\pa(r)}$ the latent is evaluated twice at the same
input. Under the Gaussian joint posterior this is $\Var[f^F-\tilde f]=0$
(\eqref{eq:pred_cov} with $\mathbf k_\star=\mathbf k_\dagger$, and its analogues in
Lemmas~\ref{lem:ordinal-map} and~\ref{lem:poset-laplace}), while fitted parameters such as
cutpoints are shared outright, so $\tilde{\mathbf s}=\mathbf s^F$ almost surely. Every
$u\in A^F$ satisfies $h(\mathbf s^F,u)=x_r^F$ by definition, so $h(\tilde{\mathbf s},U_r)=x_r^F$
for every abducted $U_r$, and~\eqref{eq:general_cf} is the point mass at $x_r^F$. The argument
uses no property of $h$ or $P_U$ (which need only define $h$ almost everywhere, so that ties
in an $\argmax$ can be ignored). An empirical failure of consistency therefore diagnoses
the sampler, such as marginal rather than joint draws, not the fit or the coupling.
\end{proof}

\begin{proof}[Proof of Proposition~\ref{prop:unified}]
Equation~\eqref{eq:poset_scm} is an instance of~\eqref{eq:general_scm} with score
$\mathbf s=(\mathbf f,\bm\ell)$ (and the cutpoints) and noise $U_r=(\mathbf N_r,\mathbf G_r)$,
which is independent of the latent and of everything upstream. So it remains to
verify the compatibility identity and the exact sampler for $U_r$ restricted to
$A^F$; prediction and consistency then follow from Proposition~\ref{prop:general}
and Corollary~\ref{cor:consistency}.

\emph{Compatibility.} Given $\mathbf f(\mathbf x)$, $\mathbf Y_r\sim\mathcal N(\mathbf f(\mathbf x),I_k)$,
so the event $\kappa(\mathbf Y_r)=\mathbf g$ has probability
$\prod_jP\bigl(b_{j,g_j-1}\le f_j+N_{r,j}<b_{j,g_j}\bigr)=m_{\mathbf g}(\mathbf f(\mathbf x))$ by
independence of the coordinates of $\mathbf N_r$. The cells partition $\R^k$, and
$\Lambda_r=A$ iff $\kappa(\mathbf Y_r)\in\pi^{-1}(A)$, so $P(\Lambda_r=A)$ is the sum of
$m_{\mathbf g}$ over that preimage. $\mathbf G_r$ is independent of $\mathbf N_r$, so
conditionally on $\Lambda_r=A$ the race over $A$ is a Gumbel-max race on
$\bm\ell_A$, which selects $a$ with probability $\mathrm{softmax}(\bm\ell_A)_a$
(Lemma~\ref{lem:gumbelmax}). Multiplying gives~\eqref{eq:poset_identity}.

\emph{Abduction.} Given the latent, $\{X_r=a^\ast\}$ is the intersection of
$\{\mathbf N_r\in R\}$, with $R:=\bigcup_{\mathbf g\in\pi^{-1}(A^\ast)}(B_{\mathbf g}-\mathbf f^F)$ a
disjoint union of boxes, and $\{a^\ast=\argmax_{a\in A^\ast}(\ell^F_{A^\ast,a}+G_{r,a})\}$.
The first event involves only $\mathbf N_r$ and the second only
$(G_{r,a})_{a\in A^\ast}$. These are independent of each other and of the remaining
Gumbels, so the posterior factorises into (i)--(iii). For (i), $\mathbf N_r$ is a
priori standard normal, so its posterior is the standard normal restricted to
$R$: a mixture over the boxes with weights equal to their Gaussian masses
$m_{\mathbf g}(\mathbf f^F)$, each component a product of univariate truncated normals.
Hence~\eqref{eq:poset_abduct}. For (ii), Lemma~\ref{lem:truncgumbel} applied to the
race over $A^\ast$ gives~\eqref{eq:gumbel_abduct}. (iii) is immediate. Every step
is exact.

\emph{Prediction.} The intervention changes $\mathbf x_{\pa(r)}$ and leaves the latent,
the cutpoints, and $U_r$ untouched, so $\tilde{\mathbf Y}_r=\tilde{\mathbf f}+\mathbf N_r$
and~\eqref{eq:poset_cf} follows, with the joint posterior draws of
Lemma~\ref{lem:poset-laplace} and~\eqref{eq:joint_latent}.
\end{proof}

\begin{proof}[Proof of Corollary~\ref{cor:poset-monotone}]
(i) Each $\kappa_j$ is non-decreasing, so $\tilde{\mathbf f}\ge\mathbf f^F$ gives
$\kappa(\tilde{\mathbf f}+\mathbf N_r)\ge\kappa(\mathbf f^F+\mathbf N_r)$ coordinatewise. Since $\pi$ is
order-preserving, $\tilde\Lambda_r\succeq\pi(\kappa(\mathbf f^F+\mathbf N_r))=A^\ast$.
(ii) By~\eqref{eq:poset_identity} the likelihood of point $i$ is the product of a
threshold term, which depends only on $\mathbf f$ and the cutpoints, and a
within-block term, which depends only on $\bm\ell_{A(x_r^i)}$. These latents have
independent priors, and the within-block term of block $A$ involves only the
points observed in $A$, so the posterior is the product of the threshold fit
(Lemma~\ref{lem:poset-laplace}) and one classifier fit per block with
$\lvert A\rvert\ge2$.
\end{proof}
\end{revblock}

\subsection{Laplace machinery for the multiclass GP classifier (GPC)}
\label{sec:laplace-machinery}

\textbf{Joint latent posterior used in Corollary~\ref{cor:nominal}.} Under the
one-versus-rest fit, for each class $c$ the pair $(f^F_c,\tilde f_c)$ at the factual
and counterfactual parents is bivariate Gaussian with means
$\mathbf{k}_F^\top(\mathbf{e}_{r,c}-\hat{\bm\pi}_{r,c})$ and
$\tilde{\mathbf{k}}^\top(\mathbf{e}_{r,c}-\hat{\bm\pi}_{r,c})$ and covariance
\begin{equation}\label{eq:joint_latent}
    \mathrm{Cov}\bigl(f^F_c,\tilde f_c\bigr)
    = k_r\bigl(\mathbf{x}^F_{\pa(r)},\tilde{\mathbf{x}}_{\pa(r)}\bigr)
      - \mathbf{k}_F^\top\bigl(\mathbf{K}+\hat W_c^{-1}\bigr)^{-1}\tilde{\mathbf{k}},
    \qquad
    \mathbf{k}_F := \bigl(k_r(\mathbf{x}^F_{\pa(r)},\mathbf{x}^i_{\pa(r)})\bigr)_{i=1}^n,
\end{equation}
independently across classes.

\paragraph{Notation and ordering conventions.}
For $C>2$ the latent object is the $nC$-vector collecting $C$ latent values per training point,
and two orderings are in circulation. Write
\begin{equation}\label{eq:orderings}
    \mathbf{f}:=\bigl(\mathbf{f}_{r,1}^\top,\ldots,\mathbf{f}_{r,C}^\top\bigr)^\top\in\R^{nC}
    \quad\text{(class-ordered)},
    \qquad
    \mathbf{f}^\pi := P\mathbf{f}
    \quad\text{(point-ordered)},
\end{equation}
where $\mathbf{f}_{r,c}:=\bigl(f_{r,c}(\mathbf{x}^1_{\pa(r)}),\ldots,f_{r,c}(\mathbf{x}^n_{\pa(r)})\bigr)^\top$
and $P$ is the permutation sending index $(c-1)n+i$ to $(i-1)C+c$. Since $f_{r,1},\ldots,f_{r,C}$
are independent GPs sharing the kernel $k_r$, the prior covariance of $\mathbf{f}$ is
\begin{equation}\label{eq:prior_cov}
    \bm{\mathcal{K}}_r := \mathbf{I}_C\otimes\Kb,
    \qquad \Kb:=\bigl(k_r(\mathbf{x}^i_{\pa(r)},\mathbf{x}^j_{\pa(r)})\bigr)_{ij},
\end{equation}
Let
$\mathbf{e}_{r,c}\in\{0,1\}^n$ have $i$-th entry $\mathbb{I}\{x^i_r=c\}$, let
$\mathbf{e}\in\R^{nC}$ be the class-ordered stack of the $\mathbf{e}_{r,c}$, and let
$\bm{\pi}\in\R^{nC}$ be the corresponding stack of $\bm{\pi}_{r,c}$, where
$\bm{\pi}_i(\mathbf{f})=\mathrm{softmax}(\mathbf{f}_r(\mathbf{x}^i_{\pa(r)}))\in\R^C$, the softmax
reduction.
 
\noindent The following lemma derives the multiclass Laplace posterior whose two-point
covariance is~\eqref{eq:joint_latent}.
 
\begin{lemma}[Strict concavity, the MAP equation, and the Laplace predictive]
\label{lem:map}
Let $\Psi(\mathbf{f}):=\log p(\mathbf{x}_r\mid\mathbf{f})+\log p(\mathbf{f}\mid\mathbf{X}_{\pa(r)})$
be the unnormalised log posterior of the multiclass GPC model, i.e.\ of the bare softmax
likelihood $P(X_r=c\mid\mathbf{f})=\mathrm{softmax}(\mathbf{f}_r(\mathbf{x}))_c$ \emph{without} the
Gumbel-max coupling of Corollary~\ref{cor:nominal}'s \eqref{eq:gumbel_scm}, with $\Kb\succ0$. Then:
\begin{enumerate}
\item[(i)] $\nabla\Psi(\mathbf{f})=(\mathbf{e}-\bm{\pi})-\bm{\mathcal{K}}_r^{-1}\mathbf{f}$ and
$\nabla^2\Psi(\mathbf{f})=-W(\mathbf{f})-\bm{\mathcal{K}}_r^{-1}$, where $W(\mathbf{f})$ is the
$nC\times nC$ matrix with $P\,W(\mathbf{f})\,P^\top=\blockdiag\bigl(\Lambda_1(\mathbf{f}),\ldots,\Lambda_n(\mathbf{f})\bigr)$
and $\Lambda_i(\mathbf{f}):=\diag(\bm\pi_i)-\bm\pi_i\bm\pi_i^\top$ (derived below);
\item[(ii)] $W(\mathbf{f})\succeq0$, with null space of dimension $n$;
\item[(iii)] $\Psi$ is strictly concave, hence has a unique maximiser $\hat{\mathbf{f}}_r$,
characterised for each $c\in\{1,\ldots,C\}$ by
\begin{equation}\label{eq:map_eq}
    \hat{\mathbf{f}}_{r,c} = \Kb\bigl(\mathbf{e}_{r,c}-\hat{\bm{\pi}}_{r,c}\bigr),
    \qquad \hat{\bm{\pi}}_{r,c}:=\bm{\pi}_{r,c}(\hat{\mathbf{f}}_r);
\end{equation}
\item[(iv)] the Laplace approximation to the latent posterior is
$q(\mathbf{f}_r\mid\cdot)=\mathcal{N}\bigl(\hat{\mathbf{f}}_r,(\bm{\mathcal{K}}_r^{-1}+\hat W)^{-1}\bigr)$,
and the induced predictive law at any input
$\mathbf{x}^\star$ has, per class,
\begin{equation}\label{eq:pred_moments}
    \E_q\bigl[f_{r,c}(\mathbf{x}^\star)\bigr]
    = \mathbf{k}_\star^\top\bigl(\mathbf{e}_{r,c}-\hat{\bm\pi}_{r,c}\bigr),
    \qquad
    \Var_q\bigl[f_{r,c}(\mathbf{x}^\star)\bigr]
    = k_\star - \mathbf{k}_\star^\top\bigl(\Kb+\hat W_c^{-1}\bigr)^{-1}\mathbf{k}_\star,
\end{equation}
with $\mathbf{k}_\star:=(k_r(\mathbf{x}^\star,\mathbf{x}^i_{\pa(r)}))_{i=1}^n$,
$k_\star:=k_r(\mathbf{x}^\star,\mathbf{x}^\star)$ and
$\hat W_c:=\diag(\hat\pi_{i,c}(1-\hat\pi_{i,c}))_{i=1}^n$. More generally, for two inputs
$\mathbf{x}^\star,\mathbf{x}^\dagger$,
\begin{equation}\label{eq:pred_cov}
    \Cov_q\bigl[f_{r,c}(\mathbf{x}^\star),f_{r,c}(\mathbf{x}^\dagger)\bigr]
    = k_r(\mathbf{x}^\star,\mathbf{x}^\dagger)
      - \mathbf{k}_\star^\top\bigl(\Kb+\hat W_c^{-1}\bigr)^{-1}\mathbf{k}_\dagger .
\end{equation}
\end{enumerate}
\end{lemma}
 
\begin{proof}
(i) The bare softmax likelihood factorises over training points, so
\begin{equation}\label{eq:logjoint}
    \Psi(\mathbf{f})
    = \sum_{i=1}^{n}\Bigl[f_{r,x_r^i}(\mathbf{x}^i_{\pa(r)})
        - \log\textstyle\sum_{c=1}^{C}\exp f_{r,c}(\mathbf{x}^i_{\pa(r)})\Bigr]
      - \tfrac12\mathbf{f}^\top\bm{\mathcal{K}}_r^{-1}\mathbf{f} + \text{const},
\end{equation}
the constant absorbing $-\tfrac12\log\lvert\bm{\mathcal{K}}_r\rvert-\tfrac{nC}{2}\log2\pi$.
Differentiating the $i$-th bracket in $f_{r,c}(\mathbf{x}^i_{\pa(r)})$ gives
$\mathbb{I}\{x_r^i=c\}-\pi_{i,c}$, which stacks to $\mathbf{e}-\bm{\pi}$; differentiating
$-\tfrac12\mathbf{f}^\top\bm{\mathcal{K}}_r^{-1}\mathbf{f}$ gives
$-\bm{\mathcal{K}}_r^{-1}\mathbf{f}$. Differentiating once more, the second derivative of the $i$-th
bracket with respect to $\bigl(f_{r,c}(\mathbf{x}^i_{\pa(r)})\bigr)_{c=1}^{C}$ is
$-\bigl(\diag(\bm{\pi}_i)-\bm{\pi}_i\bm{\pi}_i^\top\bigr)=-\Lambda_i(\mathbf{f})$, and mixed
derivatives across distinct training points vanish because the likelihood factorises.
 
(ii) For $\mathbf{v}\in\R^{nC}$ with point-ordered blocks $\mathbf{v}_i\in\R^C$,
\begin{equation}\label{eq:psd}
    \mathbf{v}^\top W\mathbf{v}
    = \sum_{i=1}^{n}\mathbf{v}_i^\top\Lambda_i\mathbf{v}_i
    = \sum_{i=1}^{n}\Bigl[\sum_{c}\pi_{i,c}v_{i,c}^2-\Bigl(\sum_c\pi_{i,c}v_{i,c}\Bigr)^2\Bigr]
    = \sum_{i=1}^{n}\Var_{c\sim\bm{\pi}_i}\bigl[v_{i,c}\bigr]\;\ge\;0 .
\end{equation}
The variance in \eqref{eq:psd} vanishes iff $v_{i,c}$ is constant in $c$ for each $i$ (as
$\pi_{i,c}>0$ for all $c$), i.e.\ iff $\mathbf{v}_i\in\mathrm{span}\{\mathbf{1}_C\}$; the null space
of $W$ therefore has dimension exactly $n$. This degeneracy is the shift-invariance of the softmax,
$\mathrm{softmax}(\mathbf{f}_i+a\mathbf{1}_C)=\mathrm{softmax}(\mathbf{f}_i)$.
 
(iii) By (i) and (ii), $\nabla^2\Psi=-W-\bm{\mathcal{K}}_r^{-1}\prec0$ everywhere, since
$W\succeq0$ and $\bm{\mathcal{K}}_r^{-1}=\mathbf{I}_C\otimes\Kb^{-1}\succ0$. Hence $\Psi$ is
strictly concave and, being coercive, attains a unique global maximum $\hat{\mathbf{f}}_r$. Note
that it is the GP prior, not the likelihood, that removes the softmax degeneracy and makes the MAP
well defined. Setting $\nabla\Psi=0$ gives
$\bm{\mathcal{K}}_r^{-1}\hat{\mathbf{f}}_r=\mathbf{e}-\hat{\bm{\pi}}$, i.e.\
$\hat{\mathbf{f}}_r=\bm{\mathcal{K}}_r(\mathbf{e}-\hat{\bm{\pi}})$. Since
$\bm{\mathcal{K}}_r=\mathbf{I}_C\otimes\Kb$ is block diagonal, this
decouples across $c$ into \eqref{eq:map_eq}.
 
(iv) Since $\nabla\Psi(\hat{\mathbf{f}}_r)=0$, the second-order Taylor expansion of $\Psi$ at its
mode is
$\Psi(\mathbf{f})=\Psi(\hat{\mathbf{f}}_r)-\tfrac12(\mathbf{f}-\hat{\mathbf{f}}_r)^\top
(\bm{\mathcal{K}}_r^{-1}+\hat W)(\mathbf{f}-\hat{\mathbf{f}}_r)+O(\|\mathbf{f}-\hat{\mathbf{f}}_r\|^3)$;
exponentiating and normalising the quadratic part gives the stated Gaussian
\citep{williams1998bayesian,Rasmussen_2006}, whose covariance is well defined because
$\bm{\mathcal{K}}_r^{-1}+\hat W\succ0$ by (ii)--(iii) notwithstanding the singularity of $\hat W$.
For the predictive moments, $f_{r,c}(\mathbf{x}^\star)\mid\mathbf{f}_{r,c}\sim
\mathcal{N}(\mathbf{k}_\star^\top\Kb^{-1}\mathbf{f}_{r,c},\,k_\star-\mathbf{k}_\star^\top\Kb^{-1}\mathbf{k}_\star)$
by the GP prior; marginalising over $q$ with the laws of total expectation and total variance and
using \eqref{eq:map_eq} gives the mean in \eqref{eq:pred_moments}, while the variance follows on
approximating the $c$-th diagonal block of $(\bm{\mathcal{K}}_r^{-1}+\hat W)^{-1}$ by
$(\Kb^{-1}+\hat W_c)^{-1}$ and applying the Woodbury identity
$\Kb^{-1}-\Kb^{-1}(\Kb^{-1}+\hat W_c)^{-1}\Kb^{-1}=(\Kb+\hat W_c^{-1})^{-1}$, legitimate because
$0<\hat\pi_{i,c}<1$ makes $\hat W_c$ invertible. Equation~\eqref{eq:pred_cov} is the same
computation with $\mathbf{k}_\star$ and $\mathbf{k}_\dagger$ in the two slots.
\end{proof}

\begin{remark}[One-versus-rest predictive moments]
\label{rem:ovr}
Lemma~\ref{lem:map} concerns the joint multiclass Laplace approximation with a softmax
likelihood. The implementation evaluated in \S\ref{sec:scm_suite} instead fits $C$
independent binary Laplace classifiers and renormalises, and
\eqref{eq:pred_moments}--\eqref{eq:pred_cov} do not apply to it verbatim: there each class
has its own bivariate Gaussian posterior over $(f^F_c,\tilde f_c)$, obtained from that
class's own binary Laplace fit, and the classes are independent. As established in the proof
of Corollary~\ref{cor:nominal}, this changes \emph{how} $\mathbf{p}$ is estimated but not
the counterfactual construction, because the coupling is stated on $\log\mathbf{p}$. The
compatibility identity, the abduction of Lemma~\ref{lem:truncgumbel} and
Corollary~\ref{cor:consistency} all go through unchanged; only the accuracy of $\mathbf{p}$
is affected --- which is exactly what $\mathrm{TV}_{\mathrm{int}}$ in
Table~\ref{tab:scm_suite_results} measures. At $C=2$ the two approximations coincide and
Corollary~\ref{cor:binary} applies as stated.
\end{remark}

\subsection{Proof of Corollary~\ref{cor:nominal} and Corollary~\ref{cor:binary}}
\label{sec:proof-prop-gumbel}

\begin{lemma}[Gumbel-max representation {\normalfont\citep{gumbel1954,yellott1977}}]
\label{lem:gumbelmax}
Let $G_1,\ldots,G_C\stackrel{iid}{\sim}\mathrm{Gumbel}(0,1)$, and let
$\bm{\ell}=(\ell_1,\ldots,\ell_C)\in\R^C$ be any score vector.
Then
\begin{equation}\label{eq:gm_identity}
    P\Bigl(\argmax_{c}\{\ell_c+G_c\}=c^\ast\Bigr)=\mathrm{softmax}(\bm{\ell})_{c^\ast},
    \quad
    \max_c\{\ell_c+G_c\}\sim\mathrm{Gumbel}\Bigl(\log\textstyle\sum_c e^{\ell_c},\,1\Bigr),
\end{equation}
and the $\argmax$ and the $\max$ are independent.
\end{lemma}
 
\begin{proof}
$\ell_c+G_c\sim\mathrm{Gumbel}(\ell_c,1)$ with CDF $\exp(-e^{-(t-\ell_c)})$. By independence,
$P(\max_c\{\ell_c+G_c\}\le t)=\prod_c\exp(-e^{-(t-\ell_c)})=\exp\bigl(-e^{-t}\sum_ce^{\ell_c}\bigr)
=\exp\bigl(-e^{-(t-\log\sum_ce^{\ell_c})}\bigr)$, which is the second claim. For the first, condition
on $G_{c^\ast}$ and integrate:
\begin{equation}\label{eq:gm_integral}
    P\bigl(\argmax=c^\ast\bigr)
    = \int_\R e^{-(g-\ell_{c^\ast})}e^{-e^{-(g-\ell_{c^\ast})}}
      \prod_{c\ne c^\ast}\exp\bigl(-e^{-(g-\ell_c)}\bigr)\,dg .
\end{equation}
Substituting $z:=e^{-g}$, $dz=-e^{-g}dg$, the integrand becomes
$e^{\ell_{c^\ast}}\exp\bigl(-z\sum_ce^{\ell_c}\bigr)dz$ on $z\in(0,\infty)$, whence
\eqref{eq:gm_integral} equals $e^{\ell_{c^\ast}}/\sum_ce^{\ell_c}$. Independence of the $\argmax$ and the
$\max$ is the standard consequence of the same computation, the conditional law of the maximum
given $\argmax=c^\ast$ being free of $c^\ast$.
\end{proof}
 
\begin{lemma}[Posterior of the Gumbel vector given the realised class
  {\normalfont\citep{maddison2014astar}}]
\label{lem:truncgumbel}
In the setting of Lemma~\ref{lem:gumbelmax}, condition on $\argmax_c\{\ell_c+G_c\}=c^\ast$. Let
$M:=\max_c\{\ell_c+G_c\}$. Then $M\sim\mathrm{Gumbel}(\log\sum_ce^{\ell_c},1)$, and given $M$ the
coordinates are conditionally independent with
\begin{equation}\label{eq:truncg}
    G_{c^\ast}=M-\ell_{c^\ast},
    \qquad
    G_c \mid M \;\sim\; \mathrm{Gumbel}(0,1)\ \text{conditioned on}\ G_c\le M-\ell_c,
    \quad c\ne c^\ast .
\end{equation}
Consequently \eqref{eq:gumbel_abduct} samples the posterior exactly.
\end{lemma}
 
\begin{proof}
The event $\{\argmax=c^\ast, M\in dm\}$ is $\{\ell_{c^\ast}+G_{c^\ast}\in dm\}\cap
\bigcap_{c\ne c^\ast}\{\ell_c+G_c\le m\}$. Given $M=m$, the coordinate $c^\ast$ is determined,
$G_{c^\ast}=m-\ell_{c^\ast}$, and the remaining coordinates are independent (by independence of the
$G_c$) each conditioned on the event $\{\ell_c+G_c\le m\}$, i.e.\ $\{G_c\le m-\ell_c\}$. That
$M\sim\mathrm{Gumbel}(\log\sum_ce^{\ell_c},1)$ unconditionally, and independently of the $\argmax$, is
Lemma~\ref{lem:gumbelmax}; hence its conditional law given $\argmax=c^\ast$ is the same. For the
sampler, if $G\sim\mathrm{Gumbel}(0,1)$ conditioned on $G\le b$ then its CDF is
$\exp(-e^{-g})/\exp(-e^{-b})$ for $g\le b$; setting this equal to $V\sim\mathrm{Uniform}(0,1)$ and
solving gives $e^{-g}=e^{-b}-\log V$, i.e.\ $g=-\log\bigl(e^{-b}-\log V\bigr)$, which with
$b=m-\ell_c$ is the expression in~\eqref{eq:gumbel_abduct}. Note $e^{-b}-\log V>0$ almost surely, so the logarithm is
well defined.
\end{proof}
 
\begin{proof}[Proof of Corollary~\ref{cor:nominal}]
\emph{Compatibility.} Apply Lemma~\ref{lem:gumbelmax} with $\ell_c=\log p_c(\mathbf{x})$. For any
probability vector $\mathbf{p}$,
\begin{equation}
    \mathrm{softmax}(\log\mathbf{p})_c = \frac{p_c}{\sum_{c'}p_{c'}} = p_c,
\end{equation}
so \eqref{eq:gm_identity} gives~\eqref{eq:gumbel_identity} directly, for arbitrary $\mathbf{p}$, and
Proposition~\ref{prop:general} gives the interventional law.
 
The coupling must sit on $\log\mathbf{p}$, not on $\mathbf{f}_r$. Under the joint Laplace fit,
$\mathrm{softmax}(\mathbf{f}_r)_c=p_c$, so
\begin{equation}
    \log p_c = f_{r,c} - \log\textstyle\sum_{c'}e^{f_{r,c'}} = f_{r,c} - Z(\mathbf{x})
\end{equation}
with $Z(\mathbf{x}):=\log\sum_{c'}e^{f_{r,c'}}$ a shift invisible to $\argmax$, so coupling on $\mathbf{f}_r$ or on $\log\mathbf{p}$ agree. Under a
one-versus-rest fit, $p_c\propto\sigma(f_{r,c})$, so $\mathrm{softmax}(\mathbf{f}_r)\ne\mathbf{p}$ in
general and coupling on $\mathbf{f}_r$ would violate~\eqref{eq:gumbel_identity}; only $\log\mathbf{p}$ works for both.
 
\emph{Abduction.} By Proposition~\ref{prop:general} we need $\mathbf G_r$ restricted to the
event that $c^\ast=x_r^F$ wins the race. Given $\bm{\ell}^F$, Lemma~\ref{lem:truncgumbel} with
$\ell_c=\ell_c^F=\log p_c(\mathbf{x}^F_{\pa(r)})$ samples it exactly, which
is~\eqref{eq:gumbel_abduct}. Since $\bm\ell^F=\log\mathbf p$ is normalised,
$\sum_ce^{\ell_c^F}=1$ and $M\sim\mathrm{Gumbel}(0,1)$ here; \eqref{eq:gumbel_abduct} retains the general location
because it also holds for unnormalised scores.
 
\emph{Prediction.} Substituting $\tilde{\bm\ell}$ into~\eqref{eq:gumbel_scm}
gives~\eqref{eq:gumbel_cf_given_f}. The joint moments of $(\mathbf{f}^F,\tilde{\mathbf{f}})$ are given by
Lemma~\ref{lem:map}(iv) at $\mathbf{x}^\star=\mathbf{x}^F_{\pa(r)}$,
$\mathbf{x}^\dagger=\tilde{\mathbf{x}}_{\pa(r)}$, with covariance~\eqref{eq:joint_latent}.
\end{proof}

\begin{corollary}[Binary case: closed form]
\label{cor:binary}
Let $C=2$ and put $p(\mathbf{x}) := p_1(\mathbf{x}) = \sigma(f_r(\mathbf{x}))$
with $\sigma(t)=(1+e^{-t})^{-1}$.  Then \eqref{eq:gumbel_scm} is equivalent to
\begin{equation}\label{eq:binary_scm}
    X_r := \mathbb{I}\bigl[\,p(\mathbf{X}_{\pa(r)}) > U_r\,\bigr],
    \qquad U_r\sim\mathrm{Uniform}(0,1),
\end{equation}
the exogenous noise being $U_r=\sigma(G_{r,2}-G_{r,1})$.
Abduction is exact:
writing $p^F := p(\mathbf{x}^F_{\pa(r)})$,
\begin{equation}\label{eq:binary_abduct}
    U_r \mid \mathbf{x}^F, f_r \;\sim\;
    \begin{cases}
        \mathrm{Uniform}(0,\,p^F), & x_r^F = 1,\\[2pt]
        \mathrm{Uniform}(p^F,\,1), & x_r^F = 0,
    \end{cases}
\end{equation}
and with $p^{\CF} := p(\tilde{\mathbf{x}}_{\pa(r)})$,
\begin{equation}\label{eq:binary_cf_given_f}
    P\bigl(X_r^{\CF}=1 \mid \mathbf{x}^F, f_r\bigr)
    =
    \begin{cases}
        \min(p^{\CF},\,p^F)\,/\,p^F, & x_r^F = 1,\\[2pt]
        \max(0,\;p^{\CF}-p^F)\,/\,(1-p^F), & x_r^F = 0 .
    \end{cases}
\end{equation}
Marginalising over the joint Laplace posterior of
$(f^F,\tilde f):=\bigl(f_r(\mathbf{x}^F_{\pa(r)}),f_r(\tilde{\mathbf{x}}_{\pa(r)})\bigr)$
gives
\begin{equation}\label{eq:binary_cf}
    P\bigl(X_r^{\CF}=1\mid\mathbf{x}^F\bigr)
    = \mathbb{E}_{(f^F,\tilde f)}\Bigl[\,\text{\eqref{eq:binary_cf_given_f}}
      \text{ evaluated at } p^F=\sigma(f^F),\; p^{\CF}=\sigma(\tilde f)\Bigr].
\end{equation}
\end{corollary}

\begin{proof}[Proof of Corollary~\ref{cor:binary} (binary closed form)]
\emph{Reduction.} At $C=2$, shift-invariance of the $\argmax$ under
$\mathbf{f}\mapsto\mathbf{f}+a\mathbf{1}_2$ lets us write the event $\{X_r=1\}$ as
$\{f_{r,1}+G_{r,1}>f_{r,2}+G_{r,2}\}=\{G_{r,2}-G_{r,1}<f_r\}$ with $f_r:=f_{r,1}-f_{r,2}$. The
difference of two independent standard Gumbels is standard logistic: for $L:=G_{r,2}-G_{r,1}$,
\begin{equation}\label{eq:logistic_diff}
    P(L\le t)=\int_\R \exp\bigl(-e^{-(g+t)}\bigr)\,e^{-g}\exp\bigl(-e^{-g}\bigr)\,dg
    =\frac{1}{1+e^{-t}}=\sigma(t),
\end{equation}
by the substitution $z=e^{-g}$, which turns the integral into
$\int_0^\infty e^{-z(1+e^{-t})}dz=(1+e^{-t})^{-1}$. Setting $U_r:=\sigma(L)$, which is
$\mathrm{Uniform}(0,1)$ because $\sigma$ is the CDF of $L$, and using strict monotonicity of
$\sigma$,
\begin{equation}\label{eq:binary_equiv}
    \{X_r=1\}=\{L<f_r\}=\{\sigma(L)<\sigma(f_r)\}=\{U_r<p\},\qquad p:=\sigma(f_r),
\end{equation}
which is \eqref{eq:binary_scm}.
 
\emph{Abduction.} By \eqref{eq:binary_equiv} the observation $x_r^F=1$ is equivalent to
$U_r<p^F$, and $x_r^F=0$ to $U_r\ge p^F$. Since $U_r$ is uniform and independent of $f_r$, its
posterior is the uniform law restricted to the corresponding interval, which is \eqref{eq:binary_abduct}.
Note that this abduction is \emph{exact}: no approximation intervenes between the observation and
the posterior of the exogenous variable.
 
\emph{Prediction.} Conditional on $f_r$, and hence on $p^F$ and $p^\CF$, the counterfactual event
is $\{U_r<p^\CF\}$ with $U_r$ drawn from~\eqref{eq:binary_abduct}. If $x_r^F=1$ then $U_r\sim\mathrm{Uniform}(0,p^F)$ and
\begin{equation}\label{eq:binary_case1}
    P\bigl(U_r<p^\CF\bigr)=\frac{\lvert(0,p^\CF)\cap(0,p^F)\rvert}{p^F}
    =\frac{\min(p^\CF,p^F)}{p^F},
\end{equation}
while if $x_r^F=0$ then $U_r\sim\mathrm{Uniform}(p^F,1)$ and
\begin{equation}\label{eq:binary_case0}
    P\bigl(U_r<p^\CF\bigr)=\frac{\lvert(p^F,1)\cap(0,p^\CF)\rvert}{1-p^F}
    =\frac{\max(0,\,p^\CF-p^F)}{1-p^F},
\end{equation}
which together are \eqref{eq:binary_cf_given_f}. Marginalising over the joint Laplace posterior of
$(f^F,\tilde f)$, whose moments are \eqref{eq:pred_moments}--\eqref{eq:pred_cov} at $C=2$, gives
\eqref{eq:binary_cf}.
 
\emph{Consistency and monotonicity.} Setting $\tilde{\mathbf{x}}_{\pa(r)}=\mathbf{x}^F_{\pa(r)}$
gives $p^\CF=p^F$ almost surely, so \eqref{eq:binary_case1} equals $1$ and
\eqref{eq:binary_case0} equals $0$, the direct consequence of Corollary~\ref{cor:consistency}.  Moreover both expressions are non-decreasing in $p^\CF$, so the counterfactual probability is monotone in the intervention whenever $p(\cdot)$ is, and here holds exactly and pathwise.
\end{proof}
 
\begin{remark}[Common random numbers in the Monte Carlo estimator]
\label{rem:crn}
When \eqref{eq:general_cf} or \eqref{eq:binary_cf} is evaluated over a grid of intervention values
$\tilde{\mathbf{x}}_{\pa(r)}$ in order to trace a counterfactual response curve, the draws of
$(\mathbf{f}^F,\tilde{\mathbf{f}})$ should share a common random number stream across grid points.
The per-point marginals are unaffected --- the draws are merely correlated across the grid --- but
the resulting curve is smooth rather than jittery, which preserves the exact monotonicity
established in the proof of Corollary~\ref{cor:binary} at practical sample sizes. Independent draws per grid point
leave a Monte Carlo residual that can locally violate monotonicity by $O(M^{-1/2})$.
\end{remark}

\begin{remark}[No probit closed form]
\label{rem:probit_use}
The probit approximation \citep{mackay1992evidence,bishop2006pattern}
\begin{equation*}
  \E[\sigma(f)]\approx\Phi\bigl(\lambda m/\sqrt{1+\lambda^2s^2}\bigr),\qquad \lambda=\sqrt{\pi/8},
\end{equation*}
gives a closed form for the
\emph{interventional} class probability but not for~\eqref{eq:binary_cf}.
Equations~\eqref{eq:binary_case1}--\eqref{eq:binary_case0} are non-linear in
$(p^F,p^\CF)$, so substituting the two approximated marginals computes
$g\bigl(\E[p^F],\E[p^\CF]\bigr)$ instead of $\E\bigl[g(p^F,p^\CF)\bigr]$. This is
biased by Jensen's inequality and breaks the exact consistency of
Corollary~\ref{cor:consistency}, since it no longer forces $p^\CF=p^F$ pathwise. We
therefore keep the Monte Carlo form. Where gradients are needed, the
reparameterised draws of $(f^F,\tilde f)$ are differentiable in
$\tilde{\mathbf{x}}_{\pa(r)}$, so no probit step is required.
\end{remark}

\subsection{Proof of Corollary~\ref{cor:ordinal} and Corollary~\ref{cor:ordinal-fixed}}
\label{app:ordinal-proofs}
\textbf{Cutpoint parameterisation.} Corollary~\ref{cor:ordinal} writes
$\theta=(\theta_1,\ldots,\theta_{C-1})\in\R^{C-1}$ for the unconstrained parameterisation
\begin{equation}\label{eq:theta_reparam}
    b_1(\theta):=\theta_1,\qquad b_c(\theta):=b_{c-1}(\theta)+e^{\theta_c}\quad(c=2,\ldots,C-1),
\end{equation}
which maps $\R^{C-1}$ onto the strictly increasing cutpoint tuples, with an
independent Gaussian prior $\theta\sim p(\theta):=\prod_c\mathcal N(\theta_c;0,\tau_c^2)$,
a priori independent of $f_r$ and of $N_r$.

Throughout, $\Phi,\varphi$ denote the standard normal CDF and density, and for a fixed
$\theta$ we write $b_0=-\infty<b_1<\cdots<b_{C-1}<b_C=\infty$ for the cutpoints
$\mathbf b(\theta)$ of~\eqref{eq:theta_reparam}.

\begin{lemma}[Laplace approximation for the ordinal-probit likelihood
  {\normalfont\citep{chu2005gaussian}}]
\label{lem:ordinal-map}
Fix $\theta$. Let $\pi(f;c):=\Phi(b_c-f)-\Phi(b_{c-1}-f)$ be the likelihood of class $c$
given latent value $f$, and $\Psi(\mathbf{f};\theta):=\sum_{i=1}^n\log\pi(f_i;x_r^i)
-\tfrac12\mathbf{f}^\top\Kb^{-1}\mathbf{f}+\mathrm{const}$ the log posterior
of~\eqref{eq:ordinal_scm}. Then $\nabla_{\mathbf f}\Psi=\bm{\ell}'(\mathbf{f};\theta)-\Kb^{-1}\mathbf{f}$
and $\nabla_{\mathbf f}^2\Psi=-W(\mathbf{f};\theta)-\Kb^{-1}$ with $W$ \emph{diagonal},
$W_{ii}=-\ell''_i(f_i)$, where
\begin{equation}
    \ell_i'(f_i) = \frac{\varphi(z_0^i)-\varphi(z_1^i)}{\pi(f_i;x_r^i)},
    \qquad
    \ell_i''(f_i) = \frac{z_0^i\varphi(z_0^i)-z_1^i\varphi(z_1^i)}{\pi(f_i;x_r^i)}
                    - \bigl(\ell_i'(f_i)\bigr)^2,
\end{equation}
$z_1^i:=b_{x_r^i}-f_i$, $z_0^i:=b_{x_r^i-1}-f_i$ (with $\varphi(\pm\infty):=0$ and
$z\varphi(z)\to0$ as $z\to\pm\infty$). Where $W(\hat{\mathbf{f}};\theta)\succeq0$ at the mode
(verified numerically for every fit we report), the mode satisfies
$\hat{\mathbf{f}}(\theta)=\Kb\bm{\ell}'(\hat{\mathbf{f}}(\theta);\theta)$, the Laplace posterior is
$\mathcal{N}\bigl(\hat{\mathbf{f}}(\theta),(\Kb^{-1}+\hat W)^{-1}\bigr)$, and at any two inputs
$\mathbf{x}^\star,\mathbf{x}^\dagger$ the latent values are jointly Gaussian with
\begin{equation}
    \E_q[f(\mathbf{x}^\star)\mid\theta] = \mathbf{k}_\star^\top\bm{\ell}'(\hat{\mathbf{f}}(\theta);\theta),
    \quad
    \Cov_q\bigl[f(\mathbf{x}^\star),f(\mathbf{x}^\dagger)\mid\theta\bigr]
    = k(\mathbf{x}^\star,\mathbf{x}^\dagger)
      - \mathbf{k}_\star^\top(\Kb+\hat W^{-1})^{-1}\mathbf{k}_\dagger .
\end{equation}
This is the single-latent analogue of Lemma~\ref{lem:map}(iv), with $\mathbf{e}-\hat{\bm\pi}$
and $\hat W_c$ replaced by $\bm{\ell}'$ and $\hat W$.
\end{lemma}

We take the concavity and convergence of the Newton iteration from
\citet{chu2005gaussian}; what follows is the counterfactual construction on top of
the fitted latent.

\begin{lemma}[Laplace approximation for the cutpoint posterior]
\label{lem:theta-laplace}
Let $\mathrm{LML}(\theta):=\Psi(\hat{\mathbf f}(\theta);\theta)-\tfrac12\log\det\bigl(I+\hat W^{1/2}\Kb\hat W^{1/2}\bigr)$
be the Laplace log marginal likelihood of Lemma~\ref{lem:ordinal-map} at $\theta$, and let
$L(\theta):=\log p(\theta)+\mathrm{LML}(\theta)$, $\hat\theta:=\argmax_\theta L(\theta)$. Since
$\partial\Psi/\partial\mathbf f=0$ at $\hat{\mathbf f}(\theta)$, the envelope theorem gives the
gradient of $\Psi(\hat{\mathbf f}(\theta);\theta)$ as its partial derivative in $\theta$ with
$\mathbf f=\hat{\mathbf f}(\theta)$ held fixed. The remaining terms depend on $\theta$ through
$\hat{\mathbf f}$ as well, via
$$\frac{d\hat{\mathbf f}}{d\theta_c} = \bigl(\Kb^{-1}+\hat W\bigr)^{-1}\Kb\,
  \frac{\partial\bm\ell'}{\partial\theta_c}\Big|_{\hat{\mathbf f}(\theta)},$$
obtained by differentiating the mode condition implicitly \citep[\S5.5.1]{Rasmussen_2006}.
The Laplace approximation to the cutpoint posterior is
\begin{equation}\label{eq:theta_laplace}
    \theta\mid\mathcal D \ \approx\ \mathcal N\bigl(\hat\theta,\Sigma_\theta\bigr),\qquad
    \Sigma_\theta = -\bigl[\nabla^2_\theta L(\theta)\bigr]^{-1}_{\theta=\hat\theta}.
\end{equation}
In our experiments $\theta\in\R^{C-1}$ with $C\le3$, so $\Sigma_\theta$ is computed by central
finite differences of $L$ rather than by coding $d\hat{\mathbf f}/d\theta$.
\end{lemma}

\begin{proof}[Proof of Corollary~\ref{cor:ordinal}]
Fix a realisation of $\theta$ for the first three steps, and write $b_c:=b_c(\theta)$.

\emph{Compatibility.} $N_r\sim\mathcal{N}(0,1)$ is independent of $(f_r,\theta)$, so
$Y_r\mid f_r(\mathbf{x}),\theta\sim\mathcal{N}(f_r(\mathbf{x}),1)$, and $X_r=c$ iff
$Y_r\in[b_{c-1},b_c)$. Hence
$P(X_r=c\mid\mathbf{X}_{\pa(r)}=\mathbf{x},f_r,\theta)=\Phi(b_c-f_r(\mathbf{x}))-\Phi(b_{c-1}-f_r(\mathbf{x}))$,
which is~\eqref{eq:ordinal_identity} and the likelihood of Lemma~\ref{lem:ordinal-map}.

\emph{Abduction.} Let $c^\ast:=x_r^F$ and $f_r^F:=f_r(\mathbf{x}^F_{\pa(r)})$. Since
$Y_r^F=f_r^F+N_r$, the event $X_r^F=c^\ast$ is $\{N_r\in(\alpha,\beta)\}$ with
$\alpha:=b_{c^\ast-1}-f_r^F$ and $\beta:=b_{c^\ast}-f_r^F$. $N_r$ is a priori standard normal and
independent of $(f_r,\theta)$, so by Bayes' rule its posterior is the standard normal truncated to
$(\alpha,\beta)$, which is~\eqref{eq:ordinal_abduct}. As in Corollary~\ref{cor:binary}, this step is
exact given $(f_r^F,\theta)$; marginally it is a mixture of truncated normals.

\emph{Action and prediction.} The intervention leaves $f_r$, $\theta$ and $N_r$ unchanged, so
$\tilde Y_r=\tilde f_r+N_r$ with $\tilde f_r:=f_r(\tilde{\mathbf{x}}_{\pa(r)})$. With
$\gamma_k:=b_k-\tilde f_r$, $\{X_r^{\CF}=k\}=\{N_r\in(\gamma_{k-1},\gamma_k)\}$, so
\begin{equation}
  \begin{split}
    P\bigl(X_r^{\CF}=k\mid\mathbf{x}^F,f_r^F,\tilde f_r,\theta\bigr)
    &= P\bigl(N_r\in(\gamma_{k-1},\gamma_k)\mid N_r\in(\alpha,\beta)\bigr)\\
    &= \frac{\bigl|(\gamma_{k-1},\gamma_k)\cap(\alpha,\beta)\bigr|_{\mathcal{N}(0,1)}}
           {\bigl|(\alpha,\beta)\bigr|_{\mathcal{N}(0,1)}},
  \end{split}
\end{equation}
and writing the intersection as $(\max(\alpha,\gamma_{k-1}),\min(\beta,\gamma_k))$, clipped to
probability zero when empty, gives~\eqref{eq:ordinal_cf}. This is the overlap computation
of~\eqref{eq:binary_case1}--\eqref{eq:binary_case0} for two general intervals under a standard
normal. Only the posteriors of $f_r$ given $\theta$ (Lemma~\ref{lem:ordinal-map}) and of $\theta$
(Lemma~\ref{lem:theta-laplace}) are known, so~\eqref{eq:ordinal_cf} is averaged over a joint draw of
$(f_r^F,\tilde f_r)$ at each draw of $\theta$, giving~\eqref{eq:ordinal_cf_marginal}. Convergence
and consistency follow from Proposition~\ref{prop:general} and Corollary~\ref{cor:consistency}.

\emph{Corollary~\ref{cor:ordinal-fixed}.} If $p(\theta)=\delta_{\theta_0}$, every draw is
$\theta^{(s)}=\theta_0$ and the outer average in~\eqref{eq:ordinal_cf_marginal} collapses. The
estimator is then a single Laplace fit at $\mathbf b(\theta_0)$ followed by the Monte Carlo
average over $(f_r^F,\tilde f_r)$, the construction of \citet{chu2005gaussian}.
\end{proof}

\begin{remark}[Probit and logistic links at $C=2$]
At $C=2$ the ordinal likelihood has a single cutpoint and a probit link, whereas
Corollary~\ref{cor:binary} uses a logistic link \citep{mackay1992evidence,bishop2006pattern}. The two
are close but distinct binary classifiers. The counterfactual argument does not depend on the link;
only the per-point derivatives of Lemma~\ref{lem:ordinal-map} change.
\end{remark}

\begin{remark}[Which coupling each proposition selects]
\label{rem:couplings}
Given $\theta$, $X_r$ and $X_r^{\CF}$ are the same non-decreasing threshold map
applied to $f_r^F+N_r$ and $\tilde f_r+N_r$; with $U:=\Phi(N_r)$ the pair is
$\bigl(F_p^{-1}(U),F_q^{-1}(U)\bigr)$, the comonotone coupling (the monotonicity
assumption of \citealp{pearl1999probabilities,lu2020sample}). For $C=2$ and
posterior averaging see Remark~\ref{rem:couplings-c2}.
\end{remark}

\begin{remark}[Couplings at $C=2$ and under posterior averaging]
\label{rem:couplings-c2}
This continues Remark~\ref{rem:couplings}. At $C=2$ the Gumbel-max coupling of
Corollary~\ref{cor:nominal} reduces to the comonotone coupling: by
Corollary~\ref{cor:binary} it is the threshold model $\{U_r<p\}$ driven by a single uniform
noise, as is the ordinal model with one cutpoint. Two Bernoulli marginals admit a one-parameter
family of couplings, but the comonotone coupling is the only monotone one, and its single-noise
form is why Corollary~\ref{cor:binary} admits a closed form where the general case does not. The links still differ (probit here, logit in Corollary~\ref{cor:binary}), so
the two propositions agree at $C=2$ as couplings, not as fitted models, and they diverge for
$C\ge3$ (\S\ref{sec:scm_suite} measures the cost). The identification also holds only at a fixed
$\theta$: after averaging over $(f_r^F,\tilde f_r,\theta)$ in~\eqref{eq:ordinal_cf_marginal}, the
estimator is a mixture of comonotone couplings over varying marginals, which need not itself be
comonotone.
\end{remark}

\subsection{\rev{Partially ordered outcomes: the fit and the map $\pi$}}
\label{sec:proof-poset}\label{app:poset}\label{app:poset-theory}

\begin{figure}[htbp]
\centering
\color{revcolor}
\begin{tikzpicture}[scale=0.62, every node/.style={font=\scriptsize}]
  \foreach \i/\j in {1/1,2/1,3/1,1/2,1/3}{\node[var, ordinal, minimum size=4.5mm] (g\i\j) at (\i-\j,\i+\j-2) {\i\j};}
  \node[var, nominal, minimum size=4.5mm] (T) at (0,3.2) {$\top$};
  \draw[-] (g11)--(g21) (g21)--(g31) (g11)--(g12) (g12)--(g13) (g31)--(T) (g13)--(T);
  \node at (0,-1.1) {merged-top grid, $\pi$ not injective};
\end{tikzpicture}\hspace{10mm}
\begin{tikzpicture}[scale=0.62, every node/.style={font=\scriptsize}]
  \node[var, ordinal, minimum size=4.5mm] (a) at (0,0) {${<}$HS};
  \node[var, ordinal, minimum size=4.5mm] (b) at (0,1) {HS};
  \node[var, ordinal, minimum size=4.5mm] (c) at (0,2) {SC};
  \node[var, nominal, minimum size=4.5mm] (d1) at (-0.8,3) {AV};
  \node[var, nominal, minimum size=4.5mm] (d2) at (0.8,3) {AA};
  \node[var, ordinal, minimum size=4.5mm] (e) at (0,4) {BA};
  \node[var, nominal, minimum size=4.5mm] (f1) at (-0.8,5) {MA};
  \node[var, nominal, minimum size=4.5mm] (f2) at (0.8,5) {PS};
  \node[var, ordinal, minimum size=4.5mm] (g) at (0,6) {PhD};
  \draw[-] (a)--(b)--(c) (c)--(d1) (c)--(d2) (d1)--(e) (d2)--(e) (e)--(f1) (e)--(f2) (f1)--(g) (f2)--(g);
  \node at (0,-1.1) {Adult education, a tiered order};
\end{tikzpicture}
\caption{Hasse diagrams of two partially ordered outcomes. Left: the merged-top grid, the $3{\times}3$ product of chains with its
upper block merged into a single top class $\top$ (Proposition~\ref{prop:unified}). Right: Adult's
education variable merged to nine classes (${<}$HS/HS: below/high school; SC: some college; AV/AA: vocational/academic
associate; BA: Bachelor's; MA/PS: Master's/professional school; PhD: doctorate), an order made of tiers of tied
classes (Corollary~\ref{cor:tiered}).}
\label{fig:posets}
\end{figure}

\begin{revblock}
\textbf{Which posets are represented exactly.} The construction of
Proposition~\ref{prop:unified} treats the classes of one block as interchangeable,
each with the same classes above and below it, while an order-preserving $\pi$
sends the least and greatest cells to a least and a greatest block. A poset is
therefore represented exactly when its tied classes are interchangeable in this
sense and the order on blocks has a least and a greatest element and a
presentation $\pi$ (Lemma~\ref{lem:lattice-pi} builds one for any lattice). Any
other poset (a crown, say) must be coarsened by merging classes into blocks, which
drops some of its comparabilities.

\begin{lemma}[Laplace approximation for the $k$-latent poset likelihood]
\label{lem:poset-laplace}
Stack the latent values coordinate-major, $\mathbf F:=(\mathbf f_1^\top,\ldots,\mathbf f_k^\top)^\top\in\R^{kn}$,
with prior $\mathcal N(\mathbf 0,\Kb_k)$, $\Kb_k:=I_k\otimes\Kb$. Let
$\pi_i(\mathbf f_i):=\sum_{\mathbf g\in\pi^{-1}(A(x_r^i))}m_{\mathbf g}(\mathbf f_i)$ be the
threshold term of the likelihood~\eqref{eq:poset_identity} of point $i$, with gradient $\nabla\log\pi_i\in\R^k$ and
Hessian $H_i\in\R^{k\times k}$, and let $\hat W$ be the $kn\times kn$ matrix with
block $(j,l)$ the diagonal matrix $\diag_i\bigl([\,-H_i\,]_+\bigr)_{jl}$, where
$[\cdot]_+$ projects a symmetric matrix onto the positive semi-definite cone. At a mode
$\hat{\mathbf F}=\Kb_k\nabla\log\bm\pi(\hat{\mathbf F})$ the Laplace posterior is
$\mathcal N\bigl(\hat{\mathbf F},(\Kb_k^{-1}+\hat W)^{-1}\bigr)$. At any two inputs
$\mathbf x^\star,\mathbf x^\dagger$ the latent vectors are jointly Gaussian with
$\E[f_j(\mathbf x^\star)]=\mathbf k_\star^\top\nabla_j\log\bm\pi(\hat{\mathbf F})$ and
\begin{equation}\label{eq:poset_pred}
    \Cov\bigl[f_j(\mathbf x^\star),f_l(\mathbf x^\dagger)\bigr]
    =\delta_{jl}\,k(\mathbf x^\star,\mathbf x^\dagger)
    -\bigl(\mathbf e_j\otimes\mathbf k_\star\bigr)^\top
     \hat W^{1/2}\bigl(I+\hat W^{1/2}\Kb_k\hat W^{1/2}\bigr)^{-1}\hat W^{1/2}
     \bigl(\mathbf e_l\otimes\mathbf k_\dagger\bigr).
\end{equation}
\end{lemma}

\begin{proof}
The log posterior $\Psi(\mathbf F)=\sum_i\log\pi_i(\mathbf f_i)-\tfrac12\mathbf F^\top\Kb_k^{-1}\mathbf F$
has gradient $\nabla\log\bm\pi-\Kb_k^{-1}\mathbf F$ and Hessian $-W-\Kb_k^{-1}$, where $W$ has
per-point blocks $-H_i$ that couple the $k$ coordinates of a point but not distinct
points. The mode condition and the Gaussian follow as in Lemma~\ref{lem:map}, and the
predictive moments follow from Lemma~\ref{lem:map}(iv) with $\Kb$ replaced by $\Kb_k$ and the
block-diagonal $\hat W_c$ replaced by $\hat W$. The Woodbury form
$(\Kb_k+\hat W^{-1})^{-1}=\hat W^{1/2}(I+\hat W^{1/2}\Kb_k\hat W^{1/2})^{-1}\hat W^{1/2}$ avoids
inverting a singular $\hat W$. When $\pi$ is not injective, $\pi_i$ is a sum of
box masses, which need not be log-concave, so $-H_i$ can be indefinite away from the mode. The
projection $[\cdot]_+$ is the usual remedy, applied inside a damped Newton
iteration (step halved until $\Psi$ increases). At a mode where every $-H_i\succeq0$
it changes nothing; Appendix~\ref{app:estimators} reports how often it is
active at the fitted modes. When
$\pi$ is injective, $\pi_i(\mathbf f_i)=\prod_jm_j(f_{ij})$ factorises, every $H_i$ is
diagonal, $\hat W$ is block diagonal, and the posterior is a product of $k$
independent single-latent ordinal posteriors (Lemma~\ref{lem:ordinal-map}).
\end{proof}

The counterfactual construction for these outcomes is Proposition~\ref{prop:unified}, proved in Appendix~\ref{sec:proof-general}; this section adds the fit and the construction of $\pi$.
\end{revblock}

\subsubsection{Lattices and tiered orders}
\label{app:poset-more}
\begin{revblock}
\begin{lemma}[Constructing $\pi$]
\label{lem:lattice-pi}
Let $P$ be a finite poset in which every pair of classes has a unique least
common upper class and greatest common lower class (a \emph{lattice}), with a realiser $L_1,\ldots,L_k$, i.e.\ linear
extensions with $p\preceq q$ iff $p$ precedes or equals $q$ in every $L_j$
\citep{dushnik1941partially}. Let $\phi(p):=(\mathrm{rank}_{L_1}(p),\ldots,\mathrm{rank}_{L_k}(p))\in[\lvert P\rvert]^k$ and
$\pi(\mathbf g):=\bigvee\{p:\phi(p)\le\mathbf g\}$ (with $\bigvee\emptyset:=\bot$).
Then $\pi$ is order-preserving and surjective, with $\pi\circ\phi=\mathrm{id}_P$.
\end{lemma}

\begin{proof}[Proof of Lemma~\ref{lem:lattice-pi}]
If $\mathbf g\le\mathbf g'$ then $\{p:\phi(p)\le\mathbf g\}\subseteq\{p:\phi(p)\le\mathbf g'\}$, and joins
are monotone under inclusion, so $\pi$ is order-preserving. By the realiser
property, $\phi(q)\le\phi(p)$ iff $q\preceq p$. So $\{q:\phi(q)\le\phi(p)\}$ is the
down-set of $p$, whose join is $p$. Hence $\pi(\phi(p))=p$, and $\pi$ is
surjective. A realiser always exists, since every finite poset is the
intersection of its linear extensions \citep{szpilrajn1930extension}. The
smallest one has size equal to the order dimension \citep{dushnik1941partially},
and that is the $k$ to use.
\end{proof}

An order-preserving $\pi$ sends the least and greatest cells of $\mathcal G$
to a least and a greatest block. Classes that are tied in rank but have no
geometry of their own therefore go into a block and are decided by a race, as in
the tiered case:

\begin{corollary}[Tiered orders]
\label{cor:tiered}
Let the classes form tiers (levels) $A_1,\ldots,A_L$ of tied classes, so that
$p\prec q$ iff $\mathrm{lev}(p)<\mathrm{lev}(q)$ (a \emph{weak order}, or ordinal sum
of antichains). Proposition~\ref{prop:unified} with $k=1$, $C_1=L$, $\pi=\mathrm{id}$
and the tiers as blocks gives
\begin{equation}\label{eq:weak_scm}
    \Lambda_r:=\kappa\bigl(f_{r,1}(\mathbf X_{\pa(r)})+N_r\bigr),\qquad
    X_r:=\argmax_{a\in A_{\Lambda_r}}\bigl\{\ell_{\Lambda_r,a}(\mathbf X_{\pa(r)})+G_{r,a}\bigr\}:
\end{equation}
the level by the cutpoint mechanism of Corollary~\ref{cor:ordinal}, and the class
within the level by the Gumbel-max mechanism of Corollary~\ref{cor:nominal}.
Abduction draws $N_r$ from the truncated normal~\eqref{eq:ordinal_abduct} on the
factual level's interval and the factual level's Gumbels
from~\eqref{eq:gumbel_abduct}; prediction reruns the threshold, then the race of
whichever level is reached, and never lowers the level when
$\tilde f_{r,1}\ge f^F_{r,1}$. Fitting takes one GP-ordinal fit on all rows and one
classifier per multi-class level, on that level's rows. The case $L=1$ is
Corollary~\ref{cor:nominal}, and the case with every $\lvert A_l\rvert=1$ is
Corollary~\ref{cor:ordinal}.
\end{corollary}

\begin{proof}
Every statement is Proposition~\ref{prop:unified} and
Corollary~\ref{cor:poset-monotone} at these choices: with $k=1$ and $\pi=\mathrm{id}$
there is one box per level, so~\eqref{eq:poset_abduct} is~\eqref{eq:ordinal_abduct}.
\end{proof}

\end{revblock}

\FloatBarrier
\section{Experimental setup}
\label{app:protocol-v2}

\begin{revblock}
Node letters give the type throughout: $X$ continuous, $B$ binary, $C$ nominal,
$O$ ordinal, $P$ partially ordered, as in Figure~\ref{fig:scm_levels}.
\end{revblock}

\subsection{Estimators}
\label{app:estimators}
\begin{revblock}
Classifiers are one-versus-rest binary Laplace GPCs with automatic relevance determination (ARD) RBF kernels. Their
ARD directions come from a GP-regression pilot on the one-hot targets, and a
global length-scale multiplier ($\{0.5,1,2\}$) and signal variance
($\{1,4,16\}$) are chosen by the Laplace evidence. This is about ten times
faster than \texttt{sklearn}'s gradient search and reproduces the
\texttt{sklearn} fits of Appendix~\ref{app:protocol} (Table~\ref{tab:core-regression}).
Tables~\ref{tab:scm_suite_results} and~\ref{tab:ordering} use it. The ordinal
and poset GPs use the same ARD rescaling, with cutpoints set to the MAP of the
profiled Laplace evidence under the prior of Corollary~\ref{cor:ordinal},
initialised at the probit of the cumulative class frequencies. Continuous nodes
use \texttt{sklearn}'s GP regressor with ARD and a white-noise term. Parents
enter as standardised values (continuous, binary, ordinal) or one-hot vectors
(nominal, poset).

\textbf{Positive-semidefinite (PSD) projection.} At the fitted modes of the merged-top grid of
\S\ref{sec:poset-exp} (both noise laws, five seeds each), every per-point
curvature $-H_i$ is positive semi-definite, so the projection of
Lemma~\ref{lem:poset-laplace} is active only during the iteration, never at the
mode (\texttt{psd\_diagnostic.py}).

\textbf{Cost.} Fitting is dominated by the MAP cutpoint search of ordinal and
poset nodes, a quasi-Newton search with a Laplace fit per evaluation. On one
core, fitting a whole graph takes on average $3$, $12$, and $63$ minutes at
$d=8,12,16$, mostly in the poset nodes' $k$-latent fits. Abduction and
prediction take seconds.
\end{revblock}

\subsection{Protocol}
\label{app:protocol}

\textbf{Classifier fitting.} Tables~\ref{tab:scm_suite_results}
and~\ref{tab:ordering} use the grid-searched classifier of
Appendix~\ref{app:estimators}, which reproduces the \texttt{sklearn} fits
(Table~\ref{tab:core-regression}) and is the one every baseline was run with.
We bound each RBF length scale to $[0.1\hat\sigma_d,\,10\hat\sigma_d]$
rather than \texttt{sklearn}'s default $[10^{-5},10^{5}]$, whose ten-decade
range causes most optimiser restarts to start far outside the data support,
where the kernel is flat and fits get stuck. Under the default, the Single-parent
SCM's baseline-class latent came out exactly constant ($P\approx0.20$ vs.\ a true
$0.05$) and stayed there even with five times more restarts. The bounded fit
attains a \emph{higher} marginal likelihood, so this is a correction, not a
modelling choice.

\textbf{Seeds.} Each entry in Table~\ref{tab:scm_suite_results} is a mean
$\pm$ sd over $10$ random seeds at $n=500$. Each seed redraws the training data,
refits every GP, and re-runs all factual/intervention scenarios, so the spread
covers the whole pipeline. For example, the mismatched arm of the Interaction
SCM ranges from $0.11$ to $0.41$ across seeds.

\textbf{Monte Carlo budget.} Every counterfactual probability is a Monte Carlo
average over $4{,}000$ draws of the joint latent and the abducted noise.
Re-running the Interaction SCM at $n_{\mathrm{mc}}$ from $250$ to $16{,}000$ moves
$\mathrm{TV}_{\mathrm{cf}}$ by $+0.012$ between the first two budgets and by at
most $0.0015$ beyond $2{,}000$, two orders of magnitude below the seed-to-seed sd
of $0.100$.

\begin{revblock}
\textbf{Scenarios.} Core suite: $8$ factual training rows, $\doop(X_1=-x_1^F)$.
Poset SCMs: $8$ rows, each with $x_1$ and with $x_2$ shifted by $\pm1$ sd,
alternating. Random SCMs: $5$ rows, each with the root $v_0$ and with one random
interior node that has discrete descendants. A continuous target is shifted by
$\pm1$ sd; a discrete one is set to a different random class. A graph is
redrawn (at most $50$ times) if any class of any discrete node has under $3\%$ of
the training sample. Adult: $6$ applicants with income ${\le}50$K and no
Bachelor's degree, four queries each.
\end{revblock}

\subsection{Structural causal models}
\label{app:scms}
\textbf{Core suite.} The structural equations of the SCMs of Figure~\ref{fig:scm_levels}:

\noindent\begin{minipage}[t]{0.315\linewidth}
\small
\textbf{Binary}
\begin{align*}
  X_1 &:= U_1,\ U_1\sim\mathcal{N}(0,1), \\
  B_1 &:= \mathbb{I}\bigl[2X_1{+}U_2{>}0\bigr], \\
  &\quad U_2\sim\mathcal{N}(0,1).
\end{align*}
\end{minipage}\hfill
\begin{minipage}[t]{0.315\linewidth}
\small
\textbf{Single-parent}
\begin{align*}
  X_1 &:= U_1,\ U_1\sim\mathcal{N}(0,1), \\
  C_1 &:= \argmax_{c\in\{1,2,3\}}\{\eta_c(X_1){+}G_c\}, \\
  &\eta(x_1){=}(0,1.2x_1,{-}x_1{+}0.4).
\end{align*}
\end{minipage}\hfill
\begin{minipage}[t]{0.315\linewidth}
\small
\textbf{Interaction}
\begin{align*}
  X_1,X_2 &\sim \mathcal{N}(0,1)\text{ iid}, \\
  Y_3 &:= X_1X_2{+}N_3, \\
  &\quad N_3\sim\mathcal{N}(0,1), \\
  O_1 &:= 1{+}\mathbb{I}\{Y_3{>}b_1\} \\
  &\quad{+}\mathbb{I}\{Y_3{>}b_2\}.
\end{align*}
\end{minipage}

\noindent\begin{minipage}[t]{0.6\linewidth}
\small
\textbf{Chain (mixed mechanisms)}
\begin{align*}
  X_1 &:= U_1, \\
  O_1 &:= 1{+}\mathbb{I}\{X_1{+}N_2{>}b_1\}{+}\mathbb{I}\{X_1{+}N_2{>}b_2\}, \\
  &\quad N_2\sim\mathcal{N}(0,1), \\
  C_1 &:= \argmax_c\{\tilde\eta_c(X_1,O_1)+G_{3,c}\}, \\
  &\quad\tilde\eta(x_1,o_1) = (0,\;u,\;{-}u),\ u{=}x_1{+}(o_1{-}2).
\end{align*}
\end{minipage}

Computing the ground-truth counterfactual for $C_1$ under
$\doop(X_1=\tilde x_1)$ in the Chain SCM chains two different per-node abductions through the
mediator $O_1$: a truncated-normal posterior for $O_1$'s noise
(\eqref{eq:ordinal_abduct}) and a truncated-Gumbel posterior for
$C_1$'s (\eqref{eq:gumbel_abduct}).

\begin{revblock}
\textbf{Poset SCM.} $X_1,X_2\sim\mathcal N(0,1)$ independently.
$P_1=\pi(\kappa(s(X_1,X_2)+\mathbf N))$, with $\mathbf N\sim\mathcal N(\mathbf 0,I_2)$,
on the merged-top grid of Figure~\ref{fig:posets}. Here $s:\R^2\to\R^2$ is a
random smooth map (\S\ref{app:mechanisms}), and the cutpoints sit at the tertiles
of each coordinate. $C_1=\argmax_c\{s'_c(P_1,X_1)+G_c\}$ is a three-class
Gumbel-max race. Ground truth uses the rejection abduction of
\S\ref{sec:additional}. The poset fit learns its cutpoints by MAP.

\textbf{Partially ordered experiment.} The same graph, with $P_1$ in four
designs. \emph{Merged-top, Gaussian} uses Proposition~\ref{prop:unified}'s own
noise, with $\pi$ not injective. \emph{Merged-top, logistic} uses logistic noise,
so only the link is wrong. \emph{Product} is a $2{\times}3$ product of chains
with logistic noise. \emph{Tiered} has tiers $\{a\},\{b,c\},\{d,e\},\{f\}$, a
logistic-noise tier and a correlated multinomial-probit race within each tier,
so both of Corollary~\ref{cor:tiered}'s components are misspecified.

\textbf{Semi-synthetic Adult.} We use the graph of \citet{nabi2018fair}. Sex,
age, and native country (US or not) are roots and drive marital status. Marital
status drives education. These drive workclass, occupation, and hours per week,
and everything drives income. Education is the nine-class tiered order of
Figure~\ref{fig:posets}. Marital status (3), workclass (3), and occupation (4)
are nominal; sex, native country, and income are binary. To keep the test from
being circular, the true mechanisms are \emph{not} our families.
Gradient-boosted trees fitted to the real data supply class scores, on which
nominal nodes run a correlated probit race. Education's level is a
logistic-noise threshold of a boosted index, with a probit race within each
level. Continuous nodes add resampled real residuals to a boosted mean. Every
marginal is within $0.04$ of the real data's. Queries are recourse-style:
$\doop(\text{education}=\text{Bachelors})$,
$\doop(\text{marital}=\text{married})$, $\doop(\text{hours}\mathrel{+}=10)$, and
$\doop(\text{age}\mathrel{+}=5)$.
\end{revblock}

\subsection{Mechanism families}
\label{app:mechanisms}
\begin{revblock}
Here $s(\mathbf z)$ denotes a random smooth map $\R^p\to\R^m$:
$\Omega_2\tanh(\Omega_1\mathbf z+\mathbf b_1)+\Omega_\ell\mathbf z+\mathbf b_2$ with Gaussian weights. It is
applied to parents encoded as real values, with nominal and poset parents
mapped through a fixed random two-dimensional embedding.
\begin{itemize}
\item \emph{Nominal, probit:} $X=\argmax_c\{s_c(\mathbf z)+\varepsilon_c\}$,
  $\bm\varepsilon\sim\mathcal N(\mathbf 0,R)$, $R$ a random correlation matrix.
\item \emph{Nominal, nearest prototype:} $X=\argmin_c\lVert s(\mathbf z)+\bm\varepsilon-\bm\mu_c\rVert$,
  $s:\R^p\to\R^2$, $\bm\varepsilon\sim\mathcal N(\mathbf 0,0.36I_2)$, prototypes on a circle.
\item \emph{Ordinal, continuation ratio:} starting from class $1$, advance
  from $c$ to $c+1$ iff $s(\mathbf z)+o_c+\varepsilon_c>0$, $\varepsilon_c$ iid logistic,
  stopping at the first failure. It is monotone in every $\varepsilon_c$ but not comonotone.
\item \emph{Ordinal, threshold:} $X=\kappa(s(\mathbf z)+\sigma(\mathbf z)\varepsilon)$, $\varepsilon$
  logistic or $t_3$, $\sigma$ constant or $0.5+\mathrm{softplus}(s'(\mathbf z))$. It is
  comonotone, so only its link is misspecified.
\item \emph{Poset:} \eqref{eq:poset_scm} with logistic $\mathbf N$, on the
  $2{\times}2$ and $2{\times}3$ products and the merged-top grid of Figure~\ref{fig:posets}.
\item \emph{Binary:} $X=\mathbb I[s(\mathbf z)+\varepsilon>0]$, $\varepsilon$ logistic.
\item \emph{Continuous:} $X=s(\mathbf z)+0.5\,\varepsilon$, with $\varepsilon$ bimodal
  ($\pm0.8+0.6\,\mathcal N(0,1)$), Laplace, or $t_3$.
\end{itemize}
Cutpoints and stage offsets are set from a pilot sample to give roughly
balanced classes. The control rows of Table~\ref{tab:large-full} use Gumbel-max
nominal nodes, Gaussian threshold ordinals, and Gaussian posets.
\end{revblock}

\subsection{Baselines}
\label{app:adult-setup}
\begin{revblock}
Every method is defined in \S\ref{sec:setup}. All-categorical
(\S\ref{sec:ordering}) is not a baseline method from the literature but the
type-agnostic model we built first, combining the GP-SCM of
\citet{DBLP:journals/corr/abs-2006-06831} with the Gumbel-max coupling of
\citet{oberst2019counterfactual}; TACC replaces its coupling on ordinal and
partially ordered nodes and is otherwise identical, so it appears only in the
comparison of \S\ref{sec:ordering}. The negative baselines apply continuous machinery to a
discrete node. \emph{ANM + round} treats the integer code as a real number, fits
the additive-noise GP, computes its exactly abducted continuous counterfactual,
and rounds it to the nearest code. \emph{One-hot GPR} fits one GP regression per
class on the indicator target and takes the argmax of the per-class
counterfactuals. \emph{Relaxed} keeps our estimator but replaces the event that
class $c^\ast$ won the race by the Gumbel-softmax likelihood
$\mathrm{softmax}((\bm\ell^F+\mathbf G)/\tau)_{c^\ast}$
\citep{jang2017categorical,maddison2017concrete}, abducting the Gumbels by
importance resampling; $\tau\to0$ recovers Corollary~\ref{cor:nominal}.
\end{revblock}

\FloatBarrier
\section{Additional results}
\label{app:results}

\subsection{Core suite}
\label{app:ablations}
\rev{The asymmetry of \S\ref{sec:ordering} runs the other way for the fit. Fitting
an ordinal model to genuinely nominal Gumbel-max data (the Single-parent SCM)
degrades $\mathrm{TV}_{\mathrm{int}}$ itself, by a factor of six
($0.033\to0.193$; Table~\ref{tab:ordering}), because forcing a monotone latent
score onto classes with no true order is a harder regression problem outright.
Imposing order costs the fit; imposing nominality does not. At $C=2$ the two
couplings coincide (Remark~\ref{rem:couplings-c2}), so on binary nodes the two
models are identical, and refitting the Binary SCM with the probit link barely
moves $\mathrm{TV}_{\mathrm{cf}}$ (Figure~\ref{fig:int-vs-cf}). Nor is it enough
for an order merely to be present; the order type must be right. On
continuation-ratio ordinals, whose stages draw independent noises, the Gumbel-max
coupling beats the cutpoint coupling ($0.067$ against $0.139$, on $16$ of $20$
random graphs; Table~\ref{tab:ordering}), although the classes are totally
ordered. The coupling is a structural assumption about how the variable is
generated, which a ranking of its classes does not settle.}

\begin{figure}[htbp]
\centering
\includegraphics[width=0.7\linewidth]{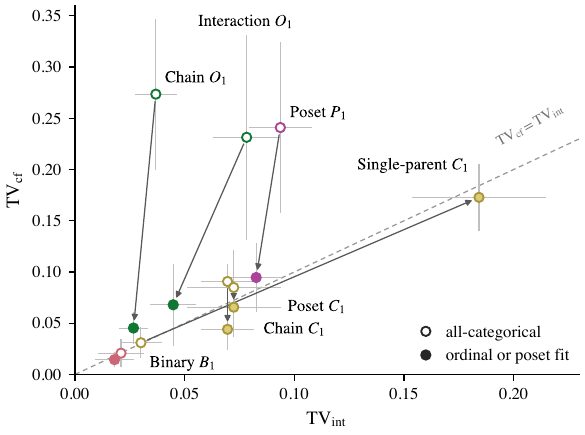}
\caption{\rev{Interventional against counterfactual error for every node of
the synthetic SCMs, under the original \texttt{sklearn} protocol
(Appendix~\ref{app:protocol}), mean $\pm$ sd over $10$ seeds at $n=500$.
Hollow: all-categorical; filled: ordinal fit (poset fit for the
Poset SCM); colour: the node's type, as in Figure~\ref{fig:scm_levels}. Where
ordered data is fit with the Gumbel-max coupling, at the node or its mediator,
the arrow is nearly vertical: the fitted equation barely changes but the
counterfactual error falls by up to a factor of $6$. Chain $C_1$ and Poset
$C_1$ have identical fitted equations in both fits. Imposing a false order on
nominal data (Single-parent) instead moves the point along the diagonal, so
that misspecification is visible in the fit. Chain $O_1$'s poset fit is
omitted; it lies on its ordinal fit.}}
\label{fig:int-vs-cf}
\end{figure}

\begin{table}[htbp]
\centering
\color{revcolor}
\caption{Table~\ref{tab:scm_suite_results} with standard deviations:
$\mathrm{TV}_{\mathrm{cf}}$ and $\mathrm{TV}_{\mathrm{int}}$ (shaded), mean $\pm$ sd over
$10$ seeds at $n=500$. Bold: within one sd of the lowest in the row (mean
$\le$ lowest mean $+$ its sd).}
\label{tab:scm_suite_full}
\small
\setlength{\tabcolsep}{4pt}
\begin{adjustbox}{max width=\textwidth}
\begin{tabular}{llllcccc}
\hline
SCM & Node & Mediator & Error & TACC & ANM + round & one-hot GPR & relaxed \\
\hline
Binary & $B_1$ &  & $\mathrm{TV}_{\mathrm{cf}}$ & $\mathbf{0.020 \pm 0.013}$ & $0.123 \pm 0.055$ & $0.125 \pm 0.055$ & $\mathbf{0.032 \pm 0.017}$ \\
 &  &  & \tvi{$\mathrm{TV}_{\mathrm{int}}$} & \tvi{$\mathbf{0.024 \pm 0.007}$} & \tvi{$0.086 \pm 0.018$} & \tvi{$0.043 \pm 0.012$} & \tvi{$\mathbf{0.024 \pm 0.007}$} \\[2pt]
Single-parent & $C_1$ &  & $\mathrm{TV}_{\mathrm{cf}}$ & $\mathbf{0.033 \pm 0.014}$ & $0.238 \pm 0.061$ & $0.246 \pm 0.067$ & $0.070 \pm 0.028$ \\
 &  &  & \tvi{$\mathrm{TV}_{\mathrm{int}}$} & \tvi{$\mathbf{0.033 \pm 0.007}$} & \tvi{$0.222 \pm 0.011$} & \tvi{$0.077 \pm 0.017$} & \tvi{$\mathbf{0.033 \pm 0.007}$} \\[2pt]
Interaction & $O_1$ &  & $\mathrm{TV}_{\mathrm{cf}}$ & $\mathbf{0.074 \pm 0.041}$ & $0.208 \pm 0.103$ & $0.316 \pm 0.145$ & \multicolumn{1}{c}{$=$ TACC} \\
 &  &  & \tvi{$\mathrm{TV}_{\mathrm{int}}$} & \tvi{$\mathbf{0.060 \pm 0.012}$} & \tvi{$0.093 \pm 0.013$} & \tvi{$0.097 \pm 0.013$} & \multicolumn{1}{c}{\tvi{$=$ TACC}} \\[2pt]
Chain & $O_1$ &  & $\mathrm{TV}_{\mathrm{cf}}$ & $\mathbf{0.031 \pm 0.015}$ & $0.252 \pm 0.060$ & $0.422 \pm 0.096$ & \multicolumn{1}{c}{$=$ TACC} \\
 &  &  & \tvi{$\mathrm{TV}_{\mathrm{int}}$} & \tvi{$\mathbf{0.028 \pm 0.008}$} & \tvi{$0.111 \pm 0.012$} & \tvi{$0.078 \pm 0.014$} & \multicolumn{1}{c}{\tvi{$=$ TACC}} \\[2pt]
Chain & $C_1$ & $O_1$ & $\mathrm{TV}_{\mathrm{cf}}$ & $\mathbf{0.038 \pm 0.013}$ & $0.290 \pm 0.053$ & $0.333 \pm 0.077$ & $0.060 \pm 0.025$ \\
 &  &  & \tvi{$\mathrm{TV}_{\mathrm{int}}$} & \tvi{$\mathbf{0.051 \pm 0.009}$} & \tvi{$0.287 \pm 0.014$} & \tvi{$0.095 \pm 0.010$} & \tvi{$\mathbf{0.051 \pm 0.009}$} \\[2pt]
Poset & $P_1$ &  & $\mathrm{TV}_{\mathrm{cf}}$ & $\mathbf{0.114 \pm 0.029}$ & $0.358 \pm 0.073$ & $0.277 \pm 0.083$ & --- \\
 &  &  & \tvi{$\mathrm{TV}_{\mathrm{int}}$} & \tvi{$\mathbf{0.085 \pm 0.010}$} & \tvi{$0.263 \pm 0.048$} & \tvi{$0.159 \pm 0.029$} & \tvi{---} \\[2pt]
Poset & $C_1$ & $P_1$ & $\mathrm{TV}_{\mathrm{cf}}$ & $\mathbf{0.057 \pm 0.018}$ & $0.156 \pm 0.062$ & $0.139 \pm 0.039$ & --- \\
 &  &  & \tvi{$\mathrm{TV}_{\mathrm{int}}$} & \tvi{$\mathbf{0.069 \pm 0.018}$} & \tvi{$0.237 \pm 0.065$} & \tvi{$0.103 \pm 0.022$} & \tvi{---} \\[2pt]
\hline
\end{tabular}
\end{adjustbox}
\end{table}

\textbf{Sample size} (Table~\ref{tab:sizes}). Where the coupling is correct,
the error falls with $n$ as GP asymptotics predict: from $0.061$ to $0.021$ on the
Single-parent SCM between $n=100$ and $n=1000$, and from $0.055$ to $0.022$ on the
Chain SCM's refitted mediator. Where it is wrong, the fit keeps converging but the
counterfactual does not. $\mathrm{TV}_{\mathrm{int}}$ falls by $61\%$ on the
Interaction SCM ($0.144\to0.056$) and by $74\%$ on the Chain SCM's mismatched
mediator ($0.091\to0.024$), while $\mathrm{TV}_{\mathrm{cf}}$ flattens near
$0.21$--$0.24$ ($0.253\to0.207$ and $0.279\to0.242$). The ratio
$\mathrm{TV}_{\mathrm{cf}}/\mathrm{TV}_{\mathrm{int}}$ therefore grows with $n$
under misspecification (from $1.8$ to $3.7$, and from $3.1$ to $10.2$) and stays
near $1$ in every correctly specified row: a better fit does not reveal a wrong
coupling, at any sample size.

\begin{table}[htbp]
\centering
\small
\setlength{\tabcolsep}{4pt}
\caption{Counterfactual error $\mathrm{TV}_{\mathrm{cf}}$ against training-set
size, mean over $10$ seeds. Rows marked $^\dagger$ use an estimator whose
coupling does not match the true mechanism; their error flattens rather than
falling with $n$. \rev{In the refit block, ordinal nodes use Corollary~\ref{cor:ordinal} (probit at $C=2$), and Chain's $C_1$ keeps its Gumbel-max fit but is fed the refit $O_1$.}}
\label{tab:sizes}
\begin{tabular}{llcccc}
\hline
SCM & Node & $n=100$ & $n=250$ & $n=500$ & $n=1000$ \\
\hline
\multicolumn{6}{l}{\rev{\emph{all-categorical}}} \\
Binary & $B_1$ & $0.040$ & $0.022$ & $0.021$ & $0.015$ \\
Single-parent & $C_1$ & $0.061$ & $0.045$ & $0.031$ & $0.021$ \\
Interaction & $O_1$$^\dagger$ & $0.253$ & $0.226$ & $0.231$ & $0.207$ \\
Chain & $C_1$ & $0.119$ & $0.084$ & $0.091$ & $0.071$ \\
Chain & $O_1$$^\dagger$ & $0.279$ & $0.249$ & $0.274$ & $0.242$ \\
\hline
\multicolumn{6}{l}{\rev{\emph{type-matched refit}}} \\
Binary & $B_1$ & $0.039$ & $0.021$ & $0.014$ & $0.013$ \\
Interaction & $O_1$ & $0.083$ & $0.094$ & $0.068$ & $0.055$ \\
Chain & $C_1$ & $0.081$ & $0.056$ & $0.044$ & $0.024$ \\
Chain & $O_1$ & $0.055$ & $0.036$ & $0.045$ & $0.022$ \\
\hline
\end{tabular}
\end{table}

\textbf{Ablating the estimator} (Table~\ref{tab:baselines}), on the Single-parent and
Interaction SCMs. \emph{Abduction does the work.} Ignoring the factual outcome
(``interventional'') gives $\mathrm{TV}_{\mathrm{cf}}=0.295$ on the Single-parent SCM,
against $0.031$ for the full estimator. Conditioning the latent on the factual
outcome without abducting a noise scarcely helps ($0.300$): a shared latent
constrains the \emph{probabilities} but leaves the realised outcome free to be
redrawn. \emph{Posterior averaging barely moves the point estimate.} The plug-in
posterior mean matches the full estimator to three decimals on both SCMs. The
joint draw is still needed, since it makes Corollary~\ref{cor:consistency} exact and
supplies the spread that credible intervals require; averaging buys calibration,
not accuracy\rev{ (Appendix~\ref{app:calibration})}. \emph{The oracle isolates the
coupling.} Replacing the fitted class probabilities by the true ones brings the
Single-parent SCM to $0.005\pm0.001$, the Monte Carlo floor, but takes the
Interaction SCM only from $0.231$ to $0.205$: about nine tenths of the
misspecification error is the coupling, not the GP.

\begin{table}[htbp]
\centering
\small
\setlength{\tabcolsep}{4pt}
\caption{Counterfactual error $\mathrm{TV}_{\mathrm{cf}}$ for five estimators on
the same data, mean $\pm$ sd over $10$ seeds at $n=500$. The Single-parent SCM's true mechanism
is Gumbel-max, matching the coupling; the Interaction SCM's is the ordinal cutpoint model,
so the coupling is misspecified.}
\label{tab:baselines}
\begin{tabular}{lcc}
\hline
estimator & Single-parent, $C_1$ & Interaction, $O_1$ \\
\hline
interventional (no abduction) & $0.295 \pm 0.079$ & $0.434 \pm 0.051$ \\
latent-only abduction & $0.300 \pm 0.080$ & $0.430 \pm 0.051$ \\
plug-in posterior mean & $0.031 \pm 0.014$ & $0.232 \pm 0.102$ \\
full estimator (Gumbel-max coupling) & $0.031 \pm 0.015$ & $0.231 \pm 0.100$ \\
oracle: true $\mathbf{p}$, same coupling & $0.005 \pm 0.001$ & $0.205 \pm 0.092$ \\
\hline
\end{tabular}
\end{table}

\rev{\textbf{Further checks.} Tables~\ref{tab:learned-cut}--\ref{tab:relaxed} give the
per-SCM numbers behind the learned-cutpoint, implementation, negative-baseline, and
relaxed-abduction comparisons of \S\ref{sec:findings}.}

\begin{table}[htbp]
\centering
\color{revcolor}
\small
\setlength{\tabcolsep}{4pt}
\caption{Given against learned (MAP or posterior) cutpoints for the ordinal
nodes of Table~\ref{tab:scm_suite_results}; ordinal
fit, and beneath it the poset estimator of \S\ref{sec:poset} with one axis. Mean $\pm$ sd over $10$ seeds.}
\label{tab:learned-cut}
\begin{tabular}{lllcccc}
\hline
 & & & \multicolumn{2}{c}{$\mathrm{TV}_{\mathrm{cf}}$} & \multicolumn{2}{c}{\tvi{$\mathrm{TV}_{\mathrm{int}}$}} \\
\cline{4-5}\cline{6-7}
SCM & Node & Fit & given & learned & \tvi{given} & \tvi{learned} \\
\hline
Interaction & $O_1$ & ordinal & $0.068 \pm 0.040$ & $0.070 \pm 0.041$ & \tvi{$0.045 \pm 0.011$} & \tvi{$0.049 \pm 0.010$} \\
 & & poset & $0.078 \pm 0.042$ & $0.078 \pm 0.043$ & \tvi{$0.051 \pm 0.010$} & \tvi{$0.054 \pm 0.009$} \\
Chain & $O_1$ & ordinal & $0.045 \pm 0.016$ & $0.041 \pm 0.024$ & \tvi{$0.027 \pm 0.007$} & \tvi{$0.029 \pm 0.010$} \\
 & & poset & $0.036 \pm 0.019$ & $0.038 \pm 0.020$ & \tvi{$0.024 \pm 0.008$} & \tvi{$0.029 \pm 0.008$} \\
\hline
\end{tabular}
\end{table}

\begin{table}[htbp]
\centering
\color{revcolor}
\small
\setlength{\tabcolsep}{4pt}
\caption{The grid-searched GPC of Appendix~\ref{app:protocol-v2} reproduces
the \texttt{sklearn} numbers of the original protocol (Appendix~\ref{app:protocol})
($\mathrm{TV}_{\mathrm{cf}}$, $10$ seeds, $n=500$). Ordinal nodes in the type-matched
column use MAP cutpoints (cf.\ Table~\ref{tab:learned-cut}).}
\label{tab:core-regression}
\begin{tabular}{llcc}
\hline
SCM & Node & all-categorical & TACC \\
\hline
Binary & $B_1$ & $0.020 \pm 0.014$ & $0.020 \pm 0.013$ \\
Single-parent & $C_1$ & $0.033 \pm 0.016$ & $0.033 \pm 0.014$ \\
Interaction & $O_1$ & $0.237 \pm 0.102$ & $0.074 \pm 0.041$ \\
Chain & $O_1$ & $0.275 \pm 0.075$ & $0.031 \pm 0.015$ \\
Chain & $C_1$ & $0.087 \pm 0.016$ & $0.038 \pm 0.013$ \\
\hline
\end{tabular}
\end{table}

\begin{table}[htbp]
\centering
\color{revcolor}
\small
\caption{Negative baselines on the core SCMs: $\mathrm{TV}_{\mathrm{cf}}$, mean
$\pm$ sd over $10$ seeds at $n=500$. ``Relabel'' is the $\mathrm{TV}$ between
ANM + round's counterfactual laws under two arbitrary codings of the same
nominal classes, on the same data, averaged over five recodings. It is zero for
any estimator that is invariant to relabelling, TACC included. Bold: within one sd of the lowest in the row, among the estimators (relabel excluded).}
\label{tab:negative}
\begin{tabular}{llcccc}
\hline
SCM & node type & TACC & ANM + round & relabel & one-hot GPR \\
\hline
Binary & binary & $\mathbf{0.020 \pm 0.013}$ & $0.123 \pm 0.055$ & --- & $0.125 \pm 0.055$ \\
Single-parent & nominal & $\mathbf{0.033 \pm 0.014}$ & $0.238 \pm 0.061$ & $0.100 \pm 0.052$ & $0.246 \pm 0.067$ \\
Interaction & ordinal & $\mathbf{0.074 \pm 0.041}$ & $0.208 \pm 0.103$ & --- & $0.316 \pm 0.145$ \\
Chain, $O_1$ & ordinal & $\mathbf{0.031 \pm 0.015}$ & $0.252 \pm 0.060$ & --- & $0.422 \pm 0.096$ \\
Chain, $C_1$ & nominal & $\mathbf{0.038 \pm 0.013}$ & $0.290 \pm 0.053$ & $0.235 \pm 0.073$ & $0.333 \pm 0.077$ \\
\hline
\end{tabular}
\end{table}

\begin{table}[htbp]
\centering
\color{revcolor}
\small
\setlength{\tabcolsep}{4pt}
\caption{Gumbel-softmax-relaxed abduction against the exact one (``TACC'',
$\tau\to0$): $\mathrm{TV}_{\mathrm{cf}}$ and consistency error, mean over $10$
seeds.}
\label{tab:relaxed}
\begin{tabular}{lllcccccc}
\hline
SCM & Node & metric & TACC & $\tau{=}0.05$ & $0.1$ & $0.5$ & $1$ & $2$ \\
\hline
Single-parent & $C_1$ & $\mathrm{TV}_{\mathrm{cf}}$ & $0.033$ & $0.035$ & $0.035$ & $0.070$ & $0.127$ & $0.193$ \\
Single-parent & $C_1$ & consistency error & $0.000$ & $0.017$ & $0.037$ & $0.160$ & $0.267$ & $0.364$ \\
Chain & $C_1$ & $\mathrm{TV}_{\mathrm{cf}}$ & $0.038$ & $0.037$ & $0.037$ & $0.060$ & $0.095$ & $0.140$ \\
Chain & $C_1$ & consistency error & $0.000$ & $0.015$ & $0.029$ & $0.122$ & $0.201$ & $0.269$ \\
Binary & $B_1$ & $\mathrm{TV}_{\mathrm{cf}}$ & $0.020$ & $0.021$ & $0.022$ & $0.032$ & $0.048$ & $0.069$ \\
Binary & $B_1$ & consistency error & $0.000$ & $0.008$ & $0.014$ & $0.064$ & $0.105$ & $0.141$ \\
\hline
\end{tabular}
\end{table}

\FloatBarrier
\subsection{Partially ordered experiment}
\label{app:poset-exp}
\begin{revblock}
\begin{table}[htbp]
\centering
\color{revcolor}
\small
\setlength{\tabcolsep}{4pt}
\caption{Partially ordered outcomes: $\mathrm{TV}_{\mathrm{cf}}$ of the poset
node $P_1$ and of its nominal child $C_1$, and $\mathrm{TV}_{\mathrm{int}}$ of $P_1$ (shaded),
mean $\pm$ sd over $10$ seeds at $n=500$. Bold: within one sd of the lowest in the row.}
\label{tab:poset}
\begin{adjustbox}{max width=\textwidth}
\begin{tabular}{lllcccc}
\hline
design & node & metric & TACC & linear ext. & ANM + round & one-hot GPR \\
\hline
Merged-top, Gaussian & $P_1$ & $\mathrm{TV}_{\mathrm{cf}}$ & $\mathbf{0.114 \pm 0.029}$ & $0.339 \pm 0.101$ & $0.358 \pm 0.073$ & $0.277 \pm 0.083$ \\
 & $P_1$ & \tvi{$\mathrm{TV}_{\mathrm{int}}$} & \tvi{$\mathbf{0.085 \pm 0.010}$} & \tvi{$0.210 \pm 0.052$} & \tvi{$0.263 \pm 0.048$} & \tvi{$0.159 \pm 0.029$} \\
 & $C_1$ & $\mathrm{TV}_{\mathrm{cf}}$ & $\mathbf{0.057 \pm 0.018}$ & $0.085 \pm 0.036$ & $0.156 \pm 0.062$ & $0.139 \pm 0.039$ \\
\hline
Merged-top, logistic & $P_1$ & $\mathrm{TV}_{\mathrm{cf}}$ & $\mathbf{0.107 \pm 0.031}$ & $0.361 \pm 0.111$ & $0.375 \pm 0.127$ & $0.305 \pm 0.120$ \\
 & $P_1$ & \tvi{$\mathrm{TV}_{\mathrm{int}}$} & \tvi{$\mathbf{0.079 \pm 0.013}$} & \tvi{$0.211 \pm 0.052$} & \tvi{$0.272 \pm 0.052$} & \tvi{$0.162 \pm 0.026$} \\
 & $C_1$ & $\mathrm{TV}_{\mathrm{cf}}$ & $\mathbf{0.060 \pm 0.019}$ & $0.112 \pm 0.043$ & $0.144 \pm 0.053$ & $0.140 \pm 0.041$ \\
\hline
Product $2{\times}3$, logistic & $P_1$ & $\mathrm{TV}_{\mathrm{cf}}$ & $\mathbf{0.126 \pm 0.033}$ & $0.392 \pm 0.079$ & $0.415 \pm 0.153$ & $0.355 \pm 0.139$ \\
 & $P_1$ & \tvi{$\mathrm{TV}_{\mathrm{int}}$} & \tvi{$\mathbf{0.080 \pm 0.012}$} & \tvi{$0.237 \pm 0.061$} & \tvi{$0.266 \pm 0.058$} & \tvi{$0.174 \pm 0.029$} \\
 & $C_1$ & $\mathrm{TV}_{\mathrm{cf}}$ & $\mathbf{0.064 \pm 0.018}$ & $0.125 \pm 0.041$ & $0.159 \pm 0.047$ & $0.158 \pm 0.048$ \\
\hline
Tiered & $P_1$ & $\mathrm{TV}_{\mathrm{cf}}$ & $\mathbf{0.125 \pm 0.058}$ & $0.225 \pm 0.090$ & $0.339 \pm 0.071$ & $0.390 \pm 0.105$ \\
 & $P_1$ & \tvi{$\mathrm{TV}_{\mathrm{int}}$} & \tvi{$\mathbf{0.078 \pm 0.013}$} & \tvi{$0.128 \pm 0.030$} & \tvi{$0.204 \pm 0.043$} & \tvi{$0.158 \pm 0.025$} \\
 & $C_1$ & $\mathrm{TV}_{\mathrm{cf}}$ & $\mathbf{0.060 \pm 0.016}$ & $0.090 \pm 0.028$ & $0.218 \pm 0.095$ & $0.200 \pm 0.075$ \\
\hline
\end{tabular}
\end{adjustbox}
\end{table}

\textbf{Our coupling is lowest on every row} (Table~\ref{tab:poset}).
On $P_1$ it cuts the counterfactual error of the best alternative by
$44$--$65\%$: from $0.277$ to $0.114$ on the merged-top grid, $0.305$ to $0.107$ with
logistic noise, $0.355$ to $0.126$ on the product, and $0.225$ to $0.125$ on the
tiered order. Misspecifying only the link barely matters ($0.114$ against $0.107$).

\textbf{Forcing a total order is costly unless the ties are few.} On the
merged-top grid and the product, a linear extension degrades the fit
($\mathrm{TV}_{\mathrm{int}}$ $0.21$--$0.24$ against $0.08$), as in the reverse
misspecification of \S\ref{sec:ordering}, and its counterfactual error is
$3.0$--$3.4\times$ TACC's. On the tiered order only the ties are broken
arbitrarily, and the linear extension is the best alternative ($0.225$) while still
trailing TACC. Discarding the order altogether is in Tables~\ref{tab:ordering} and~\ref{tab:ordering-poset}.
\end{revblock}

\begin{table}[htbp]
\centering
\small
\setlength{\tabcolsep}{3pt}
\caption{The value of the ordering on the further poset designs, as in
Table~\ref{tab:ordering}: TACC against all-categorical (order discarded) on $P_1$.}
\label{tab:ordering-poset}
\begin{adjustbox}{max width=\textwidth}
\begin{tabular}{lllcccccrr}
\hline
 & & & \multicolumn{2}{c}{$\mathrm{TV}_{\mathrm{cf}}$} & \multicolumn{2}{c}{\tvi{$\mathrm{TV}_{\mathrm{int}}$}} & & & \\
experiment & node & true type & TACC & swapped & \tvi{TACC} & \tvi{swapped} & ratio & wins & $p$ \\
\hline
Poset, merged-top, logistic & $P_1$ & partial & $\mathbf{0.107}$ & $0.205$ & \tvi{$\mathbf{0.079}$} & \tvi{$0.103$} & $1.9$ & $9/10$ & $0.004$ \\
Poset, product & $P_1$ & partial & $\mathbf{0.126}$ & $0.266$ & \tvi{$\mathbf{0.080}$} & \tvi{$0.106$} & $2.1$ & $9/10$ & $0.004$ \\
Poset, tiered & $P_1$ & tiered & $\mathbf{0.125}$ & $0.293$ & \tvi{$\mathbf{0.078}$} & \tvi{$0.101$} & $2.3$ & $10/10$ & $0.002$ \\
\hline
\end{tabular}
\end{adjustbox}
\end{table}
\begin{revblock}

\textbf{Errors propagate to correctly specified children.} $C_1$ has the same
Gumbel-max fit under TACC and linear ext.\ ($\mathrm{TV}_{\mathrm{int}}$
$0.068$--$0.070$), yet its counterfactual error is $33$--$49\%$ lower under TACC,
because it is fed a better counterfactual of its parent. Every arm has zero
consistency error.
\end{revblock}

\FloatBarrier
\subsection{Random SCMs}
\label{app:large}
\rev{Table~\ref{tab:large-full} gives every arm with standard deviations,
Table~\ref{tab:large-mech} splits ordinal and nominal descendants by their true
mechanism, and Figure~\ref{fig:depth} plots error against graph distance from
the intervened node.}

\rev{\textbf{Paired comparisons.} Bold marks overlap of the across-graph sds,
which is conservative because all arms are evaluated on the same graphs. Pairing
by graph, with a two-sided Wilcoxon signed-rank test on the per-graph
difference from TACC (graphs pooled over $d$): TACC is lower than ANM + round and
one-hot GPR on every discrete type ($p\le0.04$), lower than linear ext.\ on
partially ordered ($20$ of $20$ graphs, $p<0.001$) and nominal nodes ($22$ of
$27$, $p=0.001$), and lower than the relaxed estimator on every discrete type but
partially ordered ($p<0.001$). The paired comparison with all-categorical is in
Table~\ref{tab:ordering}.}

\begin{table}[htbp]
\centering
\color{revcolor}
\small
\setlength{\tabcolsep}{4pt}
\caption{Random SCMs whose mechanisms all lie outside the estimators'
families: counterfactual error by descendant type ($\mathrm{TV}_{\mathrm{cf}}$;
$W_1/\sigma$, the 1-Wasserstein distance over the node's sd, for continuous nodes), mean $\pm$ sd over $10$ random graphs per
$d$. $^\ast$Control: $d=12$ with every mechanism drawn from the estimators' own
families. Bold: within one sd of the lowest in the row.}
\label{tab:large-full}
\begin{adjustbox}{max width=\textwidth}
\begin{tabular}{llccccc}
\hline
setting & type & TACC & linear ext. & ANM + round & one-hot GPR & relaxed, $\tau{=}0.5$ \\
\hline
$d=8$ & $B$ & $\mathbf{0.029 \pm 0.014}$ & $\mathbf{0.028 \pm 0.015}$ & $0.078 \pm 0.042$ & $0.082 \pm 0.047$ & $0.098 \pm 0.030$ \\
 & $C$ & $\mathbf{0.042 \pm 0.032}$ & $\mathbf{0.051 \pm 0.034}$ & $0.151 \pm 0.147$ & $0.142 \pm 0.132$ & $0.113 \pm 0.060$ \\
 & $O$ & $\mathbf{0.113 \pm 0.056}$ & $\mathbf{0.114 \pm 0.057}$ & $\mathbf{0.135 \pm 0.078}$ & $0.182 \pm 0.140$ & $\mathbf{0.121 \pm 0.051}$ \\
 & $P$ & $\mathbf{0.047 \pm 0.012}$ & $0.234 \pm 0.182$ & $0.303 \pm 0.177$ & $0.274 \pm 0.179$ & $\mathbf{0.057 \pm 0.023}$ \\
 & $X$ ($W_1/\sigma$) & $\mathbf{0.112 \pm 0.056}$ & $\mathbf{0.120 \pm 0.064}$ & $0.172 \pm 0.105$ & $\mathbf{0.167 \pm 0.105}$ & $\mathbf{0.138 \pm 0.046}$ \\
\hline
$d=12$ & $B$ & $\mathbf{0.028 \pm 0.019}$ & $\mathbf{0.028 \pm 0.018}$ & $0.069 \pm 0.050$ & $0.073 \pm 0.054$ & $0.083 \pm 0.030$ \\
 & $C$ & $\mathbf{0.049 \pm 0.027}$ & $\mathbf{0.051 \pm 0.028}$ & $0.132 \pm 0.112$ & $0.126 \pm 0.103$ & $0.128 \pm 0.041$ \\
 & $O$ & $\mathbf{0.123 \pm 0.063}$ & $\mathbf{0.126 \pm 0.063}$ & $\mathbf{0.164 \pm 0.071}$ & $0.210 \pm 0.137$ & $\mathbf{0.136 \pm 0.054}$ \\
 & $P$ & $\mathbf{0.064 \pm 0.038}$ & $0.159 \pm 0.189$ & $0.169 \pm 0.187$ & $0.156 \pm 0.161$ & $\mathbf{0.093 \pm 0.044}$ \\
 & $X$ ($W_1/\sigma$) & $\mathbf{0.097 \pm 0.058}$ & $\mathbf{0.108 \pm 0.070}$ & $\mathbf{0.128 \pm 0.082}$ & $\mathbf{0.138 \pm 0.083}$ & $\mathbf{0.133 \pm 0.054}$ \\
\hline
$d=16$ & $B$ & $\mathbf{0.027 \pm 0.020}$ & $\mathbf{0.028 \pm 0.018}$ & $0.081 \pm 0.069$ & $0.085 \pm 0.076$ & $0.093 \pm 0.025$ \\
 & $C$ & $\mathbf{0.058 \pm 0.029}$ & $\mathbf{0.065 \pm 0.034}$ & $0.163 \pm 0.121$ & $0.144 \pm 0.095$ & $0.130 \pm 0.039$ \\
 & $O$ & $\mathbf{0.137 \pm 0.096}$ & $\mathbf{0.139 \pm 0.095}$ & $\mathbf{0.153 \pm 0.096}$ & $\mathbf{0.150 \pm 0.073}$ & $\mathbf{0.164 \pm 0.094}$ \\
 & $P$ & $\mathbf{0.071 \pm 0.036}$ & $0.196 \pm 0.157$ & $0.201 \pm 0.139$ & $0.191 \pm 0.135$ & $\mathbf{0.092 \pm 0.038}$ \\
 & $X$ ($W_1/\sigma$) & $\mathbf{0.102 \pm 0.055}$ & $\mathbf{0.120 \pm 0.070}$ & $0.167 \pm 0.110$ & $0.160 \pm 0.101$ & $\mathbf{0.138 \pm 0.038}$ \\
\hline
control$^\ast$ & $B$ & $\mathbf{0.040 \pm 0.026}$ & --- & $0.081 \pm 0.052$ & --- & --- \\
 & $C$ & $\mathbf{0.061 \pm 0.042}$ & --- & $0.146 \pm 0.070$ & --- & --- \\
 & $O$ & $\mathbf{0.081 \pm 0.032}$ & --- & $0.135 \pm 0.087$ & --- & --- \\
 & $P$ & $\mathbf{0.069 \pm 0.038}$ & --- & $0.343 \pm 0.159$ & --- & --- \\
 & $X$ ($W_1/\sigma$) & $\mathbf{0.115 \pm 0.062}$ & --- & $\mathbf{0.161 \pm 0.093}$ & --- & --- \\
\hline
\end{tabular}
\end{adjustbox}
\end{table}

\begin{table}[htbp]
\centering
\color{revcolor}
\small
\setlength{\tabcolsep}{4pt}
\caption{Random SCMs, ordinal and nominal descendants split by their true
mechanism: $\mathrm{TV}_{\mathrm{cf}}$, mean $\pm$ sd over the graphs (``graphs'',
pooled over $d$) that contain at least one such node. Bold: within one sd of the
lowest in the row. All-categorical: Table~\ref{tab:ordering}.}
\label{tab:large-mech}
\begin{adjustbox}{max width=\textwidth}
\begin{tabular}{llccccc}
\hline
type & true mechanism & graphs & TACC & linear ext. & ANM + round & one-hot GPR \\
\hline
$O$ & continuation ratio & 20 & $\mathbf{0.139 \pm 0.093}$ & $\mathbf{0.139 \pm 0.092}$ & $\mathbf{0.126 \pm 0.085}$ & $\mathbf{0.112 \pm 0.071}$ \\
$O$ & threshold, logistic & 7 & $\mathbf{0.113 \pm 0.058}$ & $\mathbf{0.115 \pm 0.060}$ & $0.190 \pm 0.112$ & $0.360 \pm 0.105$ \\
$O$ & threshold, $t_3$ & 13 & $\mathbf{0.057 \pm 0.020}$ & $\mathbf{0.064 \pm 0.022}$ & $0.122 \pm 0.098$ & $0.134 \pm 0.121$ \\
$C$ & multinomial probit & 17 & $\mathbf{0.044 \pm 0.029}$ & $\mathbf{0.052 \pm 0.034}$ & $0.160 \pm 0.145$ & $0.140 \pm 0.121$ \\
$C$ & nearest prototype & 19 & $\mathbf{0.050 \pm 0.037}$ & $\mathbf{0.052 \pm 0.039}$ & $0.097 \pm 0.099$ & $0.089 \pm 0.093$ \\
\hline
\end{tabular}
\end{adjustbox}
\end{table}

\begin{figure}[htbp]
\centering
\color{revcolor}
\includegraphics[width=0.6\linewidth]{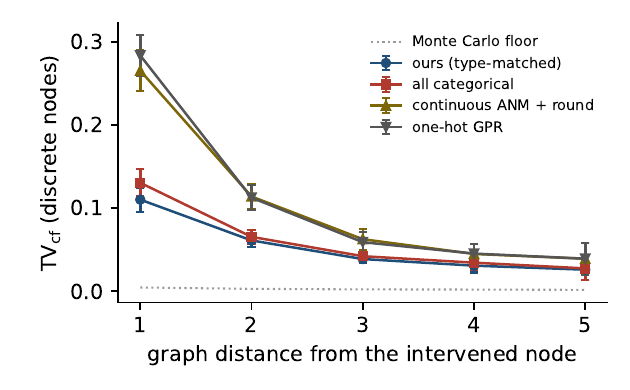}
\caption{Counterfactual error of discrete nodes against their graph distance
from the intervened node, on the random SCMs of \S\ref{sec:large} (all $d$,
mean $\pm$ standard error over seeds).}
\label{fig:depth}
\end{figure}

\FloatBarrier
\subsection{Adult}
\label{app:adult}
\rev{Table~\ref{tab:adult} gives every discrete descendant with standard deviations
at $n=1000$, Table~\ref{tab:adult-n500} repeats it at $n=500$, and
Table~\ref{tab:adult-real} compares the arms on the real Adult data. Table~\ref{tab:adult-main} summarises the mean errors across arms.}

\begin{table}[htbp]
\centering
\caption{\rev{Semi-synthetic Adult: $\mathrm{TV}_{\mathrm{cf}}$, mean over $10$
seeds at $n=1000$ (all nodes, with sd: Table~\ref{tab:adult}). Under
\texttt{do(education)} TACC and linear ext.\ coincide by construction: they
differ only in education's coupling, which the intervention overrides.
All-categorical: Table~\ref{tab:ordering}. Bold: within one sd of the lowest in the row (sds from Table~\ref{tab:adult}).}}
\label{tab:adult-main}
\small
\setlength{\tabcolsep}{3pt}
\begin{revblock}
\begin{tabular}{llcccc}
\hline
query & node & TACC & linear ext. & ANM + round & one-hot GPR \\
\hline
\texttt{do(education=Bachelors)} & occupation & $\mathbf{0.075}$ & $\mathbf{0.075}$ & $0.507$ & $0.410$ \\
 & income & $\mathbf{0.094}$ & $\mathbf{0.093}$ & $0.206$ & $0.259$ \\
\hline
\texttt{do(marital=married)} & education & $\mathbf{0.094}$ & $\mathbf{0.097}$ & $0.181$ & $0.215$ \\
 & occupation & $\mathbf{0.040}$ & $\mathbf{0.041}$ & $0.110$ & $0.086$ \\
\hline
\texttt{do(age{+}{=}5)} & education & $\mathbf{0.210}$ & $\mathbf{0.232}$ & $0.393$ & $0.383$ \\
 & occupation & $\mathbf{0.075}$ & $\mathbf{0.080}$ & $0.166$ & $0.116$ \\
 & income & $\mathbf{0.045}$ & $\mathbf{0.048}$ & $0.082$ & $0.087$ \\
\hline
\end{tabular}
\end{revblock}
\end{table}

\begin{table}[htbp]
\centering
\color{revcolor}
\small
\setlength{\tabcolsep}{4pt}
\caption{Semi-synthetic Adult: $\mathrm{TV}_{\mathrm{cf}}$ of each discrete
descendant, by query, mean $\pm$ sd over $10$ seeds at $n=1000$. Bold: within one sd of the lowest in the row.}
\label{tab:adult}
\begin{adjustbox}{max width=\textwidth}
\begin{tabular}{llcccc}
\hline
query & node & TACC & linear ext. & ANM + round & one-hot GPR \\
\hline
\texttt{do(education=Bachelors)} & workclass & $\mathbf{0.064 \pm 0.024}$ & $\mathbf{0.065 \pm 0.024}$ & $\mathbf{0.069 \pm 0.016}$ & $\mathbf{0.070 \pm 0.022}$ \\
 & occupation & $\mathbf{0.075 \pm 0.032}$ & $\mathbf{0.075 \pm 0.032}$ & $0.507 \pm 0.149$ & $0.410 \pm 0.114$ \\
 & income & $\mathbf{0.094 \pm 0.040}$ & $\mathbf{0.093 \pm 0.041}$ & $0.206 \pm 0.100$ & $0.259 \pm 0.105$ \\
\hline
\texttt{do(marital=married)} & education & $\mathbf{0.094 \pm 0.073}$ & $\mathbf{0.097 \pm 0.071}$ & $0.181 \pm 0.091$ & $0.215 \pm 0.119$ \\
 & workclass & $\mathbf{0.040 \pm 0.039}$ & $\mathbf{0.039 \pm 0.040}$ & $\mathbf{0.039 \pm 0.034}$ & $\mathbf{0.028 \pm 0.019}$ \\
 & occupation & $\mathbf{0.040 \pm 0.022}$ & $\mathbf{0.041 \pm 0.027}$ & $0.110 \pm 0.064$ & $0.086 \pm 0.045$ \\
 & income & $\mathbf{0.064 \pm 0.052}$ & $\mathbf{0.065 \pm 0.052}$ & $\mathbf{0.094 \pm 0.053}$ & $0.124 \pm 0.065$ \\
\hline
\texttt{do(hours{+}{=}10)} & income & $\mathbf{0.057 \pm 0.026}$ & $\mathbf{0.056 \pm 0.027}$ & $\mathbf{0.060 \pm 0.032}$ & $\mathbf{0.060 \pm 0.032}$ \\
\hline
\texttt{do(age{+}{=}5)} & education & $\mathbf{0.210 \pm 0.064}$ & $\mathbf{0.232 \pm 0.093}$ & $0.393 \pm 0.099$ & $0.383 \pm 0.096$ \\
 & marital & $\mathbf{0.035 \pm 0.015}$ & $\mathbf{0.036 \pm 0.016}$ & $0.139 \pm 0.056$ & $0.139 \pm 0.056$ \\
 & workclass & $\mathbf{0.040 \pm 0.013}$ & $\mathbf{0.041 \pm 0.015}$ & $0.058 \pm 0.025$ & $0.055 \pm 0.016$ \\
 & occupation & $\mathbf{0.075 \pm 0.027}$ & $\mathbf{0.080 \pm 0.028}$ & $0.166 \pm 0.072$ & $0.116 \pm 0.044$ \\
 & income & $\mathbf{0.045 \pm 0.019}$ & $\mathbf{0.048 \pm 0.026}$ & $0.082 \pm 0.036$ & $0.087 \pm 0.036$ \\
\hline
\end{tabular}
\end{adjustbox}
\end{table}

\begin{table}[htbp]
\centering
\color{revcolor}
\small
\setlength{\tabcolsep}{4pt}
\caption{Semi-synthetic Adult at $n=500$ (cf.\ Table~\ref{tab:adult}, $n=1000$).}
\label{tab:adult-n500}
\begin{adjustbox}{max width=\textwidth}
\begin{tabular}{llcccc}
\hline
query & node & TACC & linear ext. & ANM + round & one-hot GPR \\
\hline
\texttt{do(education=Bachelors)} & workclass & $\mathbf{0.108 \pm 0.061}$ & $\mathbf{0.107 \pm 0.062}$ & $0.152 \pm 0.084$ & $\mathbf{0.095 \pm 0.032}$ \\
 & occupation & $\mathbf{0.132 \pm 0.045}$ & $\mathbf{0.132 \pm 0.045}$ & $0.485 \pm 0.138$ & $0.451 \pm 0.099$ \\
 & income & $\mathbf{0.091 \pm 0.033}$ & $\mathbf{0.090 \pm 0.032}$ & $0.129 \pm 0.073$ & $0.146 \pm 0.075$ \\
\hline
\texttt{do(marital=married)} & education & $\mathbf{0.149 \pm 0.037}$ & $\mathbf{0.149 \pm 0.028}$ & $0.222 \pm 0.121$ & $0.302 \pm 0.082$ \\
 & workclass & $0.048 \pm 0.026$ & $0.047 \pm 0.025$ & $\mathbf{0.030 \pm 0.015}$ & $\mathbf{0.038 \pm 0.017}$ \\
 & occupation & $\mathbf{0.064 \pm 0.030}$ & $\mathbf{0.068 \pm 0.034}$ & $0.139 \pm 0.088$ & $0.121 \pm 0.061$ \\
 & income & $\mathbf{0.083 \pm 0.033}$ & $\mathbf{0.082 \pm 0.033}$ & $\mathbf{0.082 \pm 0.047}$ & $0.124 \pm 0.062$ \\
\hline
\texttt{do(hours{+}{=}10)} & income & $\mathbf{0.030 \pm 0.015}$ & $\mathbf{0.030 \pm 0.015}$ & $\mathbf{0.028 \pm 0.019}$ & $\mathbf{0.028 \pm 0.019}$ \\
\hline
\texttt{do(age{+}{=}5)} & education & $\mathbf{0.230 \pm 0.111}$ & $\mathbf{0.235 \pm 0.095}$ & $0.406 \pm 0.126$ & $0.404 \pm 0.121$ \\
 & marital & $\mathbf{0.053 \pm 0.041}$ & $\mathbf{0.053 \pm 0.041}$ & $0.168 \pm 0.106$ & $0.168 \pm 0.106$ \\
 & workclass & $\mathbf{0.052 \pm 0.025}$ & $\mathbf{0.055 \pm 0.028}$ & $\mathbf{0.050 \pm 0.018}$ & $\mathbf{0.055 \pm 0.023}$ \\
 & occupation & $\mathbf{0.104 \pm 0.033}$ & $\mathbf{0.100 \pm 0.030}$ & $0.233 \pm 0.073$ & $0.162 \pm 0.060$ \\
 & income & $\mathbf{0.058 \pm 0.034}$ & $\mathbf{0.058 \pm 0.039}$ & $\mathbf{0.077 \pm 0.064}$ & $\mathbf{0.088 \pm 0.061}$ \\
\hline
\end{tabular}
\end{adjustbox}
\end{table}

\begin{table}[htbp]
\centering
\color{revcolor}
\small
\setlength{\tabcolsep}{4pt}
\caption{Real Adult ($n=1000$ training rows): held-out mean log-likelihood of
each node's fitted equation on $5{,}000$ held-out records (left), and the mean
$\mathrm{TV}$ between each arm's counterfactual law and TACC for the same
queries on real individuals (right).}
\label{tab:adult-real}
\begin{adjustbox}{max width=\textwidth}
\begin{tabular}{l ccccc cccc}
\hline
 & \multicolumn{5}{c}{held-out log-likelihood} & \multicolumn{4}{c}{$\mathrm{TV}$ to TACC} \\
node & TACC & all-cat. & lin.\ ext. & ANM + round & one-hot & all-cat. & lin.\ ext. & ANM + round & one-hot \\
\hline
marital & $-0.739$ & $-0.739$ & $-0.739$ & $-0.932$ & $-0.753$ & $0.005$ & $0.007$ & $0.161$ & $0.161$ \\
education & $-1.781$ & $-1.747$ & $-1.777$ & $-2.076$ & $-1.752$ & $0.376$ & $0.052$ & $0.289$ & $0.463$ \\
workclass & $-0.714$ & $-0.714$ & $-0.714$ & $-0.855$ & $-0.736$ & $0.019$ & $0.012$ & $0.079$ & $0.078$ \\
occupation & $-1.092$ & $-1.092$ & $-1.092$ & $-1.220$ & $-1.119$ & $0.046$ & $0.010$ & $0.316$ & $0.233$ \\
income & $-0.418$ & $-0.418$ & $-0.418$ & $-0.423$ & $-0.386$ & $0.023$ & $0.009$ & $0.170$ & $0.180$ \\
\hline
\end{tabular}
\end{adjustbox}
\end{table}

\FloatBarrier
\subsection{\rev{Calibration and partial-identification bounds}}
\label{app:calibration}
\begin{revblock}
Posterior averaging leaves the point estimate unchanged
(Table~\ref{tab:baselines}); here we test whether its spread is calibrated. For
each query we draw $K=400$ joint latent pairs $(\mathbf f^F,\tilde{\mathbf f})$ and
compute the counterfactual distribution \emph{conditional on each draw} with the
estimator's own coupling ($2{,}000$ Gumbel-max abduction draws, or the exact
truncated-normal identity for the ordinal fit). The mean of the $K$ distributions is
the full estimator, and their central $90\%$ interval per class is a credible
interval. Same SCMs, seeds, queries, and $n=500$ as Table~\ref{tab:baselines}.

With the correct coupling the intervals cover the exact counterfactual probability
$89.6\%$ of the time on both the Single-parent SCM and the ordinal fit of the
Interaction SCM, against a target of $90\%$. Under misspecification coverage
collapses to $51\%$, with intervals as narrow as before: the posterior quantifies
uncertainty about the class probabilities, not the coupling. The log score barely
moves between the full and plug-in arms.
\end{revblock}

\begin{table}[htbp]
\centering
\small
\setlength{\tabcolsep}{4pt}
\caption{\rev{Calibration of the posterior counterfactual, mean $\pm$ sd over
$10$ seeds ($8$ queries each) at $n=500$. Coverage: fraction of true
per-class counterfactual probabilities inside the central $90\%$ posterior
interval. Log score: cross-entropy $H(p_{\mathrm{true}},\hat p)$ of the point
distribution, with ``truth'' its minimum. $^\dagger$coupling misspecified (ordinal data fit with the Gumbel-max coupling).}}
\label{tab:calibration}
\begin{revblock}
\begin{adjustbox}{max width=\linewidth}
\begin{tabular}{lccccc}
\hline
 & \multicolumn{2}{c}{$90\%$ interval} & \multicolumn{3}{c}{log score (nats)} \\
arm & coverage & width & full & plug-in & truth \\
\hline
Single-parent $C_1$, Gumbel-max & $0.896 \pm 0.149$ & $0.072 \pm 0.013$ & $0.499 \pm 0.083$ & $0.500 \pm 0.083$ & $0.491$ \\
Interaction $O_1$, Gumbel-max$^\dagger$ & $0.508 \pm 0.121$ & $0.122 \pm 0.045$ & $0.738 \pm 0.327$ & $0.783 \pm 0.358$ & $0.391$ \\
Interaction $O_1$, ordinal (TACC) & $0.896 \pm 0.134$ & $0.155 \pm 0.064$ & $0.431 \pm 0.176$ & $0.438 \pm 0.185$ & $0.391$ \\
\hline
\end{tabular}
\end{adjustbox}
\end{revblock}
\end{table}

\begin{revblock}
\textbf{Partial-identification bounds.} Uncertainty about the coupling itself is
quantified by bounds over \emph{all} couplings of the factual and counterfactual
class distributions $\mathbf p,\mathbf q$. For a single node with factual class $i$
these are the Fr\'echet--Hoeffding bounds
\begin{equation*}
  \frac{\max(0,\,p_i+q_j-1)}{p_i}\;\le\;P(X_r^{\CF}=j\mid X_r=i)\;\le\;\frac{\min(p_i,\,q_j)}{p_i},
\end{equation*}
which are sharp. Table~\ref{tab:bounds} reports them at the true $\mathbf p,\mathbf q$ (pure
coupling uncertainty) and widened by the GP posterior (the $5\%$ quantile of the
per-draw lower bounds to the $95\%$ quantile of the per-draw upper bounds). The
oracle bounds always contain the truth. The posterior bounds cover it
$95$--$100\%$ of the time, including under misspecification, but are $0.65$--$0.79$
wide against $0.07$--$0.16$ for our intervals: a safe but nearly uninformative
fallback when the variable's type is in doubt. Assuming monotonicity, as the ordinal coupling does,
collapses them to the comonotone point. For Gumbel-max data the true
counterfactual often lies on the boundary of the bounds: counterfactual stability
is an extreme coupling, not a central one.
\end{revblock}

\begin{table}[htbp]
\centering
\footnotesize
\setlength{\tabcolsep}{2pt}
\caption{\rev{Posterior intervals (Table~\ref{tab:calibration}) against sharp
partial-identification bounds on $P(X_r^{\CF}=j\mid X_r=i)$, mean $\pm$ sd over
$10$ seeds at $n=500$. Oracle bounds use the true class distributions and
contain the truth in every query. Posterior bounds take the $5\%$/$95\%$
quantiles of the per-draw bounds. $^\dagger$As in Table~\ref{tab:calibration}.}}
\label{tab:bounds}
\begin{revblock}
\begin{adjustbox}{max width=\linewidth}
\begin{tabular}{lccccc}
\hline
 & \multicolumn{2}{c}{posterior interval (TACC)} & oracle bounds & \multicolumn{2}{c}{posterior bounds} \\
arm & coverage & width & width & coverage & width \\
\hline
Single-parent $C_1$, Gumbel-max & $0.896 \pm 0.149$ & $0.072 \pm 0.013$ & $0.581 \pm 0.098$ & $1.000 \pm 0.000$ & $0.650 \pm 0.097$ \\
Interaction $O_1$, Gumbel-max$^\dagger$ & $0.508 \pm 0.121$ & $0.122 \pm 0.045$ & $0.711 \pm 0.113$ & $0.954 \pm 0.054$ & $0.787 \pm 0.096$ \\
Interaction $O_1$, ordinal (TACC) & $0.896 \pm 0.134$ & $0.155 \pm 0.064$ & $0.711 \pm 0.113$ & $0.975 \pm 0.045$ & $0.777 \pm 0.090$ \\
\hline
\end{tabular}
\end{adjustbox}
\end{revblock}
\end{table}

\FloatBarrier
\end{document}